%% file: main.tex
\documentclass[twoside,11pt]{article}
\usepackage{lastpage}

\usepackage{blindtext}
\usepackage{amsmath,amssymb,amsfonts}
\usepackage{pifont}
\usepackage{microtype}
\usepackage{xcolor}
\usepackage{subcaption}
\usepackage{booktabs} 
\usepackage{float}
\usepackage{makecell}

\usepackage{bm}
\usepackage{algorithm}
\usepackage{algorithmic}
\usepackage{graphicx}
\usepackage{textcomp}
\usepackage{comment}
\usepackage{graphicx,psfrag,epsf}
\usepackage{enumerate}
\usepackage{natbib}
\usepackage{xspace}
\usepackage{mathrsfs}
\usepackage{epsfig}
\usepackage[abbrvbib]{jmlr2e}
\allowdisplaybreaks
\input{headjmlr}

\usepackage{lastpage}

\firstpageno{1}
\begin{document}

\jmlrheading{25}{2024}{1-\pageref{LastPage}}{3/24; Revised
10/24}{11/24}{24-0400}{Naichen Shi, Salar Fattahi, Raed Al Kontar}
\ShortHeadings{TCMF}{Shi, Fattahi, and Al Kontar}


\title{Triple Component Matrix Factorization: Untangling Global, Local, and Noisy Components}

\author{\name Naichen Shi \email naichens@umich.edu \\
       \addr Department of Industrial \& Operations Engineering\\
       University of Michigan\\
       Ann Arbor, MI 48109, USA       
       \AND
       \name Salar Fattahi \email fattahi@umich.edu \\
       \addr Department of Industrial \& Operations Engineering\\
       University of Michigan\\
       Ann Arbor, MI 48109, USA
       \AND
       \name Raed Al Kontar \thanks{Corresponding author} \email alkontar@umich.edu \\
       \addr Department of Industrial \& Operations Engineering\\
       University of Michigan\\
       Ann Arbor, MI 48109, USA}

\editor{Mahdi Soltanolkotabi}
\maketitle

\begin{abstract}
In this work, we study the problem of common and unique feature extraction from noisy data. When we have $N$ observation matrices from $N$ different and associated sources corrupted by sparse and potentially gross noise, can we recover the common and unique components from these noisy observations? This is a challenging task as the number of parameters to estimate is approximately thrice the number of observations. Despite the difficulty, we propose an intuitive alternating minimization algorithm called triple component matrix factorization (\name) to recover the three components exactly. \name is distinguished from existing works in literature thanks to two salient features. First, \name is a principled method to separate the three components given noisy observations provably. Second, the bulk of the computation in \name can be distributed. 
On the technical side, we formulate the problem as a constrained nonconvex nonsmooth optimization problem. Despite the intricate nature of the problem, we provide a Taylor series characterization of its solution by solving the corresponding Karush–Kuhn–Tucker conditions.
Using this characterization, we can show that the alternating minimization algorithm makes significant progress at each iteration and converges into the ground truth at a linear rate. Numerical experiments in video segmentation and anomaly detection highlight the superior feature extraction abilities of \name.
\end{abstract}


\section{Introduction}
In the era of Big Data, an important task is to find low-rank features from high-dimensional observations. Methods including principal component analysis  \citep{hotellingpca}, low-rank matrix factorization \citep{mfrecommender}, and dictionary learning \citep{ksvd}, have found success in numerous fields of statistics and machine learning \citep{wright2022high}. Among them, matrix factorization (MF) is an efficient method to identify the features that best explain the observation matrices. 

Despite the wide popularity, standard MF methods are known to be brittle in the presence of outliers with huge noise \citep{robustpca}. These noises are often sparse but can have large norms.
 A series of methods (e.g., \citep{robustpca,nonconvexrobustpca,robustmatrixcompletion1, fattahi2020exact, chenfan}) have been developed to estimate low-rank features from data that contain outliers. When the portion of outliers is not too large, one can \textit{provably} identify the outliers and the low-rank components with convex programming \citep{robustpca} or nonconvex optimization algorithms equipped with convergence guarantees \citep{nonconvexrobustpca}.

Recently, there has been a growing number of applications where data are acquired from diverse but connected sources, such as smartphones, car sensors, or medical records from different patients. This type of data displays both mutual and individual characteristics. For instance, in biostatistics, the measurements of different miRNA and gene expressions from the same set of samples can reveal co-varying patterns yet exhibit heterogeneous trends \citep{jive}. Statistical modeling of the common information among all data sources and the specific information for each source is of central interest in these applications. Multiple works propose to recover such common and unique features by minimizing the square norm of the residuals of fitting \citep{jive,cobe, slide,bidifac, groupnmf,inmf, personalizedpca, hmf,perdl}. These methods prove to be useful in aligning genetics features \citep{jive}, visualizing bone and soft tissues in X-ray images \citep{cobe}, functional magnetic resonance imaging \citep{fmri}, surveillance video segmentation \citep{personalizedpca}, stocks market analysis \citep{hmf}, and many more.

Though these algorithms achieve decent performance on multiple applications, they rely on least square estimates, which are not robust to outliers in data. Real-world data are commonly corrupted by outliers \citep{sensornoise}. Factors including measurement errors or sensor malfunctions can give rise to large noise in data. These outliers can substantially skew the estimation of low-rank features. As such, we attempt to answer the following question.

\begin{center}
    \noindent\fcolorbox{white}[rgb]{0.95,0.95,0.95}{\begin{minipage}{0.98\columnwidth}
	\begin{center}
Question: How can one \textit{provably} identify low-rank common and unique information robustly from data corrupted by outlier noise?
\end{center}
\end{minipage}}
\end{center}

A natural thought is to borrow techniques in robust PCA to handle outlier noise. Indeed, there exist a few heuristic methods in literature \citep{rjive,rcica,rajive} to find robust estimates of shared and unique features. These methods often use $\ell_1$ regularization \citep{rjive,rcica} or Huber loss \citep{rajive} to accommodate the sparsity of noise. However, these algorithms are mainly based on heuristics and lack theoretical guarantees, thus potentially compromising the quality of their outputs. A theoretically justifiable method to identify low-rank shared and unique components from outlier noise is still lacking. In this paper, we will study the question rigorously and develop an efficient algorithm to solve it.

\section{Problem Statement}
\label{sec:problemstatement}
We consider the framework where $N$ observation matrices $\matM_{(1)}, \matM_{(2)},\cdots,\matM_{(N)}$ come from $N\in \mathbb{N}^+$ different but associated sources. These matrices $\matM_{(i)}\in \mathbb{R}^{n_1\times n_{2,(i)}}$ have the same number of features $n_1$. To model their commonality and uniqueness, we assume each matrix is driven by $r_1$ shared factors and $r_{2,(i)}$ unique factors and contaminated by potentially gross noise. More specifically, we consider the model where the observation $\matM_{(i)}$ from source $i$ is generated by,
\begin{equation}
\label{eqn:matrixmodel}
    \matM_{(i)} = \matUst_g\matVst_{(i),g}^T+\matUst_{(i),l}\matVst_{(i),l}^T + \matSst_{(i)},
\end{equation} 
where $\matUst_g\in \mathbb{R}^{n_1\times r_1}$, $\matVst_{(i),g}\in \mathbb{R}^{n_{2,(i)}\times r_1}$, $\matUst_{(i),l}\in \mathbb{R}^{n_1\times r_{2,(i)}}$, $\matVst_{(i),l}\in \mathbb{R}^{n_{2,(i)}\times r_{2,(i)}}$, $\matSst_{(i)}\in \mathbb{R}^{n_1\times n_{2,(i)}}$. We use ${}^{\star}$ to denote the ground truth. $r_1$ is the rank of global (shared) feature matrices, and $r_{2,(i)}$ is the rank of local (unique) feature matrix from source $i$. The matrix $\matUst_g\matVst_{(i),g}^T$ models the shared low-rank part of the observation matrix, as the column space is the same across different sources. $\matUst_{(i),l}\matVst_{(i),l}^T$ models the unique low-rank part. 
$\matSst_{(i)}$ models the noise from source $i$.

In matrix factorization problems, the representations $\matUst$ and $\matVst$ often correspond to latent data features. For instance, in recommender systems, $\matUst$ can be interpreted as user features that reveal their preferences on different items in the latent space \citep{mfrecommender}. For better interpretability, it is often desirable to have the underlying features disentangled so that each feature can vary independently of others \citep{betavae}. Under this rationale, we consider the model where shared and unique factors are orthogonal,

\begin{equation}
\label{eqn:matrixorthogonal}
     \matUst_g^T\matUst_{(i),l}=0, \ \forall i\in[N],
\end{equation} 
where $[N]$ denotes the set $\{1,2,\cdots, N\}$. The orthogonality of features implies that the shared and unique features span different subspaces, thus describing different patterns in the observation. The orthogonal condition \eqref{eqn:matrixorthogonal} is thus an inductive bias that reflects our prior belief about the independence between common and unique factors and naturally models a diverse range of applications, such as miRNA and gene expression \citep{jive}, human faces \citep{cobe}, and many more \citep{rjive,personalizedpca}. 

\revise{We should note that the orthogonality~\eqref{eqn:matrixorthogonal} does not limit the model representation power. Suppose $\matUst_g^T\matUst_{(i),l}\neq 0$ otherwise, we can decompose $\matUst_{(i),l}$ into the two parts, $
    \matUst_{(i),l} = \matUst_g\left(\matUst_g^T\matUst_g\right)^{-1}\matUst_g^T\matUst_{(i),l} + \left(\matI-\matUst_g\left(\matUst_g^T\matUst_g\right)^{-1}\matUst_g^T\right)\matUst_{(i),l}$.
The first part is in the column subspace of $\matUst_g$, while the second part is in the orthogonal space of the column subspace of $\matUst_g$. If we define $\widetilde{\matUst}_{(i),l} = \left(\matI-\matUst_g\left(\matUst_g^T\matUst_g\right)^{-1}\matUst_g^T\right)\matUst_{(i),l}$, and $\widetilde{\matVst}_{(i),g} =\matVst_{(i),g}+ \matVst_{(i),l}\matUst_{(i),l}^T\matUst_g\left(\matUst_g^T\matUst_g\right)^{-1}$, we have, $
   \matM_{(i)} = \matUst_g\widetilde{\matVst}_{(i),g}^T+\widetilde{\matUst}_{(i),l}\matVst_{(i),l}^T + \matSst_{(i)}$
where $\matUst_g^T\widetilde{\matUst}_{(i),l}=0$. This formulation admits the form of model~\eqref{eqn:matrixmodel} with constraint~\eqref{eqn:matrixorthogonal}.} 

The noise term $\matSst_{(i)}$ in \eqref{eqn:matrixmodel} models the sparse and large noise, where only a small fraction of $\matSst_{(i)}$ registers as nonzero. The noise sparsity is extensively invoked in literature, particularly when datasets are plagued by outliers \citep{robustpca,nonconvexrobustpca,rpcaoutlier,chenfan}.

\subsection{Challenges}

Given data generation model \eqref{eqn:matrixmodel}, our task is to separate common, individual, and noise components. The task seems Herculean as the problem is under-definite: we need to estimate three sets of parameters from one set of observations. There are two major challenges associated with the problem,

\underline{\textit{Challenge 1: New identifiability conditions are needed.}} Standard analysis in robust PCA \citep{robustpca,nonconvexrobustpca} often uses the incoherence condition to distinguish low-rank components from sparse noise. However, the incoherence condition alone is insufficient to guarantee the separability between common and unique features. Since there are infinitely many ways in which shared, unique, and noise components can form the observation matrices, it is not apparent whether untangling them is even feasible. Thus, the crux of our investigation is to understand when the separation is possible. 

Fortunately, we show that a group of conditions--known here as identifiability conditions--exists that can ensure the precise retrieval of the shared, unique, and sparse noise. Intuitively, these identifiability conditions require the three components to have ``little overlaps''. 

Based on these conditions, we will develop an alternating minimization algorithm called \name to iteratively update the three components. An illustration of the algorithm is shown in the left graph in Figure \ref{fig:difficulty}. The hard-thresholding step finds the closest sparse matrix for the data noise. We use \inneralg to denote a subroutine that represents a group of algorithms (e.g., \citep{jive,personalizedpca}) to identify common and unique low-rank features. In essence, \inneralg solves a sub-problem in \name. It is worth noting that there exist multiple algorithms in literature to implement \inneralg, many of which can produce high-quality outputs. With the implemented \inneralg, \name applies hard thresholding and \inneralg alternatively to estimate the sparse, as well as common and unique low-rank components. The left graph of Figure \ref{fig:difficulty} offers an intuitive understanding of how estimates of various components progress toward the ground truth with each iterative step. 

\begin{figure}[h!]
\centering
\includegraphics[width=14cm]{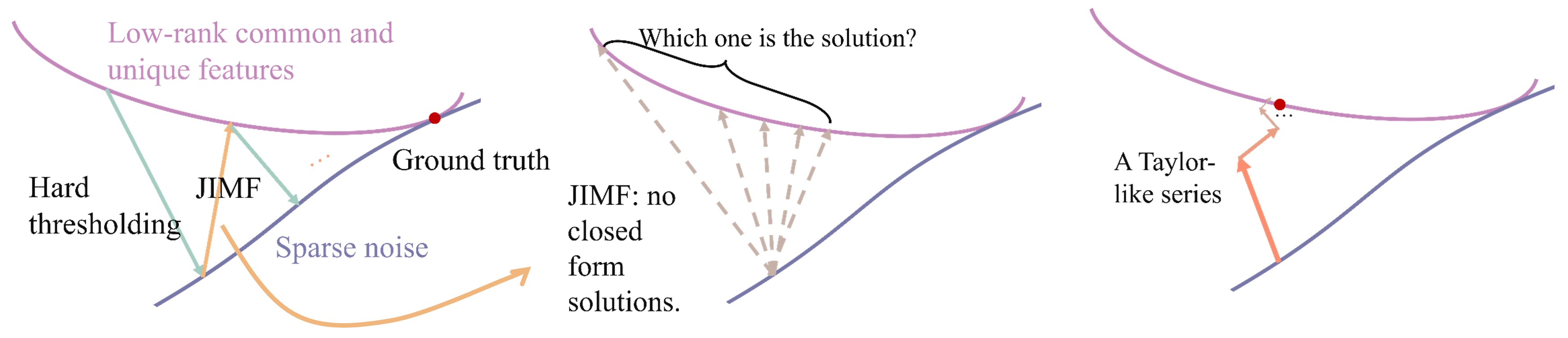}
\caption{\textit{Left}: An illustration of \name's update trajectory. The purple and blue curves represent the spaces for the low-rank and sparse matrices. The algorithm alternatively performs hard thresholding and \inneralg, making the updates closer and closer to the ground truth. \textit{Middle}: An illustration showing why insufficient understanding about the output of \inneralg can be problematic in the convergence analysis. \textit{Right}: Our contribution to represent the solution into a Taylor-like series.}
\label{fig:difficulty}
\end{figure}

\underline{\textit{Challenge 2: New analysis tools are needed.}} Showing the exact recovery of low-rank and sparse components is not easy. Even in standard robust PCA, one needs to apply highly nontrivial analytical techniques to provide theoretical guarantees. For example, Robust PCA \citep{robustpca} relies on a ``golfing scheme'' to construct dual variables that ensure the uniqueness of a convex optimization problem. Nonconvex robust PCA \citep{nonconvexrobustpca} applies a perturbation analysis of SVD to quantify the improvement of the algorithm per iteration. These techniques are tailored for standard robust PCA and cannot be directly extended to the case where both common and unique features are involved, which increases the complexity of the analysis. The major difficulty stems from the fact that \name updates the low-rank common and unique components by another iterative algorithm \inneralg. Unlike robust PCA, the output of \inneralg does not have a closed-form formula. This conceptual hurdle is illustrated in the middle graph of Figure \ref{fig:difficulty}. As a result, novel analysis tools are needed to justify the convergence of the proposed \name.

One of our key contributions in tackling the challenge is to develop innovative analysis tools by solving the Karush–Kuhn–Tucker conditions of the objective of \inneralg and express the solutions into a Taylor-like series. From the Taylor-like series, we can precisely characterize the output of \inneralg, thereby showing the series converge to a close estimate of the ground truth shared and unique features.

The Taylor-like series is depicted in the right graph of Figure \ref{fig:difficulty}. Perhaps surprisingly, regardless of the choice of the subroutine \inneralg, as long as \inneralg finds a close estimate of the optimal solutions to a subproblem, its output can be represented by an infinite series. The series describes the optimal solution of the subproblem and is independent of the intermediate steps in \inneralg. The derivation and analysis of Taylor-like series have stand-alone values in the theoretical research of the sensitivity analysis of matrix factorization. With the new analysis tool, we are able to show that even if the \inneralg only outputs a reasonable approximate solution, the meta-algorithm \name can still take advantage of the information in such an inexact solution to refine the estimates of the three components. We will elaborate on the Taylor-like series in greater detail in Section \ref{sec:convergence} and the Appendix.

We summarize our contributions in the following.

\subsection{Summary of Contributions}

\textbf{Identifiability conditions}. We discover a group of identifiability conditions sufficient for the almost exact recovery of common, unique, and sparse components from noisy observation matrices. Essentially, the identifiability conditions require that the fraction of nonzero entries in the noise not be too large, the factor matrices be incoherent, and unique factors be misaligned. The first two conditions are needed even in the standard analysis of the robust PCA, while the third condition is essential for the disentanglement of the shared and unique components.

\textbf{Efficient and distributed algorithm}. We propose a constrained nonconvex nonsmooth matrix factorization problem to solve the shared, unique, and sparse components. Despite the nonconvexity of the problem, we design a meta-algorithm called \textbf{T}riple \textbf{C}omponent \textbf{M}atrix \textbf{F}actorization (\name) to solve the problem. Our approach is able to leverage a wide range of existing methods for separating the common and unique components to precision $\epsilon$. Furthermore, \inneralg can be distributed if the subroutine \inneralg is distributed.

\textbf{Convergence guarantee}. We show that, under the identifiability conditions, our proposed \name has a convergence guarantee. To the best of our knowledge, such a guarantee is the first of its kind, as it ensures the recovery of common, unique, and noise components to high precision. Our theoretical analysis introduces new techniques to solve the KKT conditions in Taylor-like series and bound each term in the series. It sheds light on the sensitivity analysis with the $\ell_{\infty}$ norm.

\textbf{Case studies}. We use a wide range of numerical experiments to demonstrate the application of \name in different case studies, as well as the effectiveness of our proposed method. Numerical experiments corroborate theoretical convergence results. Also, the case studies on video segmentation and anomaly detection showcase the benefits of untangling shared, unique, and noisy components. 

In the rest of the paper, we provide a comprehensive review of the literature in Section \ref{sec:relatedwork}. 
Then, we elaborate on the conditions sufficient for the separation of the three components in Section \ref{sec_identifiability}. In Section \ref{sec:algorithm}, we introduce the alternating minimization algorithm. We present our convergence theorem in Section \ref{sec:convergence} and discuss the key insights in the proof and how they solve challenge 2. In Section \ref{sec:numericalexperiment}, we demonstrate the numerical experiment results. The detailed proofs are relegated to the Appendix for brevity of the main paper.

\section{Related Work}
\label{sec:relatedwork}

\textbf{Matrix Factorization} There are numerous works that analyze the theoretical and practical properties of first-order algorithms that solve the (asymmetric) matrix factorization problem $\min_{\matU,\matV}\norm{{\matM-\matU\matV^T}}_F^2$ or its variants \citep{liglobaloptimization,ye2021global,ruoyuconvergence,parkmatrixsensing,tusensing}. Among them, \citet{ruoyuconvergence} analyzes the local landscape of the optimization problem and establishes the local linear convergence of a series of first-order algorithms. \citet{parkmatrixsensing,nospuriouslocalmin} study the global geometry of the optimization problem. \citet{tusensing} proposes the Rectangular Procrustes Flow algorithm that is proved to converge linearly into the ground truth under proper initialization and a balancing regularization. Recently, \citet{ye2021global} shows that gradient descent with small and random initialization can converge to the ground truth. 

\textbf{Robust PCA} When the observation is corrupted by sparse and potentially large noise, several approaches can still identify the low-rank components. An exemplary work is Robust PCA \citep{robustpca}, which proposes an elegant convex optimization problem called principal component pursuit that uses nuclear norm and $\ell_1$ norm to promote the sparsity and low-rankness of the solutions. It is proved that under incoherence assumptions, the solution of the convex optimization is unique and corresponds to the ground truth. Several works also consider the problem of matrix completion under outlier noise \citep{robustmatrixcompletion1,chenfan}. Nonconvex robust PCA \citep{nonconvexrobustpca} improves the computational efficiency of principal component pursuit by proposing a nonconvex formulation and using an alternating projection algorithm to solve it. Though the formulation is nonconvex, the alternating projection algorithm is also proved to recover the ground truth exactly under incoherence and sparsity requirements. For the special case of rank-1 robust PCA,~\citet{fattahi2020exact} show that a simple sub-gradient method applies directly to the nonsmooth $\ell_1$-loss provably recovers the low-rank component and sparse noise, under the same incoherence and sparsity requirements. To model a broader set of noise, \citet{meng2013robust} consider a mixture of Gaussian noise models and exploit the EM algorithm to estimate the low-rank components. Robust PCA has found successful applications in video segmentation \citep{rpcavideo}, image processing \citep{rsubspace}, change point detection \citep{onlinechangepoint}, and many more. Nevertheless, the formulations of robust PCA focus on shared low-rank features among all data and neglect unique components. 

\textbf{Distributed matrix factorization}  The emergence of edge computation has prompted the research on distributed matrix factorization. \citet{dmf} exploits distributed gradient descent to factorize large matrices. \citet{securefmf} proposes a cryptographic framework where multiple clients use their local data to collaboratively factorize a matrix without leaking private information to the server. These works use one set of feature matrices $\matU$ and $\matV$ to fit data from all clients, thus also neglecting the feature differences from different sources as well as the possible outliers in data. Our method \name is distributed when its subroutine \inneralg is distributed. Different from conventional distributed matrix factorization, \name can find common and unique components simultaneously while remaining robust in the presence of outliers.

\textbf{Joint and individual feature extraction} The literature on using MF to identify shared and unique features abound \citep{jive,cobe, slide,bidifac, groupnmf,inmf, personalizedpca, perdl,hmf}. Among them, JIVE \citep{jive}, COBE \citep{cobe}, PerPCA \citep{personalizedpca},  and HMF \citep{hmf} uses mutually orthogonal features to model the shared and unique components. SLIDE \citep{slide} and BIDIFAC \citep{bidifac} do not pose orthogonality constraints but use regularizations to encourage the unique features to have small norms. GNMF \citep{groupnmf} and iNMF \citep{inmf} further add nonnegativity constraints to the factor matrices. In particular, PerPCA \citep{personalizedpca} and HMF \citep{hmf} are two distributed algorithms that are guaranteed to converge to the optimal solutions under proper conditions. It is worth mentioning that these methods do not account for sparse noise in the observations. In this work, we remedy this challenge by utilizing existing methods as basic building blocks for our approach, which focuses on simultaneously separating shared and unique features as well as noise components.

\textbf{Robust shared and unique feature extraction} As discussed, a few heuristic methods also attempt to find the shared and unique features when data are corrupted by large noise \citep{rjive,rajive,rcica}. Amid them, RaJIVE \citep{rajive} employs robust SVD \citep{robustsvd} to remove noise from the observations and then uses a variant of JIVE \citep{ajive} to separate common and unique components. RJIVE \citep{rjive} proposes a constrained optimization formulation to minimize the $\ell_1$ norm of the fitting residuals and exploits ADMM to solve the problem. RCICA \citep{rcica} adopts a similar optimization objective but uses a regularization to encourage the similarity of common subspaces and only works for $N=2$ cases. Though these methods can achieve decent performance in applications including facial expression synthesis and audio-visual fusion, they are based on heuristics and it is not clear whether their output converges to the ground truth common and unique factors. In contrast, we prove that \name is guaranteed to recover the ground truths and use a few numerical examples to show that \name indeed recover more meaningful components.

\section{Identifiability Conditions}\label{sec_identifiability}
Our goal is to decouple the common components, unique components, and the sparse noise, given a group of data observations $\left\{\matM_{(i)}\right\}_{i=1}^N$.
At first glance, such decoupling may seem impossible or even ill-defined: roughly speaking, the number of unknown variables, namely global components, local components, and noise, are thrice the number of observed data matrices $\left\{\matM_{(i)}\right\}_{i=1}^N$, and hence, there are infinite number of decouplings that can give rise to the same $\matM_{(i)}$.

The very first question to ask is whether such decoupling is possible and, if so, which properties can ensure the identifiability of three components. Intriguingly, we are able to prove that the exact decoupling of shared features, unique features, and noise is possible if there is ``little overlap'' among the three components. Below, we will formalize this intuition in more detail. Though intuitive, it turns out that these conditions can guarantee the \textit{identifiability} of the shared components, unique components, and the sparse noise. 

\subsection{Sparsity}

As discussed, identifying arbitrarily dense and large noise from signals is not possible. Hence, we consider sparse noise where only a small fraction of observations are corrupted. To characterize the sparsity of $\matSst_{(i)}$, we use the following definition of $\alpha$-sparsity.
\begin{definition}
\label{ass:alphasparse}
($\alpha$-sparsity) A matrix $\matS\in \mathbb{R}^{n_1\times n_2}$ is $\alpha$-sparse if at most $\alpha n_1$ entries in each column and at most $\alpha n_2$ entries in each row are nonzero.
\end{definition}
The definition follows from that of \cite{nonconvexrobustpca}. In Definition \ref{ass:alphasparse}, $\alpha$ characterizes the maximum portion of corrupted entries in each row and each column. Intuitively, if a matrix is $\alpha$-sparse with small $\alpha$, then its nonzero entries are ``spread out'' instead of concentrated on specific columns or rows. 

\subsection{Incoherence}
It is shown that distinguishing sparse components from arbitrary low-rank components is also hard \citep{robustpca,nonconvexrobustpca}. As a simple counterexample, the matrix $\matM=\vece_i\vece_j^T$, where we use $\vece_i$ to denote the basis vector of axis $i$, has its $ij$-th entry to be $1$ and all other entries to be $0$. This matrix has rank $1$, and is also sparse since it has only one nonzero entry. Thus, deciding whether it is sparse or low rank is difficult as it satisfies both requirements. 

From the above analysis, one can see that the low-rank components should not be sparse. In other words, to be distinguishable from the sparse noise, their elements should be sufficiently spread out. In the literature, this requirement is often characterized by the so-called incoherence condition \citep{robustpca,nonconvexrobustpca}.

\begin{definition}
\label{ass:incoherence}
($\mu$-incoherence) A matrix $\matU\in \mathbb{R}^{n\times r}$ is $\mu$-incoherent if 
$$
\max_i \norm{\vece_i^T\matU}_2\le \frac{\mu \sqrt{r}}{\sqrt{n}},
$$
where $\vece_i\in\mathbb{R}^n$ is the standard basis vector of axis $i$, defined as $\vece_i=(0,0,...0,1,0,...0)^T$
\end{definition}
The incoherence condition restricts the maximum row-wise $\ell_2$ norm of a matrix $\matU$, thus preventing the entries of $\matU$ from being too concentrated on a few specific axes.

Remember that in model \eqref{eqn:matrixmodel}, $\matUst_{g}\matVst_{(i),g}^T$ and $\matUst_{(i),l}\matVst_{(i),l}^T$ represent the global (shared) and local (unique) factors. For any $n\geq r$, we use $\mathbb{O}^{n\times r}$ to denote the set of $n$ by $r$ matrices whose column vectors are orthonormal, $\mathbb{O}^{n\times r}=\{\matW\in \mathbb{R}^{n\times r}|\matW^T\matW=\matI\}$.  We assume the SVD of
 $\matUst_{g}\matVst_{(i),g}^T$ and $\matUst_{(i),l}\matVst_{(i),l}^T$ has the following form, 
\begin{equation}
\label{eqn:truemodelsvd}
\left\{\begin{aligned}
&\matUst_{g}\matVst_{(i),g}^T =\matHst_g\matSigmast_{(i),g}\matWst_{(i),g}^T\\
&\matUst_{(i),l}\matVst_{(i),l}^T=\matHst_{(i),l}\matSigmast_{(i),l}\matWst_{(i),l}^T
\end{aligned}\right. ,
\end{equation} 
\revise{where $\matHst_g\in \mathbb{O}^{n_1\times r_1}$, $\matWst_{(i),g}\in \mathbb{O}^{n_{2,(i)}\times r_1}$, $\matHst_{(i),l}\in \mathbb{O}^{n_1\times r_{2,(i)}}$. Moreover, $\matWst_{(i),l}\in \mathbb{O}^{n_{2,(i)}\times r_{2,(i)}}$ are orthogonal matrices, $\matSigmast_{(i),g}\in \mathbb{R}^{r_1\times r_1}$ and $\matSigmast_{(i),l}\in \mathbb{R}^{r_{2,(i)}\times r_{2,(i)}}$ are positive diagonal matrices. In \eqref{eqn:truemodelsvd}, we consider the case where the global and local column singular vectors are orthogonal, i.e., $\matHst_g^T\matHst_{(i),l}=0$. We use $r_2=\max_ir_{2,(i)}$ throughout the paper.}

To avoid overlapping between sparse and low-rank components, we assume the row and column singular vectors $\matHst_g$, $\matHst_{(i),l}$, $\matWst_{(i),g}$, and $\matWst_{(i),l}$ are all $\mu$-incoherent. This assumption ensures that the low-rank components do not have entries too concentrated on specific rows or columns. As a result, the incoherence on singular vectors encourages the low-rank components to distribute evenly on all entries, which is distinguished from sparse noises that are nonzero on a small fraction of entries.

\subsection{Misalignment}
As discussed in \eqref{eqn:truemodelsvd}, we use orthogonality between shared and unique features $\hmatHst_g^T\hmatHst_{(i),l}=0$ to encode our prior belief about the independence of different features. This is equivalent to $\matUst_g^T\matUst_{(i),l}=0$. Such orthogonality, however, is still insufficient to guarantee the identifiability of shared and unique factors. 

To see this, consider a counterexample where all $\matUst_{(i),l}$'s are equal, i.e., $\matUst_{(1),l}=\matUst_{(2),l}=\cdots =\matUst_{(N),l}$. In this case, ``unique'' factors are also shared among all observation matrices. Thus, separating them from the ground truth $\matUst_{g}$ is not possible. From this counterexample, we can see that it is essential for the local features not to be perfectly aligned with each other. Next, we formally introduce the notion of misalignment. For a full column-rank matrix $\matU\in \mathbb{R}^{d\times n}$, we define the projection matrix $\matP_{\matU}\in \mathbb{R}^{d\times d}$ as $\Pj{\matU}=\matU\left(\matU^T\matU\right)^{-1}\matU^T$. 
\begin{definition}
\label{ass:identifiability}
($\theta$-misalignment) We say $\{\matUst_{(i),l}\}$ are $\theta$-misaligned if there exists a positive constant $\theta\in(0,1)$ such that:
\begin{equation}
\lambda_{\max}\left(\frac{1}{N}\sum_{i=1}^N\Pj{\matUst_{(i),l}}\right)\le 1-\theta.
\end{equation}
\end{definition}
By the triangular inequality of $\lambda_{\max}\left(\cdot\right)$, we know $\lambda_{\max}\left(\frac{1}{N}\sum_{i=1}^N\Pj{\matUst_{(i),l}}\right)\le \frac{1}{N}\sum_{i=1}^N\lambda_{\max}\left(\Pj{\matUst_{(i),l}}\right)=1$. Thus, the introduced $\theta$ is always nonnegative. Indeed, all $\Pj{\matUst_{(i),l}}$'s have a common nonempty eigenspace with eigenvalue $1$ if and only if $\theta=0$. Thus, the $\theta$-misalignment condition requires that the subspaces spanned by all unique factors do not contain a common subspace. On the contrary, all global features are shared; hence, the subspaces spanned by these features are also identical. This comparison shows that the misalignment condition unequivocally distinguishes unique features from shared ones.

As a concrete example, consider $N=2$ and $\matU_{(1),l}=\left( \cos \vartheta, \sin \vartheta\right)^T$, $\matU_{(2),l}=\left( \cos \vartheta, -\sin \vartheta\right)^T$ for $\vartheta\in[0,\frac{\pi}{4}]$. Indeed, the angle between $\matU_{(1),l}$ and $\matU_{(2),l}$ is $2\vartheta$ and
$$
\fracud\left(\matP_{\matU_{(1),l}}+\matP_{\matU_{(2),l}}\right)=\left(\begin{aligned}
\cos^2\vartheta, &\quad 0\\
0, &\quad \sin^2\vartheta\\
\end{aligned}\right).
$$
Hence, by definition, $\theta=\sin^2\vartheta$. We can thus clearly see that when $\vartheta$ increases, the $\matU_{(1),l}$ and $\matU_{(2),l}$ become more misaligned.

The notion of $\theta$-misalignment is first proposed by \citet{personalizedpca} and intimately related to the uniqueness conditions in \citet{jive}.

\section{Algorithm}
\label{sec:algorithm}
The introduced identifiability conditions restrict the overlaps between shared, unique, and sparse components. It remains to develop algorithms to untangle the three parts from $N$ matrices. In Section \ref{sec:optformulation}, we introduce a constrained optimization formulation, and in Section \ref{sec:altmin}, we propose an alternating minimization program to decouple the three parts. The alternating minimization requires solving subproblems to distinguish shared features from unique ones. 

Throughout the paper, we use $\norm{\matA}$ or $\norm{\matA}_2$ to denote the operator norm of a matrix $\matA\in\mathbb{R}^{m\times n}$ and $\norm{\matA}_F$ to denote the Frobenius norm of $\matA$. We use $r$ to denote $r=r_1+r_2$.

\subsection{Constrained Nonconvex Nonsmooth Optimization}
\label{sec:optformulation}
We design a constrained optimization problem to decouple the three components. The decision variables $\vecx$ include features, coefficients, and sparse noise estimates: $\vecx = \left(\matU_g,\{\matU_{(i),l},\matV_{(i),g},\matV_{(i),l},\matS_{(i)}\}_{i=1}^N\right)$. The constrained optimization is formulated as,
\begin{equation}
\label{eqn:bigobjective}
\begin{aligned}
\min_{\vecx} & \quad\sum_{i=1}^Nh_i(\matU_g,\matV_{(i),g},\matU_{(i),l},\matV_{(i),l},\matS_{(i)};\lambda)\\
\text{s.t.} &\quad \matU_g^T\matU_{(i),l}=0, \ \forall i\in[N] .
\end{aligned}
\end{equation}
Here, $h_i$ is a regularized fitting residual consisting of two parts:
\begin{subequations}
\label{eqn:hidef}
\begin{align}
&h_i(\matU_g,\matV_{(i),g},\matU_{(i),l},\matV_{(i),l},\matS_{(i)};\lambda)\notag \\
&=f_i(\matU_g,\matV_{(i),g},\matU_{(i),l},\matV_{(i),l},\matS_{(i)})+\Phi_i(\matS_{(i)};\lambda)\notag\\
    &=\fracud\norm{\matM_{(i)}-\matU_g\matV_{(i),g}^T-\matU_{(i),l}\matV_{(i),l}^T-\matS_{(i)}}_F^2   \label{eqn:hiterma}\tag{$f_i$}\\
    &+\lambda^2\norm{\matS_{(i)}}_0. \label{eqn:hitermc}\tag{$\Phi_i$}
\end{align}
\end{subequations} 
Term~\eqref{eqn:hiterma} measures the distance between the sum of shared, unique, and sparse components and the observation matrix $\matM_{(i)}$. It denotes the residual of fitting. A common approach for solving this problem is based on convex relaxation \citep{robustpca}. However, convex relaxation increases the number of variables to $\mathcal{O}(n_1n_2)$, while our nonconvex formulation keeps it in the order of $\mathcal{O}(\max\{n_1,n_2\}(r_1+r_2))$, which is significantly smaller. 

Term~\eqref{eqn:hitermc} is an $\ell_0$ regularization term that promotes the sparsity of matrix $\matS_{(i)}$. The parameter $\lambda$ mediates the balance between the $\ell_0$ penalty with the residual of fitting. A large value of $\lambda$ leads to sparser $\matS_{(i)}$ with only large nonzero elements. Conversely, a small value of $\lambda$ yields a denser $\matS_{(i)}$ with potentially small nonzero elements. Therefore, to identify both large and small nonzero values of $\matS_{(i)}$, while correctly filtering out its zero elements, we propose to gradually decrease the value of $\lambda$ during the optimization of objective \eqref{eqn:bigobjective}. We use the notation $h_i(\matU_g,\matV_{(i),g},\matU_{(i),l},\matV_{(i),l},\matS_{(i)};\lambda)$ to explicitly show that the objective $h_i$ is dependent on the regularization parameter $\lambda$.

At first glance, the proposed optimization problem~\eqref{eqn:bigobjective} may appear daunting due to its inherent nonconvexity and nonsmoothness. Notably, it exhibits two distinct sources of nonconvexity: firstly, both terms \eqref{eqn:hiterma} and \eqref{eqn:hitermc} are nonconvex, and secondly, the feasible set corresponding to the constraint $\matU_g^T\matU_{(i),l}=0$ is also nonconvex. Furthermore, the $\ell_0$ regularization term in \eqref{eqn:hitermc} introduces nonsmoothness into the problem. However, we will introduce an intuitive and efficient algorithm designed to alleviate these challenges and effectively solve the problem. Surprisingly, under our identifiability conditions introduced in Section \ref{sec_identifiability}, this algorithm can be proven to converge to the ground truth.

\subsection{Alternating Minimization}
\label{sec:altmin}
One efficient approach to solving a $\ell_0$ regularized objective is alternating minimization. We divide the decision variables $\vecx$ into $2$ blocks, $\left(\matU_g,\{\matV_{(i),g},\matU_{(i),l},\matV_{(i),l}\}\right)$ and $\left(\{\matS_{(i)}\}\right)$, and alternatively minimize one block with the block of variables fixed. 

More specifically, the alternating minimization proceeds by epochs, each comprised of two steps. For ease of exposition, we use $\hmatU_{g,t-1},\{\hmatV_{(i),g,t-1},\hmatU_{(i),l,t-1},\hmatV_{(i),l,t-1},\hmatS_{(i),t-1}\}_{i=1}^N$ to denote the values of $\vecx$ at the end of epoch $t-1$. The $\hat{}$ notation represents the estimated values of the variables. 

In the first step, we fix the values of $\left(\hmatU_{g,t-1},\{\hmatV_{(i),g,t-1},\hmatU_{(i),l,t-1},\hmatV_{(i),l,t-1}\}\right)$, and optimize over $\{\matS_{(i)}\}$. The optimal $\hmatS_{(i),t}$ has a simple closed-form solution given by hard-thresholding,
$$
\begin{aligned}
\hmatS_{(i),t}&=\arg\min_{\matS_{(i)}}\norm{\matM_{(i)}-\hmatU_{g,t-1}\hmatV_{(i),g,t-1}^T-\hmatU_{(i),l,t-1}\hmatV_{(i),l,t-1}^T-\matS_{(i)}}_F^2+\lambda_t^2\norm{\matS_{(i)}}_0\\
&=\text{Hard}_{\lambda_t}\left[\matM_{(i)}-\hmatU_{g,t-1}\hmatV_{(i),g,t-1}^T-\hmatU_{(i),l,t-1}\hmatV_{(i),l,t-1}^T\right],
\end{aligned}
$$
where $\text{Hard}_{\lambda}(\cdot)$ is the hard-thresholding operator. For a matrix $X\in \mathbb{R}^{m\times n}$, the hard thresholding operator is defined as:
\begin{equation}
\label{eqn:defofsoftthresholding}
\left[\text{Hard}_{\lambda}\left(X\right)\right]_{ij} = \left\{
\begin{aligned}
& X_{ij} \text{ ,  if  }\abs{X_{ij}} > \lambda\\
& 0 \text{ ,  if  }X_{ij}\in [-\lambda,\lambda]\\
\end{aligned}\right. .
\end{equation}

The coefficient $\lambda$ is a thresholding parameter that controls the sparsity of the output. To recover the correct sparsity pattern of $\hmatS_{(i),t}$, our approach is to maintain a small false positive rate (elements that are incorrectly identified as nonzero), while gradually improving the true positive rate (elements that are correctly identified as nonzero). To this goal, we start with a large $\lambda$ to obtain a conservative estimate of $\hmatS_{(i),t}$. Then, we decrease $\lambda$ to refine the estimate.

In the second step, we fix $\hmatS_{(i),t}$ and optimize $\left(\matU_g,\{\matV_{(i),g},\matU_{(i),l},\matV_{(i),l}\}\right)$ under the constraint $\matU_g^T\matU_{(i),l}=0$. Removing the $\ell_0$ regularization term that is independent of $\left(\matU_g,\{\matV_{(i),g},\matU_{(i),l},\matV_{(i),l}\}\right)$, the optimization subproblem takes the following form,
\begin{equation}
\label{eqn:subproblemobj}
\begin{aligned}
\min_{\left(\matU_g,\{\matV_{(i),g},\matU_{(i),l},\matV_{(i),l}\}\right)} &\quad \sum_{i=1}^N\norm{\hmatM_{(i)}-\matU_g\matV_{(i),g}^T-\matU_{(i),l}\matV_{(i),l}^T}_F^2 \\
\text{s.t.}& \quad  \matU_g^T\matU_{(i),l}=0, \forall i\in [N],
\end{aligned}
\end{equation}
where $\hmatM_{(i)} = \matM_{(i)}-\hmatS_{(i),t}$.

Despite its nonconvexity, there exist several iterative algorithms to solve the above optimization problem, including but not limited to JIVE \citep{jive}, COBE \citep{cobe}, PerPCA \citep{personalizedpca}, PerDL \citep{perdl}, and HMF \citep{hmf}. Given the similarity of these methods, we employ the name Joint and Individual Matrix Factorization (\inneralg) to encapsulate the subroutine addressing problem \eqref{eqn:subproblemobj}.

The versatile \inneralg is a meta-algorithm that can be implemented using any of the aforementioned methods, provided that they generate good solutions. Among these algorithms, PerPCA and HMF are of special interest as they are proved to converge to the optimal solutions of \eqref{eqn:subproblemobj} under suitable conditions. They are also intrinsically federated as most of the computation can be distributed on $N$ sources where the data are generated. 

As problem \eqref{eqn:subproblemobj} does not have a simple closed-form solution, the algorithms discussed above are iterative. The iterative algorithms do not output exact optimal solutions. Instead, they refine the estimates at every iteration. Therefore, there will be a difference between our algorithm-generated solutions and the optimal solution. To characterize the degree of such difference, we resort to employing the concept of $\epsilon$-optimality, a notion well-known in the optimization community.

\begin{definition} 
($\epsilon$-optimality) Given $(\hmatU_g,\{\hmatV_{(i),g}, \hmatU_{(i),l},\hmatV_{(i),l}\})$ as any global optimal solution to the problem \eqref{eqn:subproblemobj} and a constant $\epsilon>0$, we say $(\hmatUeps_g,\{\hmatVeps_{(i),g}, \hmatUeps_{(i),l},\hmatVeps_{(i),l}\})$ is an $\epsilon$-optimal solution to \eqref{eqn:subproblemobj} if it satisfies,
\begin{equation*}
\norm{\hmatUeps_g\hmatVeps_{(i),g}^T+\hmatUeps_{(i),l}\hmatVeps_{(i),l}^T-\hmatU_g\hmatV_{(i),g}^T-\hmatU_{(i),l}\hmatV_{(i),l}^T}_{\infty}\le \epsilon,\quad \forall i
\end{equation*}
and 
\begin{equation*}
\hmatUeps_g^T\hmatUeps_{(i),l} = 0,\quad \forall i .
\end{equation*}
\end{definition}

The nonconvexity of \eqref{eqn:subproblemobj} gives rise to multiple global optimal solutions. Our definition of $\epsilon$-optimality only emphasizes the closeness between the \textit{product} of features and the coefficients, and the product of \textit{any} set of global optimal solutions. As discussed, there exist multiple methods proposed to solve \eqref{eqn:subproblemobj} that demonstrate decent practical performance. In particular, PerPCA, PerDL, and HMF are proved to converge to the optimal solutions of \eqref{eqn:subproblemobj} at linear rates when initialized properly. Hence, under suitable initializations,  PerPCA and HMF can reach $\epsilon$-optimality of \eqref{eqn:subproblemobj} within $\mathcal{O}\left(\log\frac{1}{\epsilon}\right)$ iterations for any value of $\epsilon$. The details of the two algorithms will be discussed in Appendix \ref{sec:hmfintro}.

With the help of subroutine \inneralg, the main alternating minimization algorithm proceeds by optimizing two blocks of variables iteratively. We present the pseudo-code in Algorithm \ref{alg:altmin}.

\begin{algorithm}
\caption{\name: alternating minimization}
\label{alg:altmin}
\begin{algorithmic}[1]
\STATE Input observation matrices from $N$ sources $\{\matM_{(i)}\}_{i=1}^N$, constant $\lambda_1$, multiplicative factor $\rho \in (0,1)$, precision $\epsilon$.
\STATE Initialize $\hmatUeps_{g,0}, \hmatVeps_{(i),g,0},\hmatUeps_{(i),l,0},\hmatVeps_{(i),l,0},\hmatS_{(i),0}$ to be zero matrices.
\FOR{Epoch $t=1,...,T$}
\FOR{Source $i=1,\cdots,N$}
\STATE $\hmatS_{(i),t}=\hard{\lambda_{t}}{\matM_{(i)}-\hmatUeps_{g,t-1}\hmatVeps_{(i),g,t-1}^T-\hmatUeps_{(i),l,t-1}\hmatVeps_{(i),l,t-1}^T}$
\ENDFOR
\STATE $(\hmatUeps_{g,t},\{\hmatVeps_{(i),g,t}\},\{\hmatUeps_{(i),l,t}\},\{\hmatVeps_{(i),l,t}\})=\text{\inneralg}\left(\{\hmatM_{(i)}\}=\{\matM_{(i)}-\hmatS_{(i),t}\},\epsilon\right)$
\STATE Set $\lambda_{t+1}=\rho \lambda_{t}+\epsilon$
\ENDFOR
\STATE Return $\{\hmatUeps_{g,T},\{\hmatVeps_{(i),g,T}\},\{\hmatUeps_{(i),l,T}\},\{\hmatVeps_{(i),l,T}\}\}$.
\end{algorithmic}
\end{algorithm}

In Algorithm \ref{alg:altmin},  we use $\inneralg\left(\{\hmatM_{(i)}\},\epsilon\right)$ to denote the call for a subroutine to solve \eqref{eqn:subproblemobj} to $\epsilon$-optimality. In each epoch, sparse matrices $\hmatS_{(i),t}$ are firstly estimated by hard thresholding. Then $\hmatM_{(i)}=\matM_{(i)}-\hmatS_{(i),t}$ are calculated, which are subsequently decoupled into the shared and unique components via a \inneralg call. The output of this subroutine is represented as $(\hmatUeps_{g,t},\{\hmatVeps_{(i),g,t}\},\{\hmatUeps_{(i),l,t}\},\{\hmatVeps_{(i),l,t}\})$, where the superscript ${}^{\epsilon}$ signifies $\epsilon$-optimality. The outputs $(\hmatUeps_{g,t},\{\hmatVeps_{(i),g,t}\},\{\hmatUeps_{(i),l,t}\},\{\hmatVeps_{(i),l,t}\})$ are used to improve the estimate of $\hmatS_{(i)}$ in the next epoch. After each epoch, we decrease the thresholding parameter $\lambda_t$ by a constant $\rho<1$, then add a constant $\epsilon$. The inclusion of $\epsilon$ in $\lambda_{t+1}$ is necessary to ensure that the estimated $\hmatS_{(i),t+1}$ does not contain any false positive entries. By incorporating $\epsilon$ into $\lambda_{t+1}$, we guarantee that the inexactness of the \inneralg outputs does not undermine the false positive rate of the entries in $\hmatS_{(i),t+1}$.

\revise{ Then, per-epoch computational complexity of Algorithm~\ref{alg:altmin} is $O(n_1n_2N)$ when the rank $r_1$ and $r_2$ are small $r_1,r_2\ll n_1,n_2$. To see this, we can add up the computation complexity complexity of hard-thresholding and \inneralg. Element-wise hard-thresholding requires $O(n_1n_2)$ computations for each source. Efficient implementations of the subroutine \inneralg, such as \perpca and \hmf, can converge into the $\epsilon-$optimal solutions within $\mathcal{O}(\log\frac{1}{\epsilon})$ steps, where each step require $\mathcal{O}(n_1n_2)$ computations. Therefore, the per-epoch computational complexity of \name is $\mathcal{O}(n_1n_2N)$, where log factors are omitted.}

\revise{Furthermore, if the \inneralg and hard-thresholding are distributed among $N$ sources, \name can further exploit parallel computation to reduce the running time. In the regime where communication cost is negligible, the per-iteration total running time scales as $\mathcal{O}\left(n_1n_2+Nn_1\right)$.}

We will show later that such a design can ensure that the estimation error diminishes linearly. A pictorial representation of Algorithm \ref{alg:altmin} is plotted in the left graph of Figure \ref{fig:difficulty}. 

\section{Convergence Analysis}
\label{sec:convergence}
In this section, we will analyze the convergence of Algorithm \ref{alg:altmin}. Our theorem characterizes the conditions under which Algorithm \ref{alg:altmin} converges linearly to the ground truth. \revise{We additionally introduce $\gop >0$ to denote an upper bound of the singular values of $\{\matUst_g\matVst_{(i),g}^T+\matUst_{(i),l}\matVst_{(i),l}^T\}_{i=1}^N$, and $\sm >0$ to denote a lower bound on the smallest nonzero singular values of $\{\matUst_g\matVst_{(i),g}^T+\matUst_{(i),l}\matVst_{(i),l}^T\}_{i=1}^N$. For simplicity we assume $n_{2,(i)}=n_2$, $r_{2,(i)}=r_2$, and $r=r_1+r_2$ in this section.}

\begin{theorem}[Convergence of Algorithm~\ref{alg:altmin}]
\label{thm:recovery}
Consider the true model \eqref{eqn:matrixmodel} with SVD defined in \eqref{eqn:truemodelsvd}, \revise{where nonzero singular values of $\matUst_g\matVst_{(i),g}^T+\matUst_{(i),l}\matVst_{(i),l}^T$ are lower bounded by $\sm>0$ and upper bounded by $\gop\ge \sm$ for each source $i$}. Suppose that the following conditions are satisfied:
\begin{itemize}
    \item {\bf ($\pmb{\mu}$-incoherency)} The matrices $\matHst_g$ and $\left\{\matHst_{(i),l},\matWst_{(i),g}, \matWst_{(i),l}\right\}_{i=1}^N$ are $\mu$-incoherent for a constant $\mu>0$. 
    \item {\bf ($\pmb{\theta}$-misalignment)} The local feature matrices $\left\{\matUst_{(i),l}\right\}_{i=1}^N$ are $\theta$-misaligned for a constant $0<\theta<1$.
    \item {\bf ($\pmb{\alpha}$-sparsity)} \revise{The matrices $\left\{\matSst_{(i)}\right\}_{i=1}^N$'s are $\alpha$-sparse for some $\alpha = \mathcal{O}\left(\frac{\theta}{\mu^4r^2}\right)$, where $r=r_1+r_2$.}
\end{itemize}
Then, there exist constants $C_{g,1}, C_{g,2}, C_{l,1}, C_{l,2}, C_{s,1}, C_{s,2}>0$ and $\rhom=\mathcal{O}\left(\sqrt{\alpha}\frac{\mu^2r}{\sqrt{\theta}}\right)<1$ such that the iterations of Algorithm \ref{alg:altmin} with $\lambda_1= \frac{\gop \mu^2r}{\sqrt{n_1n_2}}$, $\epsilon\le \lambda_1\left(1-\rhom\right)$, and $1-\frac{\epsilon}{\lambda_1}>\rho\ge\rhom$ satisfy
\begin{align}
    \label{eqn:convergenceonlg}
\norm{\hmatUeps_{g,t}\hmatVeps_{(i),g,t}^T-\matUst_{g}\matVst_{(i),g}^T}_{\infty}&\leq  C_{g,1}\rho^t+C_{g,2}\epsilon
\\
\label{eqn:convergenceonll}
\norm{\hmatUeps_{(i),l,t}\hmatVeps_{(i),l,t}^T-\matUst_{(i),l}\matVst_{(i),l}^T}_{\infty} &\leq  C_{l,1}\rho^t+C_{l,2}\epsilon
\\
\label{eqn:convergenceons}
\norm{\hmatS_{(i),t}-\matSst_{(i)}}_{\infty}&\leq  C_{s,1}\rho^t+C_{s,2}\epsilon.
\end{align}
\end{theorem}

\revise{Theorem \ref{thm:recovery} presents a set of sufficient conditions under which the model is identifiable, and Algorithm \ref{alg:altmin} converges to the ground truth at a linear rate. As discussed in Section~\ref{sec_identifiability}, these conditions are indeed sufficient to guarantee the identifiability of the true model. In particular, $\mu$-incoherency is required for disentangling global and local components from noise, whereas $\theta$-misalignment is needed to separate local and global components. Moreover, there is a natural trade-off between the parameters $\mu$, $\theta$, and $\alpha$: the upper bound on the sparsity level $\alpha$, $ \mathcal{O}\left(\frac{\theta^2}{\mu^4r^2}\right)$, is proportional to $\theta^2$, implying that more alignment among local feature matrices can be tolerated only at the expense of sparser noise matrices. Similarly, $\alpha$ is inversely proportional to $\mu^4$, indicating that more coherency in the local and global components is only possible with sparser noise matrices. We also highlight the dependency of $\alpha$ on the rank $r$; such dependency is required even in the standard settings of robust PCA~\citep{nonconvexrobustpca, chandrasekaran2011rank, hsu2011robust}, albeit with a milder condition on $r$. Finally, the scaling of $\alpha$ does not depend on the the number of sources $N$, suggesting that the convergence guarantees provided by Theorem~\ref{thm:recovery} are valid for extremely large datasets.} 

Two important observations are in order. First, we do not impose any constraint on the norm or sign of the sparse noise $\matSst_{(i)}$. Thus, Theorem~\ref{thm:recovery} holds for arbitrarily large noise values. Second, at every epoch, Algorithm~\ref{alg:altmin} solves the inner optimization problem~\eqref{eqn:subproblemobj} via \inneralg to $\epsilon$-optimality. Also, the convergence of Algorithm~\ref{alg:altmin} is contingent upon the precision of \inneralg output: the $\ell_{\infty}$ norm of the optimization error should not be larger than $\mathcal{O}\left(\lambda_1\right)$. Such requirement is not strong as even the trivial solution $\matU_g,\matV_{(i),g},\matU_{(i),l},\matV_{(i),l}=0$ is $\lambda_1$-optimal. One should expect many algorithms to perform much better than the trivial solution. Indeed, methods including PerPCA and PerDL are proved to output $\epsilon$-optimal solutions for arbitrary small $\epsilon$ within logarithmic iterations, thus satisfying the requirement. In practice, heuristic methods including JIVE or COBE can output reasonable solutions that may also satisfy the requirement in Theorem~\ref{thm:recovery}.

In the next section, we provide the sketch of the proof for Theorem~\ref{thm:recovery}.

\subsection{Proof Sketch of Theorem \ref{thm:recovery}}
\label{sec:recoverysketch}
Algorithm \ref{alg:altmin} is essentially an alternating minimization algorithm comprising a hard-thresholding step, followed by a joint and individual matrix factorization step. Our overarching goal is to control the estimation error at each iteration of the algorithm, showing that it decreases by a constant factor after every epoch. 
To this goal, we make extensive use of the error matrix $\matE_{(i),t}$ defined as $\matE_{(i),t}=\matSst_{(i)}-\hmatS_{(i),t}$ for every client $i$. 

In the ideal case where $\matE_{(i),t} = 0$, the global solution of~\eqref{eqn:bigobjective} coincides with the true shared and unique components, which is guaranteed by Theorem 1 in~\citet{personalizedpca}. Therefore, it is crucial to control the behavior of $\{\matE_{(i),t}\}_{i=1}^N$ and its effect on the recovered solution throughout the course of the algorithm. We define $\matL^\star_{(i)} = \matUst_{g}\matVst_{(i),g}^T+\matUst_{(i),l}\matVst_{(i),l}^T$ and $\hat\matL_{(i),t} = \hmatU_{g,t}\hmatV_{(i),g,t}^T+\hmatU_{(i),l,t}\hmatV_{(i),l,t}^T$ as the true and estimated low-rank components of client $i$. Similarly, $\hmatLeps_{(i),t} = \hmatUeps_{g,t}\hmatVeps_{(i),g,t}^T+\hmatUeps_{(i),l,t}\hmatVeps_{(i),l,t}^T$ is the reconstructed low-rank component from $\epsilon$-optimal estimates. The following steps outline the sketch of our proof:

\paragraph{Step 1: $\alpha$-sparsity of the initial error:} At the first iteration, the threshold level $\lambda_1$ is large, enforcing $\supp{\hmatS_{(i),1}} \subseteq \supp{\matSst_{(i)}}$, which in turn implies $\supp{\matE_{(i),1}} \subseteq \supp{\matSst_{(i)}}$. Therefore, the initial error matrix $\matE_{(i),1}$ is also $\alpha$-sparse. 
 
\paragraph{Step 2: Error reduction via \inneralg.} Suppose that $\matE_{(i),t}$ is $\alpha$-sparse. In Step 7, $\hat\matL_{(i),t}$ is obtained by applying \inneralg on $\matM_{(i)}-\hmatS_{(i),t}=\matL^\star_{(i)}+\matE_{(i),t}$. Note that the input to \inneralg is the true low-rank component perturbed by an $\alpha$-sparse matrix $\matE_{(i),t}$. One of our key contributions is to show that $\norm{\matL^\star_{(i)}-\hmatL_{(i),t}}_\infty$ is much smaller than $\norm{\matE_{(i),t}}_\infty$, provided that the true local and global components are $\mu$-incoherent and $\matE_{(i),t}$ is $\alpha$-sparse. This fact is delineated in the following key lemma.

\begin{lemma}[Error reduction via \inneralg (informal)]
\label{lm:informalldiffinfnorm}
Suppose that the conditions of Theorem~\ref{thm:recovery} are satisfied. Moreover, suppose that $\matE_{(i),t}$ is $\alpha$-sparse for each client $i$. We have
\revise{$$
\norm{\matL^\star_{(i)}-\hmatL_{(i),t}}_{\infty}\leq C\cdot\frac{\sqrt{\alpha} \mu^2 r}{\sqrt{\theta}} \cdot \max_j\left\{\norm{\matE_{(j),t}}_{\infty}\right\},
$$}
where $C>0$ is a constant.
\end{lemma}

Indeed, proving Lemma \ref{lm:informalldiffinfnorm} is particularly daunting since $\hmatL_{(i),t}$ does not have a closed-form solution. We will elaborate on the major techniques to prove Lemma \ref{lm:informalldiffinfnorm} in Section \ref{sec:proofsketchofldiffinfnorm}.

Suppose that $\alpha$ is small enough such that $C\cdot\frac{\sqrt{\alpha} \mu^2 r}{\sqrt{\theta}}\leq  \frac{\rho}{2}$. Then, Lemma \ref{lm:informalldiffinfnorm} implies that $\norm{\matL^\star_{(i)}-\hmatL_{(i),t}}_{\infty}\leq \frac{\rho}{2}\max_i\left\{\norm{\matE_{(i),t}}_{\infty}\right\}$. From the definition of $\epsilon$-optimality, we know $\norm{\matL^\star_{(i)}-\hmatLeps_{(i),t}}_{\infty}\leq \frac{\rho}{2}\max_i\left\{\norm{\matE_{(i),t}}_{\infty}\right\}+\epsilon$. This implies that the $\ell_{\infty}$ norm of the error in the output of \inneralg shrinks by a factor of $\frac{\rho}{2}$ compared with the error in the input ${\norm{\matE_{(i),t}}}_{\infty}$ (modulo an additive factor $\epsilon$). As will be discussed next, this shrinkage in the $\ell_{\infty}$ norm of the error is essential for the exact sparsity recovery of the noise matrix.

\paragraph{Step 3: Preservation of sparsity via hard-thresholding.} Given that $\norm{\matL^\star_{(i)}-\hmatLeps_{(i),t}}_{\infty}\leq \frac{\rho}{2}\max_i\left\{\norm{\matE_{(i),t}}_{\infty}\right\}+\epsilon$, our next goal is to show that $\supp{\matE_{(i),t+1}}\subseteq\supp{\matSst_{(i)}}$ (i.e., $\matE_{(i),t+1}$ remains $\alpha$-sparse) and $\max_i\left\{\norm{\matE_{(i),t+1}}_{\infty}\right\}\leq  2\lambda_{t+1}$. To prove $\supp{\matE_{(i),t+1}}\subseteq\supp{\matSst_{(i)}}$, suppose that $\left(\matS_{(i)}^\star\right)_{kl}=0$ for some $(k,l)$, we have $\left(\hmatS_{(i), t+1}\right)_{kl}\not=0$ only if $\left|\left(\matM_{(i)}-\hmatLeps_{(i),t}\right)_{kl}\right|=\left|\left(\matL_{(i)}^\star-\hmatLeps_{(i),t}\right)_{kl}\right|>\lambda_{t+1}$.
On the other hand, in the Appendix, we show that $\max_i\left\{\norm{\matE_{(i),t}}_{\infty}\right\}\leq 2\lambda_t$. This implies that $\norm{\matL^\star_{(i)}-\hmatLeps_{(i),t}}_{\infty}\leq \frac{\rho}{2}\max_i\left\{\norm{\matE_{(i),t}}_{\infty}\right\}+\epsilon \leq \rho\lambda_t+\epsilon = \lambda_{t+1}$. This in turn leads to $\left(\hmatS_{(i), t+1}\right)_{kl}=\left(\matE_{(i), t+1}\right)_{kl}=0$, and hence, $\supp{\matE_{(i),t+1}}\subseteq\supp{\matSst_{(i)}}$. Finally, according to the definition of hard-thresholding, we have $\left|\left(\hmatS_{(i),t+1}-\left(\matS^\star_{(i)}+\matL^\star_{(i)}-\hmatLeps_{(i)}\right)\right)_{kl}\right|\leq \lambda_{t+1}$, which, by triangle inequality, yields $\left|\left(\matE_{(i),t+1}\right)_{kl}\right|\leq \left|\left(\matL^\star_{(i)}-\hmatLeps_{(i)}\right)_{kl}\right|+\lambda_{t+1}\leq 2\lambda_{t+1}$.

\paragraph{Step 4: Establishing linear convergence.} Repeating Steps 2 and 3, we have $\max_i\left\{\norm{\matE_{(i),t+1}}_{\infty}\right\}\leq 2\lambda_{t+1}$ and $\norm{\matL^\star_{(i)}-\hmatLeps_{(i),t}}_{\infty}\leq\lambda_{t+1}$ for all $t$. Noting that $\lambda_t = \rho\lambda_{t-1}+\epsilon = \epsilon+\rho\epsilon + \rho^2\lambda_{t-2}=\cdots = \epsilon+\rho\epsilon+\rho^3\epsilon+\cdots + \rho^{t-1}\lambda_1\leq \frac{\epsilon}{1-\rho} + \rho^{t-1}\lambda_1$, we establish that $\max_i\left\{\norm{\matE_{(i),t}}_{\infty}\right\} = \mathcal{O}(\epsilon)$ and $\norm{\matL^\star_{(i)}-\hmatLeps_{(i),t}}_{\infty} = \mathcal{O}(\epsilon)$ in $\mathcal{O}\left(\frac{\log(\lambda_1/\epsilon)}{\log(1/\rho)}\right)$ iterations.

\paragraph{Step 5: Untangling global and local components.} Under the misalignment condition, a small error of the joint low-rank components $\norm{\matL^\star_{(i)}-\hmatL_{(i),t}}_F$ indicates that both the shared component and the unique component is small. More specifically, Theorem 1 in \citet{personalizedpca} indicates $\norm{\matUst_g\matVst_{(i),g}^T-\hmatU_{g,t}\hmatV_{(i),g,t}^T}_F, \norm{\matUst_{(i),l}\matVst_{(i),l}^T-\hmatU_{(i),l,t}\hmatV_{(i),l,t}^T}_F = \mathcal{O}\left(\norm{\matLst_{(i)}-\hmatL_{(i),l,t}}_F\right)$. Since $\norm{\matLst_{(i)}-\hmatL_{(i),l,t}}_F $ shrinks linearly to a small constant, we can conclude that the estimation errors for shared and unique features also decrease linearly to $\mathcal{O}(\epsilon)$.

\subsection{Proof of Lemma~\ref{lm:informalldiffinfnorm}}
\label{sec:proofsketchofldiffinfnorm}
At the crux of our proof for Theorem~\ref{thm:recovery} lies Lemma~\ref{lm:informalldiffinfnorm}. In its essence, Lemma~\ref{lm:informalldiffinfnorm} seeks to answer the following question: if the input to \inneralg is corrupted by $\alpha$-sparse noise matrices $\{\matE_{(i)}\}$, how will the recovered solutions change in terms of $\ell_{\infty}$ norm? We highlight that the standard matrix perturbation analysis, such as the classical Davis-Kahan bound \citep{matrixanalysis} as well as the more recent $\ell_{\infty}$ bound \citep{linfboundeigenvector}, fall short of answering this question for two main reasons. First, these bounds are overly pessimistic and cannot take into account the underlying sparsity structure of the noise. Second, they often control the singular vectors and singular values of the perturbed matrices, whereas the optimal solutions to problem \eqref{eqn:subproblemobj} generally do not correspond to the singular vectors of $\hmatM_{(i)}$.

To address these challenges, we characterize the optimal solutions of \eqref{eqn:subproblemobj} by analyzing its Karush–Kuhn–Tucker (KKT) conditions. We establish the KKT condition and ensure the linear independence constraint qualification (LICQ). Afterward, we obtain closed-form solutions for the KKT conditions in the form of convergent series and use these series to control the element-wise perturbation of the solutions.

\paragraph{KKT conditions.} 

The following lemma shows two equivalent formulations for the KKT conditions. For convenience, we drop the subscript $t$ in our subsequent arguments.

\begin{lemma}
\label{lm:kktcondition}
Suppose that $\{\hmatU_g$, $\hmatU_{(i),l}$, $\hmatV_{(i),g}$, $\hmatV_{(i),l}\}$ is the optimal solution to problem \eqref{eqn:subproblemobj} and $\hmatM_{(i)}$ has rank at least $r_1+r_2$. We have
\begin{subequations}
\label{eqn:kktinnerloop1}
\begin{align}
\label{subeqn:kktinnerloopug}\sum_{i=1}^N\left(\hmatU_g\hmatV_{(i),g}^T+\hmatU_{(i),l}\hmatV_{(i),l}^T-\hmatM_{(i)}\right)\hmatV_{(i),g}&=0\\
\label{subeqn:kktinnerloopuil}\left(\hmatU_g\hmatV_{(i),g}^T+\hmatU_{(i),l}\hmatV_{(i),l}^T-\hmatM_{(i)}\right)\hmatV_{(i),l}&=0\\
\label{subeqn:kktinnerloopvil}\left(\hmatU_g\hmatV_{(i),g}^T+\hmatU_{(i),l}\hmatV_{(i),l}^T-\hmatM_{(i)}\right)^T\hmatU_{(i),l}&=0\\
\label{subeqn:kktinnerloopvig}\left(\hmatU_g\hmatV_{(i),g}^T+\hmatU_{(i),l}\hmatV_{(i),l}^T-\hmatM_{(i)}\right)^T\hmatU_{g}&=0\\
\label{subeqn:kktinnerloopxonstraint}\hmatU_{(i),l}^T\hmatU_{(i),l}=\matI, \hmatU_{g}^T\hmatU_{g}=\matI, \hmatU_{(i),l}^T\hmatU_{g}&=0.
\end{align}
\end{subequations}
Moreover, 
there exist positive diagonal matrices $\matLambda_{1}\in \mathbb{R}^{r_1\times r_1}$, $\matLambda_{2,(i)}\in \mathbb{R}^{r_2\times r_2}$, and $\matLambda_{3,(i)}\in \mathbb{R}^{r_1\times r_2}$ such that
the optimality conditions imply:
\begin{subequations}
\label{eqn:kktderivations}
\begin{align}
&\label{subeqn:kkthil}\hmatM_{(i)}\hmatM_{(i)}^T\hmatH_{(i),l}=\hmatH_{(i),l}\bm{\Lambda}_{2,(i)}+\hmatH_g\bm{\Lambda}_{3,(i)}\\
&\label{subeqn:kkthg}\frac{1}{N}\sum_{i=1}^N\hmatM_{(i)}\hmatM_{(i)}^T\hmatH_g=\hmatH_g\bm{\Lambda}_{1}+\frac{1}{N}\sum_{i=1}^N\hmatH_{(i),l}\bm{\Lambda}_{3,(i)}^T\\
&\label{subeqn:kktconstraint}\hmatH_g^T\hmatH_g=\mathbf{I},\hmatH_{(i),l}^T\hmatH_{(i),l}=\mathbf{I},\hmatH_g^T\hmatH_{(i),l}=0,
\end{align}
\end{subequations}
for some $\hmatH_{g}\in \mathbb{O}^{n_1\times r_1}$ that spans the same subspaces as $\hmatU_g$, and some $\hmatH_{(i),l}\in \mathbb{O}^{n_1\times r_2}$ that spans the same subspaces as $\hmatU_{(i),l}$.
\end{lemma}

The $\matLambda_{3,(i)}$ term in \eqref{eqn:kktderivations} complicates the relation between $\hmatH_{g}$ and $\hmatH_{(i),l}$. When $\matLambda_{3,(i)}$ is nonzero, one can see that neither $\hmatH_{g}$ nor $\hmatH_{(i),l}$ span a invariant subspace of $\hmatM_{(i)}\hmatM_{(i)}^T$. As a consequence, perturbation analysis from \cite{nonconvexrobustpca} based on characteristic equations is not applicable. To alleviate this issue, we provide a more delicate control over the solution set of \eqref{eqn:kktderivations}. 

\paragraph{Solutions to KKT conditions}
The characterization \eqref{eqn:kktderivations} contains structural information for $\hmatH_g$ and $\hmatH_{(i),l}$ that can be exploited for the perturbation analysis. To see this, recall the definition $\hat\matM_{(i)}=\matM_{(i)}-\hat\matS_{(i)}$. Combining this definition with \eqref{eqn:kktderivations} leads to

\begin{equation}
\label{eqn:informalsimplifiedkkt}
\left\{
\begin{aligned}
&\left(\matE_{(i),t}\matLst_{(i)}^T+\matLst_{(i)}\matE_{(i),t}^T+\matE_{(i),t}\matE_{(i),t}^T\right)\hmatH_{(i),l}-\hmatH_{(i),l}\matLambda_{2,(i)} = -\left(\matLst_{(i)}\matLst_{(i)}^T \hmatH_{(i),l}-\hmatH_g\matLambda_{3,(i)}\right)\\
&\frac{1}{N}\sum_{i=1}^N\left(\matE_{(i),t}\matLst_{(i)}^T+\matLst_{(i)}\matE_{(i),t}^T+\matE_{(i),t}\matE_{(i),t}^T\right)\hmatH_{g}-\hmatH_{g}\matLambda_{1} \\
&\hspace{7cm}= -\left(\frac{1}{N}\sum_{i=1}^N\matLst_{(i)}\matLst_{(i)}^T \hmatH_{g}-\frac{1}{N}\sum_{i=1}^N\hmatH_{(i),l}\matLambda_{3,(i)}\right).
\end{aligned}\right.
\end{equation}

We can show that the norms of input errors $\matE_{(i),t}$ $\matLambda_{3,(i)}$ are upper bounded by $\mathcal{O}(\sqrt{\alpha})$. Thus, when the sparsity parameter $\alpha$ is not too large, we can write the solutions to \eqref{eqn:informalsimplifiedkkt} as a series of $\alpha$. 

In the limit $\alpha=0$, we have $\matE_{(i),t}=0$, thus we can solve the leading terms of $\hmatH_{(i),l}$ and $\hmatH_{g}$ from~\eqref{eqn:informalsimplifiedkkt}. When $\alpha$ is not too large, we can prove the following lemma,
\revise{
\begin{lemma}
\label{lm:informalhpertuabation}
(informal) If $\alpha$ is not too large, then $\hmatH_g$ and $\hmatH_{(i),l}$ introduced in Lemma \ref{lm:kktcondition} satisfy,
\begin{equation}
\label{eqn:hhatsolutioninformal}
\left\{\begin{aligned}
&\hmatH_{g} = \left(\frac{1}{N}\sum_{i=1}^N\matLst_{(i)}\matLst_{(i)}^T\left(\hmatH_{g}-\hmatH_{(i),l}\matLambda_{2,(i)}^{-1}\matLambda_{3,(i)}^T\right)\right) \matLambda_6^{-1}+\mO(\sqrt{\alpha})\\
&\hmatH_{(i),l} = \matLst_{(i)}\matLst_{(i)}^T \hmatH_{(i),l}\matLambda_{2,(i)}^{-1}\\
&- \left(\frac{1}{N}\sum_{j=1}^N\matLst_{(j)}\matLst_{(j)}^T\left(\hmatH_{g}-\hmatH_{(j),l}\matLambda_{2,(j)}^{-1}\matLambda_{3,(j)}^T\right)\right) \matLambda_6^{-1}\matLambda_{3,(i)}\matLambda_{2,(i)}^{-1}+\mO(\sqrt{\alpha}),\\
\end{aligned}\right. 
\end{equation}
where $\mO(\sqrt{\alpha})$'s are terms whose Frobenius norm and $\ell_{\infty}$ norm is upper bounded by $\mO(\sqrt{\alpha})$, and $\matLambda_{6}$ is defined as $\matLambda_{6} = \matLambda_{1}-\frac{1}{N}\sum_{j=1}^N\matLambda_{3,(j)}\matLambda_{2,(j)}^{-1}\matLambda_{3,(j)}^T$.
\end{lemma}
}
\revise{The formal version of lemma \ref{lm:informalhpertuabation} and its proof are relegated to the appendix. We now briefly introduce our methodology for deriving the solutions in Lemma \ref{lm:informalhpertuabation}. For matrices $\matA,\matB,\matX,\matY$ satisfying the Sylvester equation $\matA\matX-\matX\matB=-\matY$, if the spectra of $\matA$ and $\matB$ are separated, i.e., $\sigma_{\max}(\matA)<\sigma_{\min}(\matB)$, then the solution can be written as $\matX=\sum_{p=0}^{\infty} \matA^p\matY\matB^{-1-p}$. We apply this solution form to \eqref{eqn:informalsimplifiedkkt} and iteratively expand $\hmatH_g$ and $\hmatH_{(i),l}$. The exact forms of the resulting series are shown in \eqref{eqn:ugcondition} and \eqref{eqn:ulcondition} in the appendix. In the series, each term is a product of a group of sparse matrices, an incoherent matrix, and some remaining terms. Based on the special structure of the series, we can calculate upper bounds on the Frobenius norm and maximum row norm of each term in the series. The leading terms are simply $\frac{1}{N}\sum_{i=1}^N\matLst_{(i)}\matLst_{(i)}^T\left(\hmatH_{g}-\hmatH_{(i),l}\matLambda_{2,(i)}^{-1}\matLambda_{3,(i)}^T\right) \matLambda_6^{-1}$ and $\matLst_{(i)}\matLst_{(i)}^T \hmatH_{(i),l}\matLambda_{2,(i)}^{-1}- \frac{1}{N}\sum_{j=1}^N\matLst_{(j)}\matLst_{(j)}^T\left(\hmatH_{g}-\hmatH_{(j),l}\matLambda_{2,(j)}^{-1}\matLambda_{3,(j)}^T\right) \matLambda_6^{-1}\matLambda_{3,(i)}\matLambda_{2,(i)}^{-1}$. By summing up the norm of all remaining higher-order terms in the series and applying a few basic inequalities in geometric series, we can prove the result in Lemma~\ref{lm:informalhpertuabation}.}

\paragraph{Perturbations on the optimal solutions to \eqref{eqn:subproblemobj}} Lemma \ref{lm:informalhpertuabation} then allows us to establish the inequality in Lemma \ref{lm:informalldiffinfnorm}. In epoch $t$, we know $\hmatL_{(i),t}=\hmatU_{g,t}\hmatV_{(i),g,t}^T+\hmatU_{(i),l,t}\hmatV_{(i),l,t}^T$, where $\hmatU_{g,t}, \hmatV_{(i),g,t}, \hmatU_{(i),l,t}, \hmatV_{(i),l,t}$ are the optimal solutions to the subproblem \eqref{eqn:subproblemobj}. By Lemma \ref{lm:kktcondition}, one can replace $\hmatU_{g,t}, \hmatV_{(i),g,t}, \hmatU_{(i),l,t}, \hmatV_{(i),l,t}$ by $\hmatH_g$ and $\hmatH_{(i),l}$ 
, and rewrite $\hmatL_{(i),t}$ as, 
$$
\hmatL_{(i),t} = \hmatH_{g,t}\hmatH_{g,t}^T\hmatM_{(i)}+\hmatH_{(i),l,t}\hmatH_{(i),l,t}^T\hmatM_{(i)}.
$$

Then, we can replace $\hmatH_{g,t}$ and $\hmatH_{(i),l,t}$ by the Taylor-like series described in Lemma \ref{lm:informalhpertuabation}. The error between $\matLst_{(i)}$ and $\hmatL_{(i),t}$ can be written as the summation of a few terms. The leading term is $\matHst_g\matHst_g^T\matLst_{(i)}+\matHst_{(i),l}\matHst_{(i),l}^T\matLst_{(i)}$, which is identical to $\matLst_{(i)}$ because of the SVD \eqref{eqn:truemodelsvd}. The remaining terms are errors resulting from $\mO(\sqrt{\alpha})$ terms in \eqref{eqn:hhatsolutioninformal} and $\matE_{(i)}$. Each of the error terms possesses a special structure that allows us to derive an upper bound on its $\ell_{\infty}$ norm. By summing up all these bounds, we can show that $\norm{\hmatL_{(i),t}-\matLst_{(i)}}_{\infty}\le \mO(\sqrt{\alpha}\max_j\norm{\matE_{(j),t}}_{\infty})$. The detailed calculations on the upper bounds on the $\ell_{\infty}$ norm of error terms are long and repetitive, thus relegated to the proof of Lemma \ref{lm:ldiffinfnorm} in the Appendix.

\section{Numerical Experiments}
\label{sec:numericalexperiment}
In this section, we investigate the numerical performance of \name on several datasets. We first use synthetic datasets to verify the convergence in Theorem \ref{thm:recovery} and validate \name's capability in recovering the common and individual features from noisy observations. Then, we use two examples of noisy video segmentation and anomaly detection to illustrate the utility of common, unique, and noise components. We implement Algorithm \ref{alg:altmin} with HMF \citep{hmf} as its subroutine \inneralg. Experiments in this section are performed on a desktop with 11th Gen Intel(R) i7-11700KF and NVIDIA GeForce RTX 3080. Code is available in the linked \href{https://github.com/UMDataScienceLab/TCMF}{Github repository}.

\subsection{Exact Recovery on Synthetic Data}
\label{sec:syntheticdata}
On the synthetic dataset, we simulate the data generation process in \eqref{eqn:matrixmodel}. We use $N=100$ sources and set the data dimension of $\matM_{(i)}$ to $15\times 1000$ in each source. We randomly generate $r_1=3$ global features and $r_2=3$ local features for each source. The local features are first generated randomly, then deflated to be orthogonal to the global ones. The sparse noise matrix $\matSst_{(i)}$ is randomly generated from the Bernoulli model, i.e., each entry of $\matSst_{(i)}$ is nonzero with probability $p$ and zero with probability $1-p$. We use $p$ as a proxy of the sparsity parameter $\alpha\overset{\Delta}{=}p$. The value of each entry in $\matSst_{(i)}$ is randomly sampled from $\{-100,100\}$ with equal probability. Next, we use \eqref{eqn:matrixmodel} to construct the observation matrix.  

With the generated $\{\matM_{(i)}\}$, we run Algorithm \ref{alg:altmin} with $\rho=0.99$ to estimate local, global, and sparse components. The subroutine \inneralg in Algorithm \ref{alg:altmin} is implemented by \hmf with spectral initialization. As discussed in Appendix \ref{sec:hmfintro}, \hmf is an iterative algorithm. In practice, for each call of \hmf, we run $500$ iterations with constant stepsize $0.005$, which take around $62$ seconds in our machine and generate satisfactory outputs. To quantitatively evaluate the convergence error, we calculate the $\ell_{\infty}$ error of local, global, and sparse components as specified in Theorem \ref{thm:recovery}. More specifically, we calculate the $\ell_{\infty}$ global error at epoch $t$ as $$
\ell_{\infty}-\text{global error} = \frac{1}{N}\sum_{i=1}^N\norm{\hmatUeps_{g,t}\hmatVeps_{(i),g,t}^T-\matUst_{(i),g}\matVst_{(i),g}^T}_{\infty},
$$
the $\ell_{\infty}$ local error at epoch $t$ as 
$$
\ell_{\infty}-\text{local error}=\frac{1}{N}\sum_{i=1}^N\norm{\hmatUeps_{(i),l,t}\hmatVeps_{(i),l,t}^T-\matUst_{(i),l}\matVst_{(i),l}^T}_{\infty},
$$
and the $\ell_{\infty}$ sparse noise error at epoch $t$ as 
$$
\ell_{\infty}-\text{sparse error}=\frac{1}{N}\sum_{i=1}^N\norm{\hmatS_{(i),t}-\matSst_{(i)}}_{\infty}.
$$   We show the error plot for three different sparsity parameters $\alpha$ in Figure  \ref{fig:errplot}. 

\begin{figure}[h!]
\centering
\includegraphics[width=8cm]{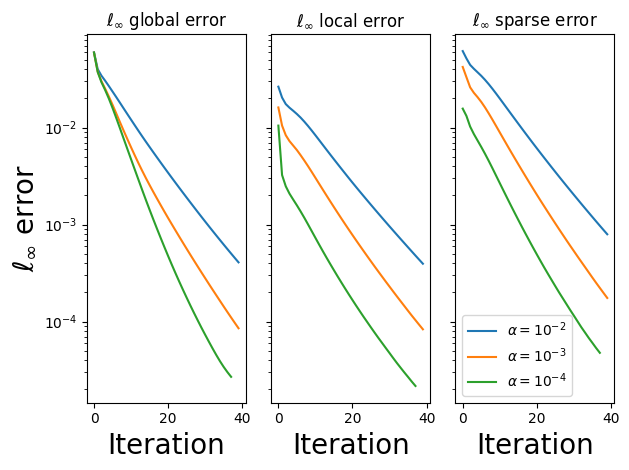}
\caption{Error plots of Algorithm \ref{alg:altmin}. The x-axis denotes the iteration index, and the y-axis shows the  $\ell_{\infty}$ error at the corresponding iteration. The y-axis is in log scale.}
\label{fig:errplot}
\end{figure}
From Figure \ref{fig:errplot}, it is clear that the global, local, and sparse components indeed converge linearly to the ground truth.

Further, we compare the feature extraction performance of \name with benchmark algorithms, including JIVE \citep{jive}, RJIVE \citep{rjive}, RaJIVE \citep{rajive}, and HMF \citep{hmf}. We do not include the comparison with RCICA \citep{rcica} because RCICA is designed only for $N=2$, while we have $100$ different sources. Since the errors of different methods vary drastically, we calculate and report the logarithm of global error as $\texttt{g-error}=\log_{10}\left(\frac{1}{N}\sum_{i=1}^N\norm{\hmatUeps_{g,t}\hmatVeps_{(i),g,t}^T-\matUst_{(i),g}\matVst_{(i),g}^T}_{F}^2\right)$, the logarithm of local error as $\texttt{l-error}=\log_{10}\left(\frac{1}{N}\sum_{i=1}^N\norm{\hmatUeps_{(i),l,t}\hmatVeps_{(i),l,t}^T-\matUst_{(i),l}\matVst_{(i),l}^T}_{F}^2\right)$, and the logarithm of sparse noise error as $\texttt{s-error}=\log_{10}\left(\sum_{i=1}^N\norm{\hmatS_{(i),t}-\matSst_{(i)}}_{F}^2\right)$ at $t=20$. We run experiments from $5$ different random seeds and calculate the mean and standard deviation of the log errors. Results are reported in Table \ref{tab:logrecoverysynthetic}.

\begin{table}[htbp]
  \centering
  \caption{Recovery error of different algorithms. The columns \texttt{g-error}, \texttt{l-error}, and \texttt{s-error} stand for the log recovery errors of global components, local components, and sparse components. }
  \scalebox{0.88}{
    \begin{tabular}{ccccccc}
    \toprule
          & \multicolumn{3}{c}{$\alpha=0.01$} & \multicolumn{3}{c}{$\alpha=0.1$}   \\
          & \texttt{g-error} & \texttt{l-error} & \texttt{s-error} & \texttt{g-error} & \texttt{l-error} & \texttt{s-error}  \\
          \hline
    JIVE & $5.52 \pm  0.01$ &  $5.64 \pm 0.01$ & - & $6.52 \pm  0.01$ &  $6.58 \pm 0.01$ & - \\
HMF & $5.49 \pm  0.01$ &  $5.62 \pm 0.01$ & - & $6.48 \pm  0.01$ &  $6.55 \pm 0.01$ & - \\
RaJIVE & $5.46 \pm  0.01$ &  $5.36 \pm 0.05$ & $ 5.71 \pm 0.05$ & $6.48 \pm  0.00$ &  $6.25 \pm 0.14$ & $ 6.59 \pm 0.12$ \\
RJIVE & $5.49 \pm  0.01$ &  $5.44 \pm 0.01$ & $ 5.77 \pm 0.01$ & $6.48 \pm  0.00$ &  $6.47 \pm 0.00$ & $ 6.78 \pm 0.00$ \\
\name & $\textbf{-3.38} \pm  0.14$ &  $\textbf{-3.37} \pm 0.13$ & $ \textbf{-2.94} \pm 0.08$ & $\textbf{-1.93} \pm  0.09$ &  $\textbf{-1.95} \pm 0.06$ & $ \textbf{-1.54} \pm 0.04$ \\
    \bottomrule
    \end{tabular}%
}\label{tab:logrecoverysynthetic}%
\end{table}%

Table \ref{tab:logrecoverysynthetic} shows that \name outperforms benchmark algorithms by several orders. This is understandable as \name is provably convergent into the ground truth, while benchmark algorithms either neglect sparse noise or rely on instance-dependent heuristics. 

\subsection{Video Segmentation from Noisy Frames}
An important task in video segmentation is background-foreground separation. There are several matrix factorization algorithms that can achieve decent performance in video segmentation, including robust PCA \citep{robustpca}, PerPCA \citep{personalizedpca}, and HMF \citep{hmf}. However, the separation is much more challenging when the videos are corrupted by large noise \citep{smoothsparsedecomposition}. \name can naturally handle such tasks with its power to recover global and local components from highly noisy measurements. 

We use a surveillance video from \cite{vacavant} as an example. In the video, multiple vehicles drive through the circle. We add large and sparse noise to the frames to simulate the effects of large measurement errors. More specifically, similar to \ref{sec:syntheticdata}, we sample each entry of noise from i.i.d. Bernoulli distribution that is zero with probability $0.99$ and nonzero with probability $0.01$. And each entry is sampled from $\{-500,500\}$ with equal probability. Then we apply \name on the noisy frames to recover $\hmatU_g,\{\hmatV_{(i),g},\hmatU_{(i),l},\hmatV_{(i),l},\hmatS_{(i)}\}$. We set $\rho$ to $0.95$ and use $T=15$ epochs. The subroutine \inneralg is still implemented by \hmf with spectral initialization. We also set the number of iterations for \hmf to $500$. This is a conservative choice to ensure small optimization error $\epsilon$ in the subroutine. To visualize the results, we plot global components $\hmatU_g\hmatV_{(i),g}^T$ and local components $\hmatU_{(i),l}\hmatV_{(i),l}^T$. They are shown in Table \ref{tab:carvideo}.

\begin{table}[h!]
\caption{Foreground Background separation }
\label{tab:carvideo}
\begin{center}
\begin{small}
\begin{tabular}{ccccc}
\toprule
Frame &  1 & 2 &3  \\ 
\hline
\begin{tabular}{@{}c@{}}Original \\noisy\\ frames\end{tabular} & \begin{minipage}[c]{0.2\textwidth} \includegraphics[width=.99\linewidth]{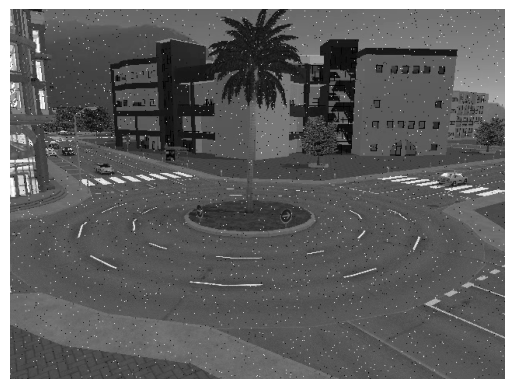} \end{minipage} & \begin{minipage}[c]{0.2\textwidth} \includegraphics[width=.99\linewidth]{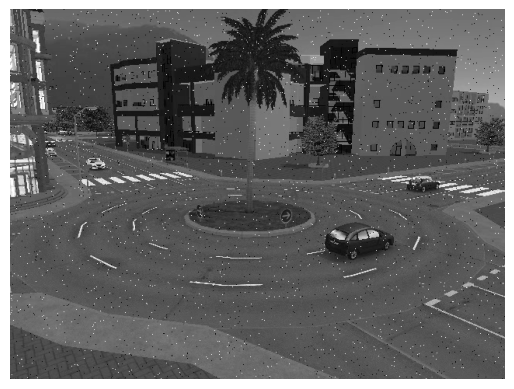} \end{minipage} &\begin{minipage}[c]{0.2\textwidth} \includegraphics[width=.99\linewidth]{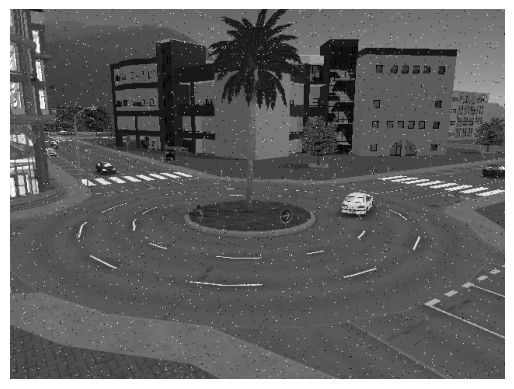} \end{minipage} \\

\begin{tabular}{@{}c@{}}Noise\end{tabular} & \begin{minipage}[c]{0.2\textwidth} \includegraphics[width=.99\linewidth]{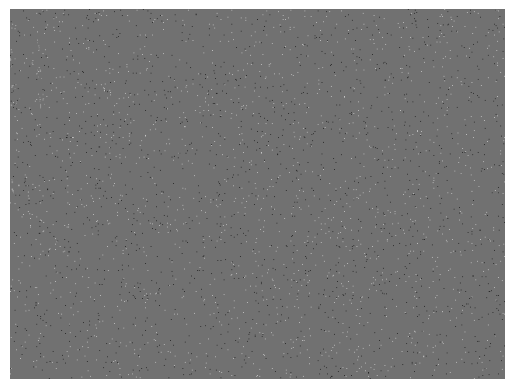} \end{minipage} & \begin{minipage}[c]{0.2\textwidth} \includegraphics[width=.99\linewidth]{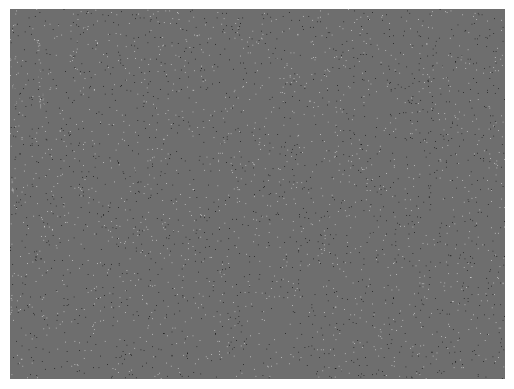} \end{minipage} &\begin{minipage}[c]{0.2\textwidth} \includegraphics[width=.99\linewidth]{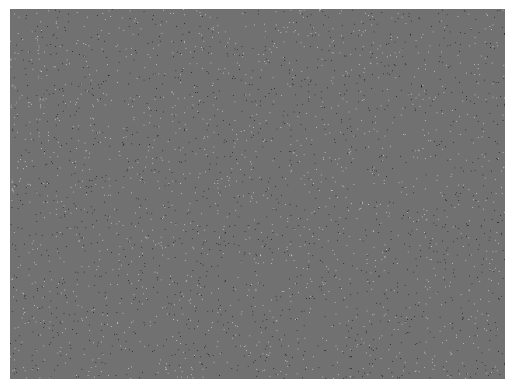} \end{minipage} \\

\begin{tabular}{@{}c@{}}Global\\ components\end{tabular} &\begin{minipage}[c]{0.2\textwidth} \includegraphics[width=.99\linewidth]{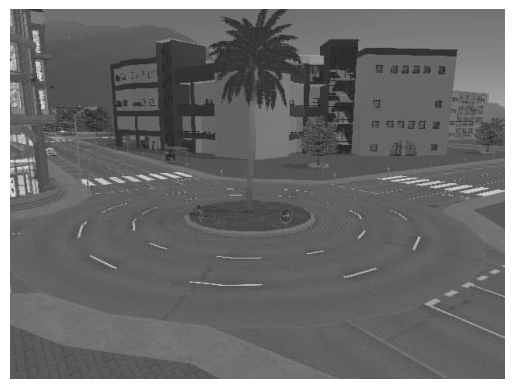} \end{minipage} & \begin{minipage}[c]{0.2\textwidth} \includegraphics[width=.99\linewidth]{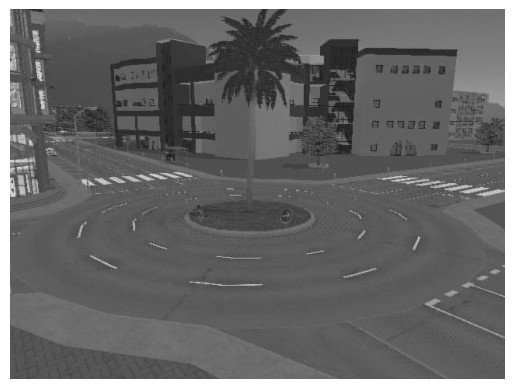} \end{minipage} &\begin{minipage}[c]{0.2\textwidth} \includegraphics[width=.99\linewidth]{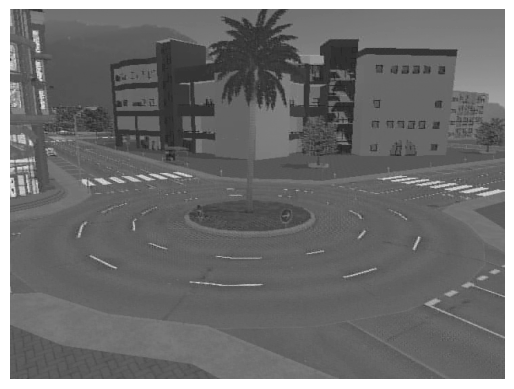} \end{minipage} \\

\begin{tabular}{@{}c@{}}Local\\ components\\
\end{tabular} & \begin{minipage}[c]{0.2\textwidth} \includegraphics[width=.99\linewidth]{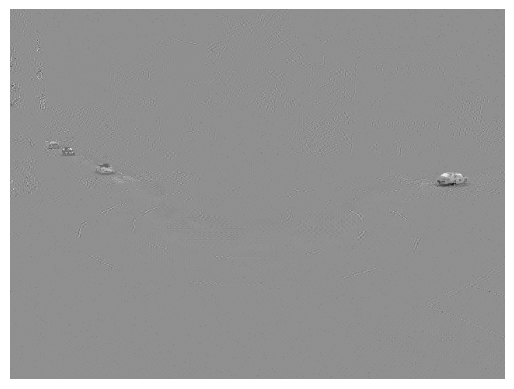} \end{minipage} & \begin{minipage}[c]{0.2\textwidth} \includegraphics[width=.99\linewidth]{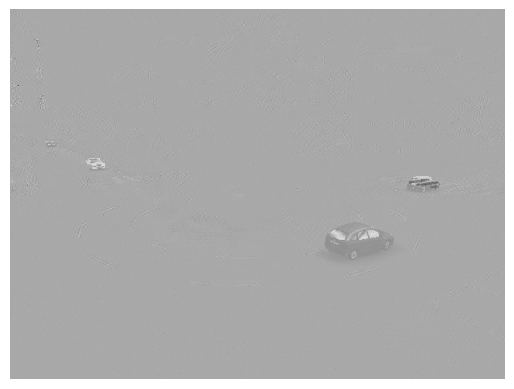} \end{minipage} &\begin{minipage}[c]{0.2\textwidth} \includegraphics[width=.99\linewidth]{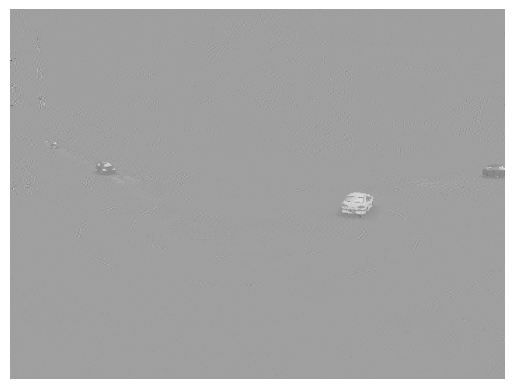} \end{minipage} \\
\bottomrule
\end{tabular}
\end{small}
\end{center}
\end{table}
In Table \ref{tab:carvideo}, the background and foreground are clearly separated from the noise. The result highlights \name's ability to extract features in high-dimensional noisy data.

We compare \name to several benchmark methods, namely JIVE, HMF, RJIVE, RaJIVE, and Robust PCA. These algorithms, including JIVE, HMF, RJIVE, and RaJIVE, are capable of producing joint and individual components of video frames. In our evaluation, we consider the joint component as the background and the individual component as the foreground. As for robust PCA, we flatten each image into a row vector and create a large matrix $\matM_{\text{stack}}$ by stacking these row vectors. We then utilize the nonconvex robust PCA \citep{nonconvexrobustpca} to extract the sparse and low-rank components from $\matM_{\text{stack}}$. The low-rank component is regarded as the background, while the sparse component captures the foreground.

To assess the performance of these methods, we calculate the differences between the recovered background and foreground compared to the ground truths. Specifically, we estimate the mean squared error (MSE), peak signal-to-noise ratio (PSNR), and structural similarity index (SSIM) of the recovered foreground and background with respect to the true foreground and background. The comparison results are presented in Table \ref{tab:videometrics}.
\begin{table}[htbp]
  \centering
  \caption{Background and foreground recovery quality metrics for different algorithms.}
\begin{tabular}{cccccccc}
\toprule
          & \multicolumn{3}{c}{Background} & \multicolumn{3}{c}{Foreground} & Wall-clock \\
          & MSE $\downarrow$   & PSNR $\uparrow$ & SSIM $\uparrow$ & MSE $\downarrow$  & PSNR $\uparrow$ & SSIM $\uparrow$ & time (s) $\downarrow$\\
          \hline
    JIVE  & 415   & -26   & 0.08  & 2521  & 14    & 0.03 & $1.6\times 10^3$ \\
    HMF   & 198   & -22   & 0.18  & 2413  & 14    & 0.05 & $2.3\times 10^1$\\
    
    PerPCA   & 236   & -23   & 0.14  & 2389  & 14    & 0.07 & $9.8\times 10^1$\\
    RJIVE & 277   & -24   & 0.13  & 1309  & 16    & 0.22 & $9.2\times 10^1$\\
    RaJIVE &  170     &   -22    &    0.22   &   166    &   26    &  0.18 & $1.2\times 10^4$\\
    Robust PCA &  0.0016     &   31    &    \textbf{1.00}   &   5105    &   11    &  0.61 & $\textbf{3.3}\times 10^{-1}$\\
    \name  & \textbf{0.0003} & \textbf{33}    & \textbf{1.00}     & \textbf{98 }   & \textbf{31}    & \textbf{0.98} & $3.5\times 10^2$\\
    \bottomrule
    \end{tabular}%
  \label{tab:videometrics}%
\end{table}%

In Table \ref{tab:videometrics}, a lower MSE, a higher PSNR, and a higher SSIM signify superior recovery quality. In terms of background recovery, both \name and robust PCA exhibit low MSE, high PSNR, and high SSIM, surpassing other methods. This suggests that both algorithms effectively reconstruct the background. This outcome was anticipated as \name and robust PCA possess the capability to differentiate between significant noise and low-rank components. In contrast, other benchmarks either neglect large noise in the model or rely on heuristics. Furthermore, \name showcases marginally superior performance in MSE and PSNR compared to robust PCA, signifying higher-quality background recovery.

When it comes to foreground recovery, \name outperforms benchmark algorithms significantly across all metrics. The inability of robust PCA to achieve high-quality foreground recovery is likely due to its inability to separate sparse noise from the foreground. JIVE and HMF yield high MSE and low PSNR, indicating noisy foreground reconstruction. Although heuristic methods, such as RJIVE and RaJIVE, exhibit slight performance improvements over JIVE and HMF, they still fall short of the performance exhibited by \name. This comparison underscores \name's remarkable power to identify unique components from sparse noise accurately.

We also report the running time of each experiment in Table \ref{tab:videometrics}. Compared with heuristic methods to robustly separate the shared and unique components, \name exhibits a slightly longer running time than RJIVE but significantly outperforms RaJIVE in terms of speed. The comparison highlights \name's superior performance with moderate computation demands. \revise{Although Robust PCA demonstrates a relatively short running time in this instance, larger-scale experiments presented in Appendix~\ref{ap:runtimecomparison}
 will show that Robust PCA has larger running time scaling as the problem size increases. }

\subsection{Case Study: Defect Detection on Steel Surface}

Hot rolling is an important process in steel manufacturing. For better product quality, a critical task is to detect and locate the defects that arise in the rolling process \citep{datasetoriginal}. In this study, the dataset \citep{datasetoriginal,haodataset} comes from the HotEye video of a rolling steel plate. The video captures sharp pictures of the surface of the steel plate. An example is shown in the left graph of Figure \ref{fig:tworesultsrolling}. The irregular dark dots in the graph indicate surface defects that require subsequent investigations \citep{datasetoriginal}.

As different frames of the rolling video are related, they possess similar background patterns. Meanwhile, each frame also contains unique variations that reflect frame-by-frame differences. On top of the changing patterns, there are small defects on the surface of the steel plate. The defects, as shown in the left graph of Figure~\ref{fig:tworesultsrolling}, only occupy small spatial regions and thus can naturally be modeled by sparse outliers.  

In such scenarios, the application of \name enables the identification of defects and extraction of common and unique patterns simultaneously. For this experiment, we use \name to segment 100 hot-rolling video frames. The right graph of Figure~ \ref{fig:tworesultsrolling} illustrates two frames selected from the rolling video alongside the corresponding recovered global, local, and sparse components. We set the reduction parameter $\rho=0.97$ and the number of epochs $T=100$. The details of the subroutine \inneralg are relegated to Appendix \ref{sec:hmfintro}. Additionally, as a comparative analysis, we employ nonconvex robust PCA \citep{nonconvexrobustpca} to recover and display the low-rank and sparse components from frames. Our robust PCA implementation alternatively applies SVD and hard thresholding. We require all entries in the sparse component to be negative in the hard-thresholding step to encode the domain knowledge that surface defects tend to have lower temperatures. The hyper-parameters for SVD and thresholding are consistent for both \name and Robust PCA. 

\begin{figure}[h!]
   \begin{minipage}{0.25\textwidth}
     \centering
     \includegraphics[width=.9\linewidth]{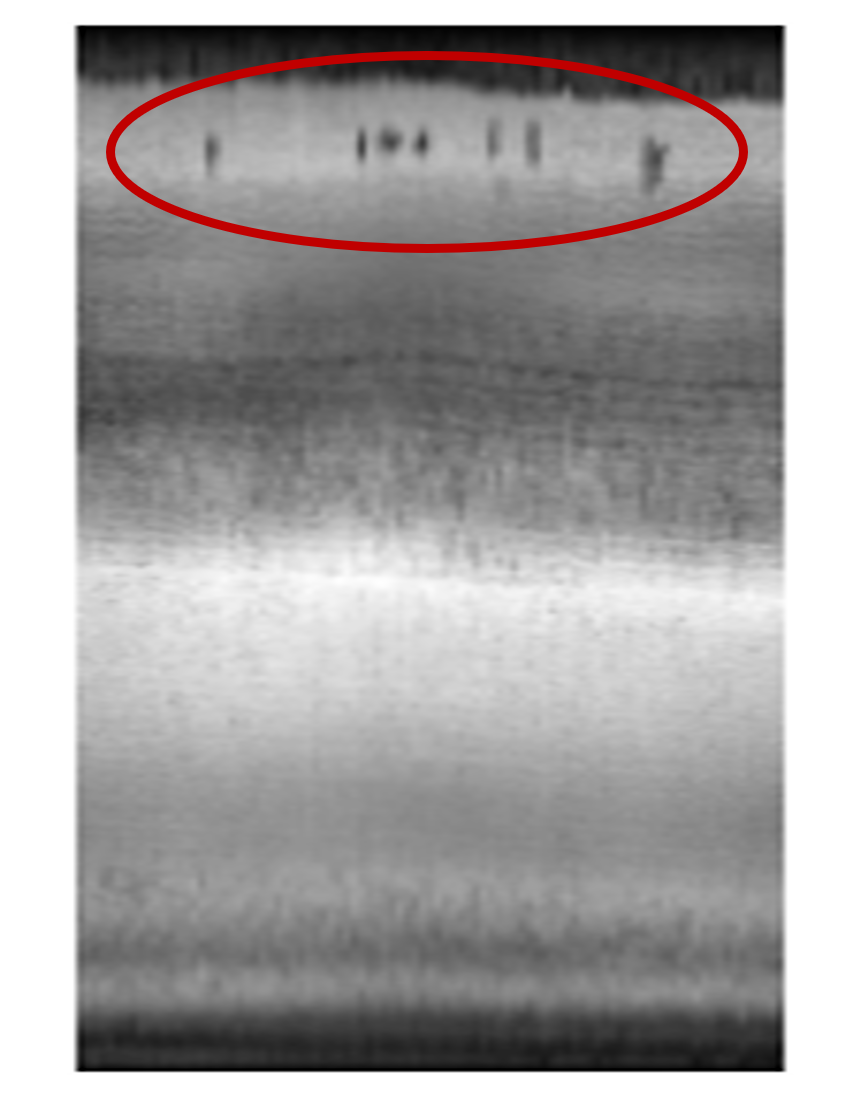}
     
   \end{minipage}\hfill
   \begin{minipage}{0.75\textwidth}
     \centering
     \includegraphics[width=.999\linewidth]{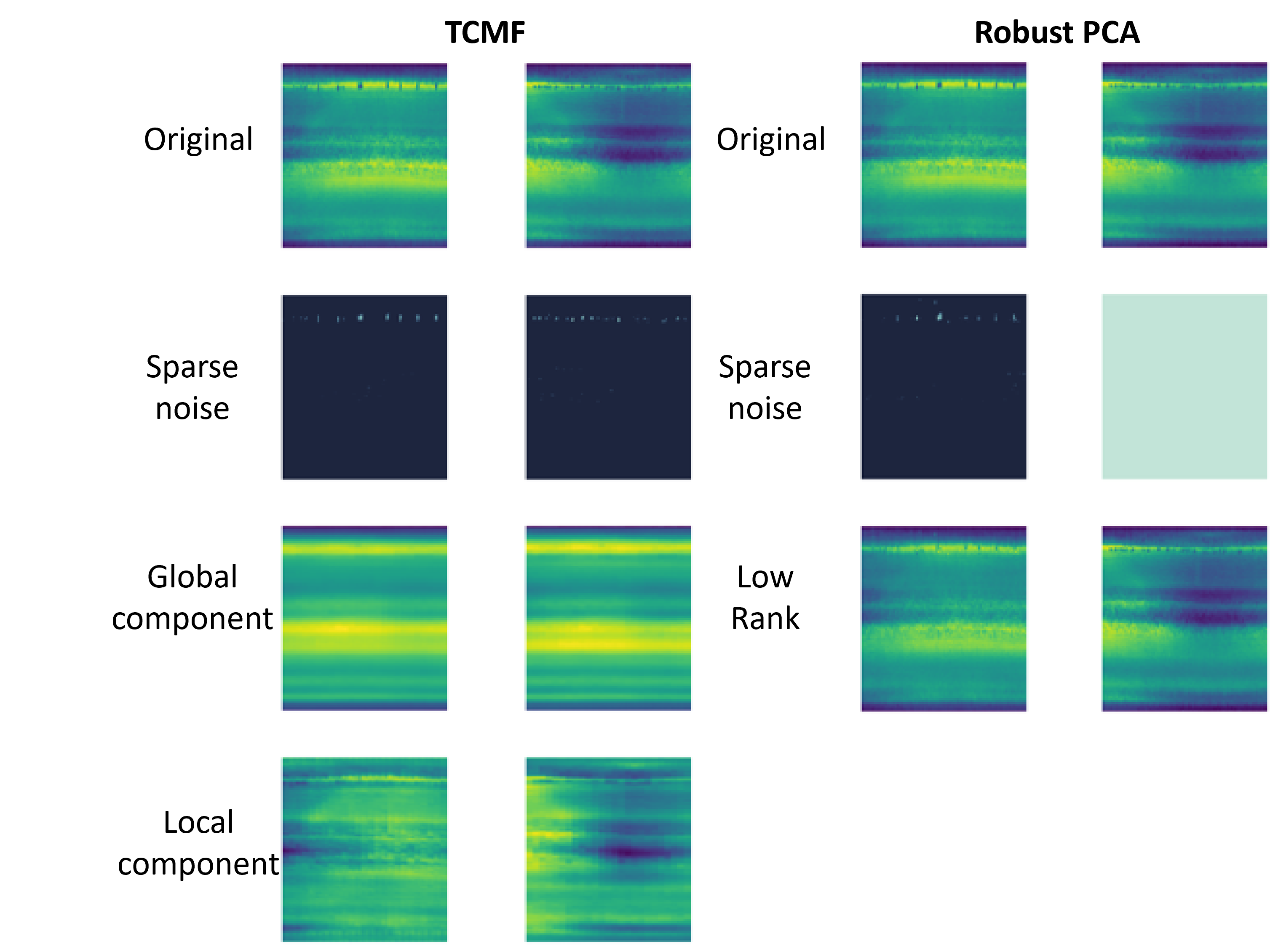}
     
   \end{minipage}
   \caption{\textit{Left}: An example of the surface of the steel bar. There are a few anomalies inside the red ellipse. \textit{Right}: Recovered sparse noises, shared components, and unique components from $2$ frames.}\label{fig:tworesultsrolling}
\end{figure}
In Figure \ref{fig:tworesultsrolling}, we can see that \name effectively identifies the small defects on the steel plate surface. The global component reflects the general patterns in the video frames, while the local component accentuates the variations in different frames. In contrast, the sparse components recovered by robust PCA do not faithfully represent the surface defects. 

We proceed to show that \name-recovered sparse components can be conveniently leveraged for frame-level anomaly detection. Our task here is to identify which frames contain surface anomalies. Inspired by the statistics-based anomaly detection \citep{anomalydetection}, we construct simple test statistics to monitor the anomalies. The test statistics is defined as the $\ell_1$ norm of the recovered sparse noise on each frame $\norm{\hmatS_{(i)}}_1$. Indeed, a large $\norm{\hmatS_{(i)}}_1$ provides strong evidence for surface defects. The choice of $\ell_1$ norm is not special as we find other norms, such as $\ell_2$ norm, would yield a similar performance. 

After using \name to extract the sparse components, we calculate the test statistics for each frame. Then, we can set up a simple threshold-based classification rule for anomaly detection: when the $\ell_1$ norm exceeds the threshold, we report an anomaly in the corresponding frame. In the case study, the threshold is set to be the highest value in the first $50$ frames, which is the in-control group that does not contain anomalies \citep{haodataset}. 
We plot the test statistics and thresholds in Figure~\ref{Fig:robustpersonalizedpcaplate}. The blue dots and red crosses denote the (ground truth) normal and abnormal frame labels in 
\citet{haodataset}. In an ideal plot of test statistics, one would expect the abnormal samples to have higher $\ell_1$ norms, while normal samples should have lower norms. This is indeed the case for Figure~\ref{Fig:robustpersonalizedpcaplate}, where a simple threshold based on the sparse features can distinguish abnormal samples from normal ones with high accuracy.

In comparison, we also calculate the $\ell_1$ norm of sparse noise recovered by robust PCA and plot the testing statistics in Figure \ref{Fig:robustpcaplate}. In Figure \ref{Fig:robustpcaplate}, the $\ell_1$ norm is less indicative of anomaly labels, as some abnormal samples have small test statistics, while some normal samples have large statistics. It is also hard to use a threshold on the test statistics to predict anomalies. 

\begin{figure}[H]
   \begin{minipage}{0.45\textwidth}
     \centering
     \includegraphics[width=.99\linewidth]{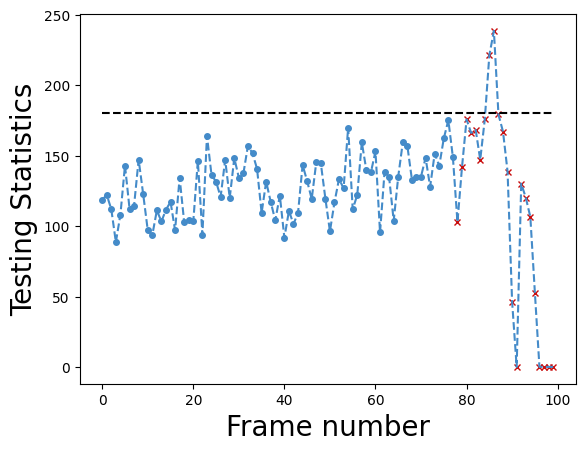}
     \caption{Test statistics of robust PCA.}\label{Fig:robustpcaplate}
   \end{minipage}\hfill
   \begin{minipage}{0.45\textwidth}
     \centering
     \includegraphics[width=.99\linewidth]{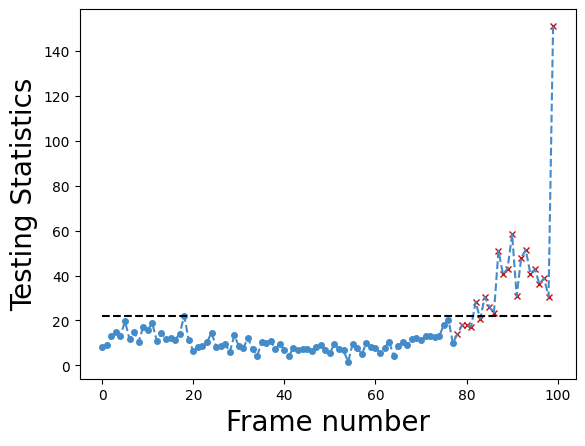}
     \caption{Test statistics of \name.}\label{Fig:robustpersonalizedpcaplate}
   \end{minipage}
\end{figure}

The comparison highlights \name's ability to find surface defects. The results are understandable as \name uses a more refined model to decompose the thermal frames into three parts, thus having more representation power to fit the underlying physics in the manufacturing process. As a result, the recovered sparse components are more representative of the anomalies.

\section{Conclusion} 
In this work, we propose a systematic method \name to separate shared, unique, and noise components from noisy observation matrices. \name is the first algorithm that is provably convergent to the ground truth under identifiability conditions that require the three components to have ``small overlaps''. \name outperforms previous heuristic algorithms by large margins in numerical experiments and finds interesting applications in video segmentation, anomaly detection, and time series imputation.

Our work also opens up several venues for future theoretical exploration in separating shared and unique low-rank features from noisy matrices. For example, a minimax lower bound on the $\mu, \theta$, and $\alpha$ can help fathom the statistical difficulty of such separation. Also, as many existing methods for \inneralg rely on good initialization to excel, designing efficient algorithms for \inneralg that are independent of the initialization is also an interesting topic. On the practical side, methods to integrate \name with other machine learning models, e.g., auto-encoders, to find nonlinear features in data are worth exploring.


\acks{
This research is supported in part by Raed Al Kontar's NSF CAREER Award 2144147 and Salar Fattahi's NSF CAREER Award CCF-2337776, NSF Award DMS-2152776, and ONR Award N00014-22-1-2127,  
}


\newpage

\appendix
\section{Details of subroutine algorithm}

In this section, we will elaborate on two subroutine algorithms in literature to solve the problem \eqref{eqn:subproblemobj}, specifically known as \hmf and \perpca. Among a plethora of existing methodologies aiming to distinguish shared and unique features, these two exhibit an exceptional significance, as they are proved to converge to the optimal resolution of problem \eqref{eqn:subproblemobj} linearly under appropriate initial conditions. We underscore that the usage of \inneralg is not restricted solely to these two methods. In essence, any algorithm with the ability to segregate common and unique components can be effectively employed as \inneralg.

\subsection{Heterogeneous matrix factorization}
\label{sec:hmfintro}
Heterogeneous matrix factorization (\hmf) \citep{hmf} is an algorithm proposed to solve the following problem,
\begin{equation}
\label{eqn:hmfproblemformulation}
\begin{aligned}
&\min_{\matU_g,\{\matU_{(i),l}\}_{i=1,\cdots,N}} \sum_{i=1}^N\tf_i\left(\matU_g,\{\matV_{(i),g},\matU_{(i),l},\matV_{(i),l}\}\right)\\
=\sum_{i=1}^N\frac{1}{2}&\norm{\hmatM_{(i)}-\matU_g\matV_{(i),g}^T-\matU_{(i),l}\matV_{(i),l}^T}_F^2+\frac{\beta}{2}\norm{\matU_g^T\matU_g-\matI}_F^2+\frac{\beta}{2}\norm{\matU_{(i),l}^T\matU_{(i),l}-\matI}_F^2\\
\text{subject to  }&  
\,  \matU_g^T\matU_{(i),l}=\bm{0},\, \forall i . \\
\end{aligned}
\end{equation}
Compared with \eqref{eqn:subproblemobj}, the objective in \eqref{eqn:hmfproblemformulation} contains two additional regularization terms $\frac{\beta}{2}\norm{\matU_g^T\matU_g-\matI}_F^2+\frac{\beta}{2}\norm{\matU_{(i),l}^T\matU_{(i),l}-\matI}_F^2$. The regularization terms enhance the smoothness of the optimization objective thereby facilitating convergence. Despite the regularization terms, any optimal solution to~\eqref{eqn:hmfproblemformulation} is also an optimal solution to~\eqref{eqn:subproblemobj}. We can prove the claim in the following proposition.
\begin{proposition}
Let $\hmatU_g^{\hmf},\{\hmatV^{\hmf}_{(i),g},\hmatU^{\hmf}_{(i),l},\hmatV^{\hmf}_{(i),l}\}$ be \textit{one} set of optimal solutions to~\eqref{eqn:hmfproblemformulation}, then $\hmatU_g^{\hmf},\{\hmatV^{\hmf}_{(i),g},\hmatU^{\hmf}_{(i),l},\hmatV^{\hmf}_{(i),l}\}$ is also a set of optimal solution to~\eqref{eqn:subproblemobj}
\end{proposition}
\begin{proof}
The proof is straightforward. We first claim that $\hmatU_g^{\hmf}{}^T\hmatU_g^{\hmf}=\matI$ and $\hmatU^{\hmf}_{(i),l}{}^T\hmatU^{\hmf}_{(i),l}=\matI$. We prove the claim by contradiction. Suppose otherwise, we can find a QR decomposition of $\hmatU_g^{\hmf}$ and $\hmatU^{\hmf}_{(i),l}$ as $\hmatU_g^{\hmf} = \matQ_g\matR_g$ and $\hmatU_g^{\hmf} = \matQ_{(i),l}\matR_{(i),l}$, where $\matQ_g$ and $\matQ_{(i),l}$'s are orthonormal and $\matR_g$ and $\matR_{(i),l}$'s are upper-triangular. Furthermore, not both $\matR_g$ and $\matR_{(i),l}$ are identity matrices, thus $\norm{\matR_g^T\matR_g-\matI}_F^2+ \norm{\matR_{(i),l}^T\matR_{(i),l}-\matI}_F^2> 0$. Now, we construct a refined set of solutions as,
\begin{equation*}
\begin{aligned}
\hmatU_g^{\hmf, refined} &= \matQ_g\\
\hmatV^{\hmf, refined}_{(i),g} &= \hmatV_{(i),g}^{\hmf} \matR_g^{T}\\
\hmatU^{\hmf, refined}_{(i),l} &= \matQ_{(i),l}\\
\hmatV^{\hmf, refined}_{(i),l} &= \hmatV^{(i),l}_{\hmf} \matR_{(i),l}^{T}.\\
\end{aligned}
\end{equation*}
Then it's easy to verify that
$$
\begin{aligned}
&\sum_{i=1}^N\tf_i\left(\hmatU^{\hmf, refined}_g,\hmatV^{\hmf, refined}_{(i),g},\hmatU^{\hmf, refined}_{(i),l},\hmatV^{\hmf, refined}_{(i),l}\right)\\
&= \sum_{i=1}^N\tf_i\left(\hmatU^{\hmf}_g,\hmatV_{(i),g}^{\hmf},\hmatU^{\hmf}_{(i),l},\hmatV^{\hmf}_{(i),l}\right) - \frac{\beta}{2}\left(\norm{\matR_g^T\matR_g-\matI}_F^2+ \norm{\matR_{(i),l}^T\matR_{(i),l}-\matI}_F^2\right)\\
&< \sum_{i=1}^N\tf_i\left(\hmatU^{\hmf}_g,\hmatV_{(i),g}^{\hmf},\hmatU^{\hmf}_{(i),l},\hmatV^{\hmf}_{(i),l}\right),
\end{aligned}
$$
which contradicts with the global optimality of $\hmatU_g^{\hmf},\{\hmatV^{\hmf}_{(i),g},\hmatU^{\hmf}_{(i),l},\hmatV^{\hmf}_{(i),l}\}$. This proves the claim.

From the orthogonality, we know $f_i\left(\hmatU^{\hmf}_g,\hmatV_{(i),g}^{\hmf},\hmatU^{\hmf}_{(i),l},\hmatV^{\hmf}_{(i),l}\right) = \tf_i\left(\hmatU^{\hmf}_g,\hmatV_{(i),g}^{\hmf},\hmatU^{\hmf}_{(i),l},\hmatV^{\hmf}_{(i),l}\right)$.

Now suppose $\hmatU_g^{\hmf},\{\hmatV^{\hmf}_{(i),g},\hmatU^{\hmf}_{(i),l},\hmatV^{\hmf}_{(i),l}\}$ is not an optimal solution to ~\eqref{eqn:subproblemobj}. Then, we can find a different set of feasible solution $\hmatU_g^{\inneralg},\{\hmatV^{\inneralg}_{(i),g},\hmatU^{\inneralg}_{(i),l},\hmatV^{\inneralg}_{(i),l}\}$ such that
$$
\begin{aligned}
&\sum_{i=1}^N f_i\left(\hmatU^{\inneralg}_g,\hmatV^{\inneralg}_{(i),g},\hmatU^{\inneralg}_{(i),l},\hmatV^{\inneralg}_{(i),l}\right)\\
&< \sum_{i=1}^Nf_i\left(\hmatU^{\hmf}_g,\hmatV_{(i),g}^{\hmf},\hmatU^{\hmf}_{(i),l},\hmatV^{\hmf}_{(i),l}\right) \\
& = \sum_{i=1}^N\tf_i\left(\hmatU^{\hmf}_g,\hmatV_{(i),g}^{\hmf},\hmatU^{\hmf}_{(i),l},\hmatV^{\hmf}_{(i),l}\right).
\end{aligned}
$$
We can similarly define a set of refined solutions 
\begin{equation*}
\begin{aligned}
\hmatU_g^{\inneralg, refined} &= \matQ^{\inneralg}_g\\
\hmatV^{\inneralg, refined}_{(i),g} &= \hmatV_{(i),g}^{\inneralg} \matR_g^{\inneralg}{}^{T}\\
\hmatU^{\inneralg, refined}_{(i),l} &= \matQ^{\inneralg}_{(i),l}\\
\hmatV^{\inneralg, refined}_{(i),l} &= \hmatV^{(i),l}_{\hmf} \matR_{(i),l}^{\inneralg}{}^{T},\\
\end{aligned}
\end{equation*}
where $\matQ^{\inneralg}_g$, $\matR_g^{\inneralg}$, $\matQ^{\inneralg}_{(i),l}$, $\matR_{(i),l}^{\inneralg}$ are QR decompositions that satisfy $\hmatU_g^{\inneralg, refined}=\matQ^{\inneralg}_g\matR_g^{\inneralg}$ and $\hmatU^{\inneralg, refined}_{(i),l} =\matQ^{\inneralg}_{(i),l}\matR_{(i),l}^{\inneralg}$. Based on the refined set of solutions, we can prove that,
$$
\begin{aligned}
&\sum_{i=1}^N \tf_i\left(\hmatU^{\inneralg, refined}_g,\hmatV^{\inneralg, refined}_{(i),g},\hmatU^{\inneralg, refined}_{(i),l},\hmatV^{\inneralg, refined}_{(i),l}\right)\\
&=\sum_{i=1}^N f_i\left(\hmatU^{\inneralg}_g,\hmatV^{\inneralg}_{(i),g},\hmatU^{\inneralg}_{(i),l},\hmatV^{\inneralg}_{(i),l}\right)\\
& < \sum_{i=1}^N\tf_i\left(\hmatU^{\hmf}_g,\hmatV_{(i),g}^{\hmf},\hmatU^{\hmf}_{(i),l},\hmatV^{\hmf}_{(i),l}\right),
\end{aligned}
$$
which contradicts the optimality of $\hmatU_g^{\hmf},\{\hmatV^{\hmf}_{(i),g},\hmatU^{\hmf}_{(i),l},\hmatV^{\hmf}_{(i),l}\}$.

This completes the proof.
\end{proof}

\hmf optimizes the objective by gradient descent. To ensure feasibility, \hmf employs a special correction step to orthogonalize $\matU_g$ and $\matU_{(i),l}$ without changing the objective at every step. The pseudo-code is presented in Algorithm \ref{alg:hmf}.

\begin{algorithm}
\caption{\inneralg by heterogeneous matrix factorization}
\label{alg:hmf}
\begin{algorithmic}[1]
\STATE Input matrices $\{\hmatM_{(i)}\}_{i=1}^N$, stepsize $\eta_\tau$, iteration budget $R$.
\STATE Initialize $\matU_{g,1}, \matV_{(i),g,\fracud},\matU_{(i),l,\fracud},\matV_{(i),l,1}$ to be small random matrices.
\FOR{
Iteration $\tau=1,...,R$}
\FOR{ index $i=1,\cdots,N$}

\STATE Correct $\matU_{(i),l,\tau}=\matU_{(i),l,\tau-\frac{1}{2}}-\matU_{g,\tau}\left(\matU_{g,\tau}^T\matU_{g,\tau}\right)^{-1}\matU_{g,\tau}^T\matU_{(i),l,\tau-\frac{1}{2}}$
\STATE Correct $\matV_{(i),g,\tau}=\matV_{(i),g,\tau-\fracud}+\matV_{(i),l,\tau}\matU_{(i),l,\tau-\fracud}^T\matU_{g,\tau}\left(\matU_{g,\tau}^T\matU_{g,\tau}\right)^{-1}$

\STATE Update $ \matU_{(i),g,\tau+1}=\matU_{g,\tau}-\eta_{\tau}\nabla_{\matU_g}\tf_i$
\STATE Update $ \matV_{(i),g,\tau+\fracud}=\matV_{(i),g,\tau}-\eta_{\tau}\nabla_{\matV_{(i),g}}\tf_i$
\STATE Update $ \matU_{(i),l,\tau+\fracud}=\matU_{(i),l,\tau}-\eta_{\tau}\nabla_{\matU_{(i),l}}\tf_i$
\STATE Update $ \matV_{(i),l,\tau+1}=\matV_{(i),l,\tau}-\eta_{\tau}\nabla_{\matV_{(i),l}}\tf_i$

\ENDFOR
\STATE Calculates $\matU_{g,\tau+1}=\frac{1}{N}\sum_{i=1}^N\matU_{(i),g,\tau+1}$

\ENDFOR
\STATE Return $ \matU_{g,R}, \{\matV_{(i),g,R}\},\{\matU_{(i),l,R}\},\{\matV_{(i),l,R}\}$.
\end{algorithmic}
\end{algorithm}

In Algorithm \ref{alg:hmf}, we use $\tau$ to denote the iteration index, where the half-integer index denotes the update of the variable is half complete: it is updated by gradient descent but is not feasible yet. It is proven that under a group of sufficient conditions, Algorithm~\ref{alg:hmf} converges to the optimal solutions of problem \eqref{eqn:hmfproblemformulation}. The sufficient conditions require the stepsize $\eta_{\tau}$ to be chosen appropriately and the initialization close to the optimal solution \citep{hmf}. 

In practice, Algorithm \ref{alg:hmf} is often efficient and accurate. Therefore, we implement \hmf as the subroutine \inneralg for all of our numerical simulations in Section \ref{sec:numericalexperiment}. To initialize Algorithm \ref{alg:hmf}, we adopt a spectral initialization approach. Specifically, we concatenate all matrices column-wise to form $\matM^{concat}=[\matM_{(1)},\matM_{(2)},\cdots,\matM_{(N)}]$. Subsequently, we perform a Singular Value Decomposition (SVD) on the concatenated matrix $\matM^{concat}$ to extract the top $r_1$ column singular vectors, which serve as the initialization for $\matU_{g,1}$ in Algorithm \ref{alg:hmf}. Utilizing the calculated $\matU_{g,1}$, we deflate $\matM_{(i)}$ by subtracting the projection of $\matM_{(i)}$ onto $\matU_{g,1}$, denoted as $\matM_{(i)}^{deflate} = \matM_{(i)} -\matU_{g,1}\matU_{g,1}^T\matM_{(i)}$. We then conduct another SVD to identify the top $r_2$ singular vectors of $\matM_{(i)}^{deflate}$, which are utilized as the initialization for $\matU_{(i),l,\fracud}$. The initializations for the coefficient matrices are established as $\matV_{(i),g,\fracud} = \matM_{(i)}^T\matU_{g,1}$ and $\matV_{(i),l,1} = \matM_{(i)}^T\matU_{(i),l,\fracud}$.

The stepsize $\eta$ in Algorithm~\ref{alg:hmf} is individually adjusted for each dataset to achieve the fastest convergence. We choose a large total number of iterations $R$ to ensure a small optimization error $\epsilon$. 
Specifically, in the synthetic data, we set the stepsize to $0.005$ and $R=500$. In the video segmentation task, we set the stepsize to $5\times10^{-6}$ and $R=200$. And on the hot rolling data, we set the stepsize to $4\times 10^{-5}$ and $R=500$. In our experiments, we observe that the regularization parameter $\beta$ exerts a negligible influence on the convergence of Algorithm~\ref{alg:hmf}. Consequently, we maintain $\beta$ within the range of $10^{-6}$ to $10^{-5}$ in all our experiments.

\subsection{Personalized PCA}
Personalized PCA \citep{personalizedpca} is another subroutine to solve \eqref{eqn:subproblemobj}. More specifically, personalized PCA seeks to find orthonormal features $\matU_g$ and $\matU_{(i),l}$ to minimize the residual of fitting, as shown in the following objective,
\begin{equation}
\label{eqn:perpcaproblemformulation}
\begin{aligned}
\min_{\matU_g,\{\matU_{(i),l}\}_{i=1,\cdots,N}}& \frac{1}{2}\sum_{i=1}^N\norm{\hmatM_{(i)}-\matU_g\matU_g^T\hmatM_{(i)}-\matU_{(i),l}\matU_{(i),l}^T\hmatM_{(i)}}_F^2\\
\text{subject to  }&  
\matU_g^T\matU_g=\matI,\, \matU_{(i),l}^T\matU_{(i),l}=\matI,\,  \matU_g^T\matU_{(i),l}=\bm{0},\, \forall i. \\
\end{aligned}
\end{equation}
The objective only optimizes the feature matrices $\matU_g$ and $\matU_{(i),l}$, but it's essentially equivalent to problem \eqref{eqn:subproblemobj}. The formal statement is presented in the following proposition.

\begin{proposition}
Let $\hmatU_g^{\perpca},\{\hmatU^{\perpca}_{(i),l}\}$ be \textit{one} set of optimal solutions to~\eqref{eqn:perpcaproblemformulation}, then $\hmatU_g^{\perpca},\{\hmatM_{(i)}^T\hmatU_g^{\perpca},\hmatU^{\perpca}_{(i),l}, \hmatM_{(i)}^T\hmatU_{(i),l}^{\perpca}\}$ is also a set of optimal solution to~\eqref{eqn:subproblemobj}
\end{proposition}

\begin{proof}
We will also prove the proposition by contradiction. If $\hmatU_g^{\perpca},\{\hmatM_{(i)}^T\hmatU_g^{\perpca},\hmatU^{\perpca}_{(i),l}, \hmatM_{(i)}^T\hmatU_{(i),l}^{\perpca}\}$ is not a set of optimal solution to~\eqref{eqn:subproblemobj}, we can find a different set of feasible solutions $\hmatU_g^{\inneralg},\{\hmatV^{\inneralg}_{(i),g},\hmatU^{\inneralg}_{(i),l},\hmatV^{\inneralg}_{(i),l}\}$ such that
$$
\begin{aligned}
&\sum_{i=1}^N f_i\left(\hmatU^{\inneralg}_g,\hmatV^{\inneralg}_{(i),g},\hmatU^{\inneralg}_{(i),l},\hmatV^{\inneralg}_{(i),l}\right)\\
&< \sum_{i=1}^Nf_i\left(\hmatU_g^{\perpca},\hmatM_{(i)}^T\hmatU_g^{\perpca},\hmatU^{\perpca}_{(i),l}, \hmatM_{(i)}^T\hmatU_{(i),l}^{\perpca}\right) \\
& = \sum_{i=1}^N\norm{\hmatM_{(i)}-\matU^{\perpca}_g\matU^{\perpca}_g{}^T\hmatM_{(i)}-\matU^{\perpca}_{(i),l}\matU^{\perpca}_{(i),l}{}^T\hmatM_{(i)}}_F^2.
\end{aligned}
$$

If we fix $\matU_g$ and $\matU_{(i),l}$ to be $\hmatU^{\inneralg}_g$ and $\hmatU^{\inneralg}_{(i),l}$ in problem~\eqref{eqn:subproblemobj}, then the optimal solution of $\matV_{(i),g}$ and $\matV_{(i),l}$ is $\matV^{\inneralg,opt}_{(i),g} = \hmatM_{(i)}^T\hmatU^{\inneralg}_g\left(\hmatU^{\inneralg}_g{}^T\hmatU^{\inneralg}_g\right)^{-1}$ and $\matV^{\inneralg,opt}_{(i),l} = \hmatM_{(i)}^T\hmatU^{\inneralg}_{(i),l}\left(\hmatU^{\inneralg}_{(i),l}{}^T\hmatU^{\inneralg}_{(i),l}\right)^{-1}$. As a result,
$$
\begin{aligned}
&\sum_{i=1}^N\norm{\hmatM_{(i)}-\hmatU^{\inneralg}_{(i),l}\left(\hmatU^{\inneralg}_{(i),l}{}^T\hmatU^{\inneralg}_{(i),l}\right)^{-1}\hmatU^{\inneralg}_{(i),l}{}^T\hmatM_{(i)}-\hmatU^{\inneralg}_{(i),g}\left(\hmatU^{\inneralg}_{(i),g}{}^T\hmatU^{\inneralg}_{(i),g}\right)^{-1}\hmatU^{\inneralg}_{(i),g}{}^T\hmatM_{(i)}}_F^2\\
&=\sum_{i=1}^N f_i\left(\hmatU^{\inneralg}_g,\hmatV^{\inneralg,opt}_{(i),g},\hmatU^{\inneralg}_{(i),l},\hmatV^{\inneralg,opt}_{(i),l}\right)\\
&\le \sum_{i=1}^N f_i\left(\hmatU^{\inneralg}_g,\hmatV^{\inneralg}_{(i),g},\hmatU^{\inneralg}_{(i),l},\hmatV^{\inneralg}_{(i),l}\right)\\
&< \sum_{i=1}^N\norm{\hmatM_{(i)}-\matU^{\perpca}_g\matU^{\perpca}_g{}^T\hmatM_{(i)}-\matU^{\perpca}_{(i),l}\matU^{\perpca}_{(i),l}{}^T\hmatM_{(i)}}_F^2.
\end{aligned}
$$
If we define $\hmatU^{\inneralg,refine}_{(i),l}=\hmatU^{\inneralg}_{(i),l}\left(\hmatU^{\inneralg}_{(i),l}{}^T\hmatU^{\inneralg}_{(i),l}\right)^{-1/2}$ and $\hmatU^{\inneralg,refine}_{(i),g}=\hmatU^{\inneralg}_{(i),g}\left(\hmatU^{\inneralg}_{(i),g}{}^T\hmatU^{\inneralg}_{(i),g}\right)^{-1/2}$, then $\hmatU^{\inneralg,refine}_{(i),l}$ and $\hmatU^{\inneralg,refine}_{(i),g}$ are also feasible for~\eqref{eqn:perpcaproblemformulation} and achieve lower objective. This contradicts the optimality of $\matU_g^{\perpca}$ and $\matU_{(i),l}^{\perpca}$.

\end{proof}

To solve the constrained optimization problem~\eqref{eqn:perpcaproblemformulation}, personalized PCA adopts a distributed version of Stiefel gradient descent. The pseudo-code is presented in Algorithm \ref{alg:perpca}.

\begin{algorithm}
\caption{\inneralg by personalized PCA}
\label{alg:perpca}
\begin{algorithmic}
\STATE Input observation matrices $\{\hmatM_{(i)}\}_{i=1}^N$, stepsize $\eta_\tau$, iteration budget $R$.
\STATE Initialize $\matU_{g,1}$, and $\matU_{(1),l,\frac{1}{2}},\cdots,\matU_{(N),l,\frac{1}{2}}$.
\STATE Calculate $\matS_{(i)}=\hmatM_{(i)}\hmatM_{(i)}^T$ for each $i$.
\FOR{iteration $\tau=1,...,R$}
\FOR{ index $i=1,\cdots,N$}
\STATE Correct $\matU_{(i),l,\tau}=\grof{\matU_{(i),l,\tau-\fracud}}{-\matU_{g,\tau}\matU_{g,\tau}\matU_{(i),l,\tau-\fracud} }$ 

\STATE Calculate $\matG_{(i),\tau}=\left(\matI-\matU_{g,\tau}\matU_{g,\tau}^T-\matU_{(i),l,\tau}\matU_{(i),l,\tau}^T\right)\left(\bm{S}_{(i)}\left[\matU_{g,\tau},\matU_{(i),l,\tau}\right]\right)$ 
\STATE Update $ \matU_{(i),g,\tau+1}=\matU_{g,\tau}+\eta_{\tau}(\matG_{(i),\tau})_{1:d,1:r_1}$ 
\STATE Update $ \matU_{(i),l,\tau+\frac{1}{2}}=\grof{\matU_{(i),l,\tau}}{ \eta_{\tau}(\matG_{(i),\tau})_{1:d,(r_1+1):(r_1+r_{2,(i)})}}$

\ENDFOR
\STATE Update $\matU_{g,\tau+1}=\grof{\matU_{g,\tau}}{\frac{1}{N}\sum_{i=1}^N\matU_{(i),g,\tau+1}-\matU_{g,\tau}}$ 
\ENDFOR
\STATE Calculate $\matV_{(i),g,R}=\hmatM_{(i)}^T\matU_{g,R}$ and $\matV_{(i),l,R}=\hmatM_{(i)}^T\matU_{(i),l,R}$.
\STATE Return $ \matU_{g,R}, \{\matV_{(i),g,R}\},\{\matU_{(i),l,R}\},\{\matV_{(i),l,R}\}$.
\end{algorithmic}
\end{algorithm}
In Algorithm \ref{alg:perpca}, $\mathcal{GR}$ denotes generalized retraction. In practice, it can be implemented via polar projection $\grof{\matU}{\matV}=\left(\matU+\matV\right)\left(\matU^T\matU+\matV^T\matU+\matU^T\matV+\matV^T\matV\right)^{-\fracud}$. Algorithm \ref{alg:perpca} can also be proved to converge to the optimal solutions with suitable choices of stepsize and initialization~\citep{personalizedpca}.
\revise{\section{Additional running time comparisons}
\label{ap:runtimecomparison}
In this section, we include the additional running time comparison between \name and Robust PCA. We use a set of synthetic datasets with varying numbers of sources $N$, then compare the per-iteration running time of the two algorithms. More specifically, we follow the setting in Section~\ref{sec:syntheticdata} where $n_1=15$ and $n_2=1000$, and generate synthetic datasets where the number of sources $N$ changes from $100$ to $10000$. Then, we apply \name and Robust PCA on the same dataset. We do not parallelize computations for either algorithm for fair comparison. The per-iteration running time of the two algorithms is collected and plotted in Figure~\ref{fig:runtime_comparison}.} 

\begin{figure}[H]
    \centering
    \includegraphics[width=0.4\linewidth]{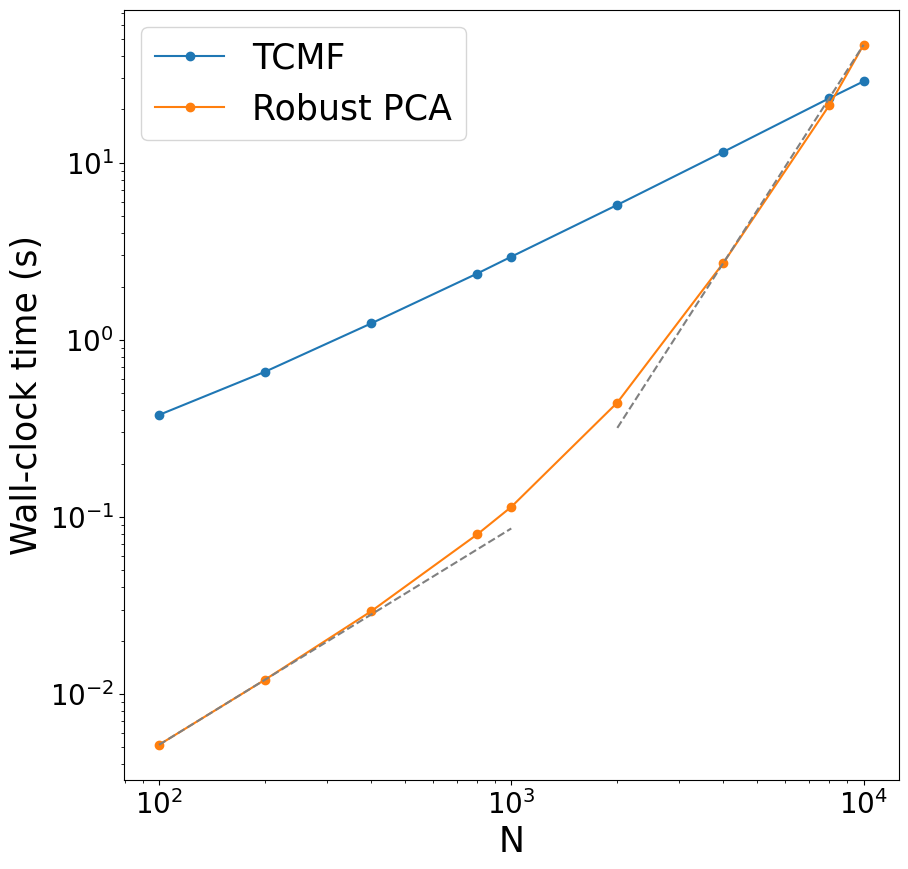}
    \caption{\revise{Running time comparison between runtime of \name and \robustpca.}}
    \label{fig:runtime_comparison}
\end{figure}

\revise{From Figure~\ref{fig:runtime_comparison}, it is clear that the running time of \name scales linearly with the number of sources $N$, which is consistent with the complexity analysis.}

\revise{In contrast, though \robustpca has a smaller per-iteration runtime when $N$ is small, as $N$ becomes larger, the runtime increases faster than \name. This is because \robustpca vectorizes the observation matrices from each source. The resulting vector from each source has dimension $n_1n_2$. \robustpca then concatenates these vectors into a $n_1n_2\times N$ matrix and alternatively performs Singular Value Decomposition and hard-thresholding. For each application of SVD, the computational complexity is $\mathcal{O}(n_1n_2(N^2+N) + N^3 )$ when $N\le n_1n_2$~\citep{svdcomplexity}. Indeed, in Figure~\ref{fig:runtime_comparison}, the slope of the initial part of the Robust PCA curve is around $1.2$, and the slope of the final part is around $3.0$, suggesting that the running time scales cubically in the large $N$ regime.}

\revise{Such comparison highlights \name's computational advantage when $N$ is large. }

\section{Proof of Theorem~\ref{thm:recovery}}
\label{sec:proofofrecovery}
 In this section, we will introduce the details of the proof of Theorem~\ref{thm:recovery}. We will firstly introduce a few basic lemmas, then prove the KKT conditions in Lemma \ref{lm:kktcondition}. Based on the KKT conditions, we introduce an infinite series to represent the solutions to \eqref{eqn:subproblemobj}. Next, we will prove Lemma \ref{lm:ldiffinfnorm}, which is a formal version of Lemma \ref{lm:informalldiffinfnorm}. Finally, we will use induction to prove Theorem~\ref{thm:formalrecovery}, which is the formal version of Theorem~\ref{thm:recovery}.
 
 Remember that we use $\matLst_{(i)}$ to denote $\matLst_{(i)}=\matLst_{(i),g}+\matLst_{(i),l}$, where $\matLst_{(i),g}$ and $\matLst_{(i),l}$ are the global and local components for source $i$ defined as
$\matLst_{(i),g}=\matUst_{g}\matVst_{(i),g}^T$ and $\matLst_{(i),l}=\matUst_{(i),l}\matVst_{(i),l}^T$. We assume all nonzero singular values of $\matLst_{(i)}$ are lower bounded by $\sm>0$ and upper bounded by $\gop>0$. As introduced in the proof sketch, we use $\matE_{(i),t}=\matSst_{(i)}-\hmatS_{(i),t}$ to denote the difference between our estimate of the sparse noise at epoch $t$ and the ground truth. The following notations will be used throughout our proof:
\begin{align}
&\matF_{(i)}=\matE_{(i),t}\matLst_{(i)}^T+\matLst_{(i)}\matE_{(i),t}^T+\matE_{(i),t}\matE_{(i),t}^T, \ \ i\in[N],\quad \text{and}\quad  \matF_{(0)}=\frac{1}{N}\sum_{i=1}^N\matF_{(i)}\label{eq_F}\\
&    \matT_{(i)}=\matLst_{(i)}\matLst_{(i)}^T, \ \ i\in[N],\quad \text{and}\quad  \matT_{(0)}=\frac{1}{N}\sum_{i=1}^N\matT_{(i)}.
\end{align}
Since in the ground truth model, the SVD of $\matLst_{(i)}$ can be written as $\matLst_{(i)}=\left[\matHst_g,\matHst_{(i),l}\right]\text{diag}(\matSigma_{(i),g},\matSigma_{(i),l})\left[\matWst_{(i),g},\matWst_{(i),l} \right]^T$, one can immediately see that $\matT_{(i)}$'s nonzero eigenvalues are upper bounded by $\gop^2$ and lower bounded by $\sms$. Finally, recall that we use $\hmatU_{g}$, $\hmatU_{(i),l}$, $\hmatV_{(i),g}$, and $\hmatV_{(i),l}$ to denote the optimal solutions to \eqref{eqn:subproblemobj} (we omit the subscript $t$ here for brevity.) For a series of square matrices of the same shape $\matA_1,\cdots, \matA_k\in \mathbb{R}^{r\times r}$, we use $\prod_{m=1}^k\matA_m$ to denote the product of these matrices in the ascending order of indices, and $\prod_{m=k}^1\matA_m$ to denote the product of these matrices in the descending order of indices,
$$
\begin{aligned}
&\prod_{m=1}^k\matA_m=\matA_1\matA_{2}\cdots \matA_{k-1}\matA_k\\
&\prod_{m=k}^1\matA_m=\matA_k\matA_{k-1}\cdots \matA_2\matA_1.
\end{aligned}
$$

Our next two lemmas provide upper bound on the maximum row-norm of the errors with respect to the $\ell_\infty$-norms of $\matE_{(i)}$. By building upon these two lemmas, we provide a key result in Lemma~\ref{lm:efuinfnorm} connecting $\{\matF_{(i)}\}$ and the error matrices $\{\matE_{(i)}\}$.

\begin{lemma}
\label{lm:eetpincoherence}
Suppose that $\matE_{(1)},\cdots,\matE_{(N)},\in \mathbb{R}^{n_1\times n_2}$ are $\alpha$-sparse and $\matU\in \mathbb{R}^{n_1\times r}$ is $\mu$-incoherent and $\norm{U}\le 1$. For any integers $p_1,p_2,\cdots,p_k\ge 0$, and $i_1,i_2,\cdots,i_k\in\{0,1,\cdots,N\}$, we have
\begin{equation}
\label{eqn:eetpincoherence}
\max_j\norm{\vece_j^T\left(\prod_{\ell=1}^k(\matE_{(i_\ell)}\matE_{(i_\ell)}^T)^{p_\ell}\right)\matU}_2\le \sqrt{\murnone}\left(\abnb\right)^{2(p_1+p_2+\cdots+p_k)}.
\end{equation}
\end{lemma}
With a slight abuse of notation, in Lemma~\ref{lm:eetpincoherence} and the rest of the paper, we define $\matE_{(0)}\matE_{(0)}^T$ to be,
\begin{equation}
\label{eqn:e0e0tdef}
\matE_{(0)}\matE_{(0)}^T = \frac{1}{N}\sum_{i=0}^N\matE_{(i)}\matE_{(i)}^T.
\end{equation}
\begin{proof}
We will prove it by induction on the exponent.
From the definition of incoherence, we know that when $p_1+\cdots+p_k=0$, the inequality \eqref{eqn:eetpincoherence} holds.
Now suppose that the inequality \eqref{eqn:eetpincoherence} holds for all $p_1,p_2,\cdots,p_k\ge 0$ such that $p_1+\cdots+p_k\le s-1$ and $i_1,i_2,\cdots,i_k\in\{0,1,\cdots,N\}$. We will prove the statement for $p_1+\cdots+p_k= s$. Without loss of generality, we assume $p_1\ge 1$. One can write

\begin{equation}
\begin{aligned}
& \norm{\vece_j^T\left(\prod_{\ell=1}^k(\matE_{(i_\ell)}\matE_{(i_\ell)}^T)^{p_\ell}\right)\matU}_2^2= \sum_l\left(\vece_j^T\left(\prod_{\ell=1}^k(\matE_{(i_\ell)}\matE_{(i_\ell)}^T)^{p_\ell}\right)\matU\vece_l\right)^2\\    &=\sum_l\left(\vece_j^T\matE_{(i_1)}\matE_{(i_1)}^T(\matE_{(i_1)}\matE_{(i_1)}^T)^{p_1-1}\left(\prod_{\ell=2}^k(\matE_{(i_\ell)}\matE_{(i_\ell)}^T)^{p_\ell}\right)\matU\vece_l\right)^2\\
&=\sum_l\left(\sum_h\left[\matE_{(i_1)}\matE_{(i_1)}^T\right]_{j,h}\vece_h^T(\matE_{(i_1)}\matE_{(i_1)}^T)^{p_1-1}\left(\prod_{\ell=2}^k(\matE_{(i_\ell)}\matE_{(i_\ell)}^T)^{p_\ell}\right)\matU\vece_l\right)^2\\
&=\sum_l\sum_{h_1,h_2}\left[\matE_{(i_1)}\matE_{(i_1)}^T\right]_{j,h_1}\left[\matE_{(i_1)}\matE_{(i_1)}^T\right]_{j,h_2}\\
&\times\vece_{h_1}^T(\matE_{(i_1)}\matE_{(i_1)}^T)^{p_1-1}\left(\prod_{\ell=2}^k(\matE_{(i_\ell)}\matE_{(i_\ell)}^T)^{p_\ell}\right)\matU\vece_l\vece_l^T\matU^T\left(\prod_{\ell=k}^2(\matE_{(i_\ell)}\matE_{(i_\ell)}^T)^{p_\ell}\right)(\matE_{(i_1)}\matE_{(i_1)}^T)^{p_1-1}\vece_{h_2}\\
\end{aligned}
\end{equation}.

Since $\sum_l\vece_l\vece_l^T=\matI$, we can simplify the summation as,

\begin{align*}
&\sum_{h_1,h_2}\left(\matE_{(i_1)}\matE_{(i_1)}^T\right)_{j,h_1}\left(\matE_{(i_1)}\matE_{(i_1)}^T\right)_{j,h_2}\\
&\times\vece_{h_1}^T(\matE_{(i_1)}\matE_{(i_1)}^T)^{p_1-1}\left(\prod_{\ell=2}^k(\matE_{(i_\ell)}\matE_{(i_\ell)}^T)^{p_\ell}\right)\matU\matU^T\left(\prod_{\ell=k}^2(\matE_{(i_\ell)}\matE_{(i_\ell)}^T)^{p_\ell}\right)(\matE_{(i_1)}\matE_{(i_1)}^T)^{p_1-1}\vece_{h_2}\\
&\le\left(\sum_{h_1,h_2}\left(\matE_{(i_1)}\matE_{(i_1)}^T\right)_{j,h_1}\left(\matE_{(i_1)}\matE_{(i_1)}^T\right)_{j,h_2}\right)\\
&\times\max_{m}\norm{\vece_{m}^T(\matE_{(i_1)}\matE_{(i_1)}^T)^{p_1-1}\left(\prod_{\ell=2}^k(\matE_{(i_\ell)}\matE_{(i_\ell)}^T)^{p_\ell}\right)\matU}_2^2\\
&\le \sum_{h_1,h_2}\left(\matE_{(i_1)}\matE_{(i_1)}^T\right)_{j,h_1}\left(\matE_{(i_1)}\matE_{(i_1)}^T\right)_{j,h_2} \murnone \left(\abnb\right)^{4s-4},
\end{align*}
where in the last step, we used the induction hypothesis. Now, to complete the proof, we consider two cases. 

If $i_1>0$, we have:
\begin{equation*}
\begin{aligned}
&\sum_{h_1,h_2}\left(\matE_{(i_1)}\matE_{(i_1)}^T\right)_{j,h_1}\left(\matE_{(i_1)}\matE_{(i_1)}^T\right)_{j,h_2}=\sum_{h_1,h_2,g_1,g_1}(\matE_{(i_1)})_{j,g_1}(\matE_{(i_1)})_{h_1,g_1}(\matE_{(i_1)})_{j,g_2}(\matE_{(i_1)})_{h_2,g_2}\\
&\le \alpha n_1\norm{\matE_{(i_1)}}_{\infty}\alpha n_2\norm{\matE_{(i_1)}}_{\infty}\alpha n_1\norm{\matE_{(i_1)}}_{\infty}\alpha n_2\norm{\matE_{(i_1)}}_{\infty}=\left(\abnb\right)^4,
\end{aligned}
\end{equation*}
where the last inequality holds because at most $\alpha n_1$ entries in each column of $\matE_{(i_1)}$ are nonzero and at most $\alpha n_2$ entries in each row of $\matE_{(i_1)}$ are nonzero.
On the other hand, if $i_1=0$, we have:
\begin{equation*}
\begin{aligned}
&\sum_{h_1,h_2}\left(\matE_{(0)}\matE_{(0)}^T\right)_{j,h_1}\left(\matE_{(0)}\matE_{(0)}^T\right)_{j,h_2}\le\frac{1}{N^2}\sum_{h_1,h_2}\left(\sum_{f_1>0}\matE_{(f_1)}\matE_{(f_1)}^T\right)_{jk_1}\left(\sum_{f_2>0}\matE_{(f_2)}\matE_{(f_2)}^T\right)_{jk_2}\\
&=\frac{1}{N^2}N^2 \left(\abnb\right)^4.
\end{aligned}
\end{equation*}
Therefore, in both cases, we have,
$$
\begin{aligned}
&\norm{\vece_j^T\prod_{\ell=1}^k(\matE_{(i_{\ell})}\matE_{(i_{\ell})}^T)^{p_{\ell}}\matU}_2^2\le \murnone \left(\abnb\right)^{4s},\\
\end{aligned}
$$
for every possible $i_1$ and every $j$. This concludes our proof. 
\end{proof}
Next, we present a slightly different lemma.
\begin{lemma}
\label{lm:eetpeytincoherence}
Suppose that $\matE_{(1)},\cdots,\matE_{(N)}\in \mathbb{R}^{n_1\times n_2}$ are $\alpha$-sparse and  $\matV\in \mathbb{R}^{n_2\times r}$ is $\mu$-incoherent. For any integers $p_1,p_2,\cdots,p_k\ge 0$, and $i_1,i_2,\cdots,i_k\in\{0,1,\cdots,N\}$, we have,
\begin{equation}
\label{eqn:eetpeytincoherence}
\begin{aligned}
&\max_j\norm{\vece_j^T\left(\prod_{\ell=1}^k(\matE_{(i_{\ell})}\matE_{(i_{\ell})}^T)^{p_{\ell}}\right)\matE_{(i_{k+1})}\matV}_2\\
&\le \sqrt{\murnone}\left(\abnb\right)^{2(p_1+p_2+\cdots+p_k)+1}.
\end{aligned}
\end{equation}
\end{lemma}
\begin{proof}
The proof is analogous to that of Lemma \ref{lm:eetpincoherence}, and hence, omitted for brevity.
\end{proof}

Combining Lemma \ref{lm:eetpincoherence} and \ref{lm:eetpeytincoherence}, we can show the following key lemma on the connection between $\{\matF_{(i)}\}$ and the error matrices $\{\matE_{(i)}\}$.

\begin{lemma}
\label{lm:efuinfnorm}
For every $i\in [N]$, suppose that 
$\matE_{(i)}\in \mathbb{R}^{n_1\times n_2}$ is $\alpha$-sparse and $\matLst_{(i)}=\matHst_{(i)}\matSigmast_{(i)}\matWst_{(i)}$ is rank-$r$ with $\mu$-incoherent matrices $\matHst_{(i)}\in \mathbb{R}^{n_1\times r}$ and $\matWst_{(i)}\in \mathbb{R}^{n_2\times r}$. For any integers $p_1,p_2,\cdots,p_k\ge 0$, and $i_1,i_2,\cdots,i_k\in\{0,1,\cdots,N\}$, the following holds for any $\mu$-incoherent matrix $\matU\in \mathbb{R}^{n_1\times r}$,
\begin{equation}
\max_j\norm{\vece_j^T\prod_{\ell=1}^k\matF_{(i_{\ell})}^{p_{\ell}}\matU}_2\le \sqrt{\murnone}\left(\abnb(\abnb+2\gop)\right)^{p_1+p_2+\cdots+p_k},
\end{equation}
where $\matF_{(i)}$ is defined as~\eqref{eq_F}.
\end{lemma}

\begin{proof}
We firstly expand $\matF_{(i_1)}^{p_1}\matF_{(i_2)}^{p_2}\cdots\matF_{(i_k)}^{p_k}\matU$ and rearrange the terms by the number of consecutive $\matE_{(i)}\matE_{(i)}^T$ terms appearing in the beginning of each factor.

$$
\begin{aligned}
&\matF_{(i_1)}^{p_1}\matF_{(i_2)}^{p_2}\cdots\matF_{(i_k)}^{p_k}\matU\\
&=\left(\matE_{(i_1)}(\matLst_{i_1})^T+\matLst_{i_1}\matE_{(i_1)}^T+\matE_{(i_1)}\matE_{(i_1)}^T\right)\cdots\left(\matE_{(i_1)}(\matLst_{i_1})^T+\matLst_{i_1}\matE_{(i_1)}^T+\matE_{(i_1)}\matE_{(i_1)}^T\right)\\
&\left(\matE_{(i_2)}(\matLst_{i_2})^T+\matLst_{i_2}\matE_{(i_2)}^T+\matE_{(i_2)}\matE_{(i_2)}^T\right)\cdots\left(\matE_{(i_2)}(\matLst_{i_2})^T+\matLst_{i_2}\matE_{(i_2)}^T+\matE_{(i_2)}\matE_{(i_2)}^T\right)\\
&\cdots\\
&\left(\matE_{(i_k)}(\matLst_{(i_k)})^T+\matLst_{(i_k)}\matE_{(i_k)}^T+\matE_{(i_k)}\matE_{(i_k)}^T\right)\cdots\left(\matE_{(i_k)}(\matLst_{(i_k)})^T+\matLst_{(i_k)}\matE_{(i_k)}^T+\matE_{(i_k)}\matE_{(i_k)}^T\right)\matU\\
&=\left(\matE_{(i_1)}\matE_{(i_1)}^T\right)^{p_1}\cdots\left(\matE_{(i_{k})}\matE_{(i_{k})}^T\right)^{p_{k}}\matU\\
&+\sum_{r=0}^{p_1+\cdots+p_k-1}\left(\matE_{(i_1)}\matE_{(i_1)}^T\right)^{p_1}\cdots\left(\matE_{(i_{t-1})}\matE_{(i_{t-1})}^T\right)^{p_{t-1}}\left(\matE_{(i_t)}\matE_{(i_t)}^T\right)^{r-(\sum_{g=1}^{t-1}p_g)}\\
&\cdot \left(\matE_{(i_t)}\matLst_{(i_t)}^T+\matLst_{(i_t)}\matE_{(i_t)}^T\right)\matF_{(i_t)}^{(\sum_{g=1}^{t}p_g)-1-r}\matF_{(i_{t+1})}^{p_{t+1}}\cdots\matF_{(i_k)}^{p_k}\matU.
\end{aligned}
$$
For the first term, by Lemma \ref{lm:eetpincoherence}, we have
$$
\norm{\vece_j^T\left(\matE_{(i_1)}\matE_{(i_1)}^T\right)^{p_1}\cdots\left(\matE_{i_{k}}\matE_{i_{k}}^T\right)^{p_{k}}\matU}_2\le \sqrt{\murnone} \left(\abnb\right)^{2(p_1+\cdots+p_k)}.
$$
For the remaining terms, we have
$$
\begin{aligned}
&\left(\matE_{(i_1)}\matE_{(i_1)}^T\right)^{p_1}\cdots\left(\matE_{(i_{t-1})}\matE_{(i_{t-1})}^T\right)^{p_{t-1}}\left(\matE_{(i_t)}\matE_{(i_t)}^T\right)^{r-(\sum_{g=1}^{t-1}p_g)}\\
&\left(\matE_{(i_t)}(\matLst_{(i_t)})^T+\matLst_{(i_t)}\matE_{(i_t)}^T\right)\matF_{(i_t)}^{(\sum_{g=1}^{t}p_g)-1-r}\matF_{(i_{t+1})}^{p_{t+1}}\cdots\matF_{(i_k)}^{p_k}\matU\\
&=\left(\matE_{(i_1)}\matE_{(i_1)}^T\right)^{p_1}\cdots\left(\matE_{(i_{t-1})}\matE_{(i_{t-1})}^T\right)^{p_{t-1}}\left(\matE_{(i_t)}\matE_{(i_t)}^T\right)^{r-(\sum_{g=1}^{t-1}p_g)}\matE_{(i_t)}\matWst_{(i_t)}\\
&\times \matSigmast_{(i_t)}\matHst_{(i_t)}^T\matF_{(i_t)}^{(\sum_{g=1}^{t}p_g)-1-r}\matF_{(i_{t+1})}^{p_{t+1}}\cdots\matF_{(i_k)}^{p_k}\matU
\\
&+\left(\matE_{(i_1)}\matE_{(i_1)}^T\right)^{p_1}\cdots\left(\matE_{(i_{t-1})}\matE_{(i_{t-1})}^T\right)^{p_{t-1}}\left(\matE_{(i_t)}\matE_{(i_t)}^T\right)^{r-(\sum_{g=1}^{t-1}p_g)}\matHst_{(i_t)}\\
&\times \matSigmast_{(i_t)}\matWst_{(i_t)}^T\matE_{(i_t)}^T\matF_{(i_t)}^{(\sum_{g=1}^{t}p_g)-1-r}\matF_{(i_{t+1})}^{p_{t+1}}\cdots\matF_{(i_k)}^{p_k}\matU.
\end{aligned}
$$
We can bound the two terms separately. By Lemma \ref{lm:eetpeytincoherence},
$$
\begin{aligned}
&\norm{\vece_j^T\left(\matE_{(i_1)}\matE_{(i_1)}^T\right)^{p_1}\cdots\left(\matE_{(i_{t-1})}\matE_{(i_{t-1})}^T\right)^{p_{t-1}}\left(\matE_{(i_t)}\matE_{(i_t)}^T\right)^{r-(\sum_{g=1}^{t-1}p_g)}\matE_{(i_t)}\matWst_{(i_t)}}\\
&\le \sqrt{\murnone}\left(\abnb\right)^{2r+1}.
\end{aligned}
$$
And by Lemma \ref{lm:eetpincoherence},
$$
\begin{aligned}
&\norm{\vece_j^T\left(\matE_{(i_1)}\matE_{(i_1)}^T\right)^{p_1}\cdots\left(\matE_{(i_{t-1})}\matE_{(i_{t-1})}^T\right)^{p_{t-1}}\left(\matE_{(i_t)}\matE_{(i_t)}^T\right)^{r-(\sum_{g=1}^{t-1}p_g)}\matHst_{(i_t)}}\\
&\le \sqrt{\murnone}\left(\abnb\right)^{2r}.
\end{aligned}
$$
For an $\alpha$-sparse matrix $\matE_{(i)}\in \mathbb{R}^{n_1\times n_2}$, its operator norm is bounded by 
$$
\begin{aligned}
&\norm{\matE_{(i)}}_2=\max_{\norm{\vecv}=1,\norm{\vech}=1}\vecv^T \matE_{(i)}\vech=\max_{\norm{\vecv}=1,\norm{\vech}=1}\sum_{j,k}\vecv_j\vech_k[\matE_{(i)}]_{jk}\\
&\le \max_{\norm{\vecv}=1,\norm{\vech}=1}\sum_{j,k}\fracud\left(\vecv_j^2\sqrt{\frac{n_1}{n_2}}+\vech_k^2\sqrt{\frac{n_2}{n_1}}\right)[\matE_{(i)}]_{jk}\\
&\le \max_{\norm{\vecv}=1,\norm{\vech}=1}\norm{\matE_{(i)}}_{\infty}\fracud\left(\sum_{j}\vecv_j^2\sqrt{\frac{n_1}{n_2}}\alpha n_2+\sum_{k}\vech_k^2 \sqrt{\frac{n_2}{n_1}}\alpha n_1\right)\\
&=\alpha \sqrt{n_1n_2} \norm{\matE_{(i)}}_{\infty}.
\end{aligned}
$$
Therefore $\norm{\matE_{(i)}}_2\le \abnb$. As a result, we know,
$$
\begin{aligned}
&\norm{\matF_{(i)}}_2\le 2\gop\abnb+\left(\abnb\right)^2\\
&=\abnb\left(2\gop+\abnb\right).    
\end{aligned}
$$

We thus have:

\begin{align*}
&\norm{\vece_j^T\matF_{(i_1)}^{p_1}\matF_{(i_2)}^{p_2}\cdots\matF_{(i_k)}^{p_k}\matUst}\\
&\le \sqrt{\murnone}\Big(\left(\abnb\right)^{2(\sum_{\ell=1}^kp_{\ell})}\\
&+\sum_{r=0}^{\sum_{\ell=1}^kp_{\ell}-1}\left(\abnb\right)^{2r+1}\\
&\times2\gop\left(\abnb\left(2\gop+\abnb\right)\right)^{\sum_{\ell=1}^kp_{\ell}-1-r}\Big)\\
&=\sqrt{\murnone}\Big(\left(\abnb\right)^{2\sum_{\ell=1}^kp_{\ell}}\\
&+\left(\abnb\right)^{\sum_{\ell=1}^kp_{\ell}}\Big(\left(2\gop+\abnb\right)^{\sum_{\ell=1}^kp_{\ell}}\\
&-\left(\abnb\right)^{\sum_{\ell=1}^kp_{\ell}}\big)\Big)\\
&=\sqrt{\murnone}\left(\abnb\right)^{\sum_{\ell=1}^kp_{\ell}}\left(\abnb+2\gop\right)^{\sum_{\ell=1}^kp_{\ell}}.
\end{align*}

This finishes our proof.
\end{proof}

Lemma \ref{lm:efuinfnorm} is an important lemma as it provides an upper bound on the maximum row norm of the product of a group of sparse matrices and an incoherent matrix. We will use Lemma \ref{lm:efuinfnorm} extensively when we calculate the $\ell_{\infty}$ norm of error terms in the output of \inneralg.

Next, we prove that the optimal solution indeed satisfies the KKT conditions delineated in Lemma~\ref{lm:kktcondition}.

\noindent{\bf Proof of Lemma \ref{lm:kktcondition}}\quad
The proof is presented in three parts. In the first part, we show that the optimal solution optimal $\hmatU_g$, $\{\hmatV_{(i),g},\hmatU_{(i),l},\hmatV_{(i),l}\}$ satisfies the linear independence constraint qualification (LICQ). This ensures that the optimal solution satisfies the KKT conditions. In the second part, we prove the validity of the equations in~\eqref{eqn:kktinnerloop1}. Finally, we prove the correctness of the equations in~\eqref{eqn:kktderivations}.

\noindent \underline{\it Proof of LICQ.} 
We begin by showing $\hmatU_g$ has full column rank. By contradiction, suppose $\hmatU_g$ has rank $r_1^{'}<r_1$. Since $\hmatM_{(i)}$ has rank at least $r_1+r_2$, the residual $\hmatM_{(i)}-\hmatU_g\hmatV_{(i),g}^T-\hmatU_{(i),l}\hmatV_{(i),l}^T$ has rank at least $1$. Therefore we can always find another $\hmatU_g^{'}$ such that $f_i(\hmatU_g^{'},\hmatV_{(i),g},\hmatU_{(i),l},\hmatV_{(i),l})< f_i(\hmatU_g,\hmatV_{(i),g},\hmatU_{(i),l},\hmatV_{(i),l})$. This contradicts the fact that $\hmatU_g$ is optimal.

Next we will establish the LICQ of the constraints. We define $h_{ijk}$ as the inner product between the $j$-th column of $\matU_g$ and the $k$-th column of $\matU_{(i),l}$, $ h_{ijk}(\vecx)=[\matU_g]_{:,j}^T[\matU_{(i),l}]_{:,k}$. The constraints in \eqref{eqn:subproblemobj} can be rewritten as $h_{ijk}(\hat{\vecx})=0,\forall i\in[r_1],\forall j\in[r_2],\forall k\in [N]$. LICQ requires $\nabla h_{ijk}(\hat{\vecx})$ to be linearly independent for all $ijk$ \cite[Proposition 3.1.1]{nonlinearprogramming}. 

Suppose  we can find constants $\psi_{ijk}$ such that $\sum_{i=1}^{N}\sum_{j=1}^{r_1}\sum_{k=1}^{r_2}\psi_{ijk}\nabla h_{ijk}(\hat{\vecx})=0$. We consider the partial derivative of $h_{ijk}$ over the $k^{'}$-th column of $\matU_{(i^{'}),l}$. It is easy to derive,
$$
\frac{\partial}{\partial [\matU_{(i^{'}),l}]_{:,k^{'}}}h_{ijk}(\hat{\vecx})= \delta_{ii^{'}}\delta_{kk^{'}}[\hmatU_g]_{:,j},
$$
where $\delta_{ii^{'}}$ is the Kronecker delta function. Then the constants $\psi_{ijk}$ should satisfy,
$$
\sum_{j=1}^{r_2}\psi_{i^{'}jk^{'}}[\hmatU_g]_{:,j}=0.
$$
As the columns of $\hmatU_g$ are linearly independent, $\psi_{i^{'}jk^{'}}=0$ for each $j$. This holds for any $i^{'}$ and $k^{'}$. Therefore $\psi_{ijk}=0$ for all $i,j,k$.  This implies $\nabla h_{ijk}$'s are linearly independent.

\noindent \underline{\it Proof of Equations~\eqref{eqn:kktinnerloop1}.} The Lagrangian of the optimization problem \eqref{eqn:subproblemobj} can be written as
\begin{equation}
\label{eqn:lagrangian}
\begin{aligned}
\mathcal{L} =& \frac{1}{2}\sum_{i=1}^N\norm{\matU_g\matV_{(i),g}^T+\matU_{(i),l}\matV_{(i),l}^T-\hmatM_{(i)}}_F^2\\
&+\tr{\matLambda_{8,(i)}\matU_g^T\matU_{(i),l}} ,   
\end{aligned}
\end{equation}
where $\matLambda_{8,(i)}$ is the dual variable for the constraint $\matU_g^T\matU_{(i),l}=0$.

Under the LICQ, we know that $\hmatU_g,\{\hmatV_{(i),g},\hmatU_{(i),l},\hmatV_{(i),l}\}$ satisfies KKT condition. Setting the gradient of $\mathcal{L}$ with respect to $\matV_{(i),g}$ and $\matV_{(i),l}$ to zero, we can prove \eqref{subeqn:kktinnerloopvig} and \eqref{subeqn:kktinnerloopvil}. Considering the constraint $\hmatU_g^T\hmatU_{(i),l}=0$, we can solve them as $\hmatV_{(i),g}=\hmatM_{(i)}^T\hmatU_{g}\left(\hmatU_{g}^T\hmatU_{g}\right)^{-1}$ and $\hmatV_{(i),l}=\hmatM_{(i)}^T\hmatU_{(i),l}\left(\hmatU_{(i),l}^T\hmatU_{(i),l}\right)^{-1}$.
Then we examine the gradient of $\mathcal{L}$ with respect to $\matU_{(i),l}$:
$$
\begin{aligned}
\frac{\partial}{\partial \matU_{(i),l}}\mathcal{L} = \left(\matU_g\matV_{(i),g}^T+\matU_{(i),l}\matV_{(i),l}^T-\hmatM_{(i)}\right)\matV_{(i),l}+\matU_g\matLambda_{(8),i}^T.
\end{aligned}
$$
Substituting  $\hmatV_{(i),g}=\hmatM_{(i)}^T\hmatU_{g}\left(\hmatU_{g}^T\hmatU_{g}\right)^{-1}$ and $\hmatV_{(i),l}=\hmatM_{(i)}^T\hmatU_{(i),l}\left(\hmatU_{(i),l}^T\hmatU_{(i),l}\right)^{-1}$ in the above gradient and setting it to zero, we have
$$
\begin{aligned}
&\left(\hmatU_g\left(\hmatU_g^T\hmatU_g\right)^{-1}\hmatU_g^T+\hmatU_{(i),l}\left(\hmatU_{(i),l}^T\hmatU_{(i),l}\right)^{-1}\hmatU_{(i),l}^T-\matI\right)\hmatM_{(i)}\hmatV_{(i),l}\\
&+\hmatU_g\matLambda_{(8),i}^T=0.
\end{aligned}
$$
Left multiplying both sides by $\hmatU_{g}^T$, we have $\matLambda_{8,(i)} = 0$. Left multiplying both sides by $\hmatU_{(i),l}^T$, we have $\hmatU_{(i),l}^T\hmatU_{(i),l}-\matI = 0$. Therefore we also have $\left(\hmatU_g\left(\hmatU_g^T\hmatU_g\right)^{-1}\hmatU_g^T+\hmatU_{(i),l}\left(\hmatU_{(i),l}^T\hmatU_{(i),l}\right)^{-1}\hmatU_{(i),l}^T-\matI\right)\hmatM_{(i)}\matV_{(i),l}=0$.
This proves equation \eqref{subeqn:kktinnerloopuil}.
Now, setting the derivative of $\mathcal{L}$ with respect to $\matU_{g}$ to zero, we have
$$
\frac{\partial}{\partial \matU_{g}}\mathcal{L}=\sum_{i=1}^N\left(\hmatU_g\hmatV_{(i),g}^T+\hmatU_{(i),l}\hmatV_{(i),l}^T-\hmatM_{(i)}\right)\hmatV_{(i),g}=0.
$$
Left multiplying both sides by $\hmatU_g^T$, we have $\hmatU_g^T\hmatU_g-\matI=0$. We have thus proven  \eqref{subeqn:kktinnerloopug}. This completes the proof for \eqref{eqn:kktinnerloop1}. 

\noindent \underline{\it Proof of Equations \eqref{eqn:kktderivations}.} Equation \eqref{subeqn:kktinnerloopuil} can be rewritten as:
\begin{equation}
\label{eqn:mmtuiexpansion}
\hmatM_{(i)}\hmatM_{(i)}^T\hmatU_{(i),l} = \hmatU_{(i),l}\hmatU_{(i),l}^T\hmatM_{(i)}\hmatM_{(i)}^T\hmatU_{(i),l}+\hmatU_{g}\hmatU_{g}^T\hmatM_{(i)}\hmatM_{(i)}^T\hmatU_{(i),l}.  
\end{equation}
Since $\hmatU_{(i),l}^T\hmatM_{(i)}\hmatM_{(i)}^T\hmatU_{(i),l}$ is positive definite, we can use $\matW_{(i),l}\matLambda_{2,(i)}\matW_{(i),l}^T=\hmatU_{(i),l}^T\hmatM_{(i)}\hmatM_{(i)}^T\hmatU_{(i),l}$ to denote its eigen-decomposition, where $\matLambda_{2,(i)}\in \mathbb{R}^{r_1\times r_1}$ is a positive definite diagonal matrix and $\matW_{(i),l}\in \mathbb{R}^{r_1\times r_1}$ is orthonormal. Upon defining $\hmatH_{(i),l}=\hmatU_{(i),l}\matW_{(i),l}$, $\hmatH_{(i),l}$ is also orthonormal as $\hmatH_{(i),l}^T\hmatH_{(i),l}=\matW_{(i),l}^T\hmatU_{(i),l}^T\hmatU_{(i),l}\matW_{(i),l}=\matI$. 
Similarly, we rewrite the equation \eqref{subeqn:kktinnerloopug} as:
\begin{equation}
\label{eqn:mmtugexpansion}
\frac{1}{N}\sum_{i=1}^N\matM_{(i)}\matM_{(i)}^T\hmatU_g = \hmatU_g\hmatU_g^T \frac{1}{N}\sum_{i=1}^N\matM_{(i)}\matM_{(i)}^T\hmatU_g +\frac{1}{N}\sum_{i=1}^N\hmatU_{(i),l}\hmatU_{(i),l}^T\matM_{(i)}\matM_{(i)}^T\hmatU_g. 
\end{equation}
Since $\hmatU_g^T \frac{1}{N}\sum_{i=1}^N\hmatM_{(i)}\hmatM_{(i)}^T\hmatU_g$ is positive definite, we can use $\matW_{g}\matLambda_{1}\matW_{g}^T=\hmatU_g^T \frac{1}{N}\sum_{i=1}^N\hmatM_{(i)}\hmatM_{(i)}^T\hmatU_g$ to denote its eigen decomposition, where $\matLambda_{1}\in \mathbb{R}^{r_1\times r_1}$ is positive diagonal, $\matW_{g}\in \mathbb{R}^{r_1\times r_1}$ is orthogonal $\matW_{g}\matW_{g}^T=\matW_{g}^T\matW_{g}=\matI$. 
We define $\hmatH_g$ as $\hmatH_g=\hmatU_{g}\matW_{g}$, then $\hmatH_g$ is also orthonormal. Additionally, $\hmatH_g^T\hmatH_{(i),l}=\matW_g^T\hmatU_g^T\hmatU_{(i),l}\matW_{(i),l}=0$. This completes the proof of equation~\eqref{subeqn:kktconstraint}. 

Next, we proceed with the proof of equations~\eqref{subeqn:kkthg} and \eqref{subeqn:kkthil}. By right multiplying both sides of \eqref{eqn:mmtugexpansion} with $\matW_g$ and replacing $\hmatU_g$ and $\hmatU_{(i),l}$ by $\hmatH_g$ and $\hmatH_{(i),l}$, we have
\begin{equation}
\frac{1}{N}\sum_{i=1}^N\hmatM_{(i)}\hmatM_{(i)}^T\hmatH_{g} = \hmatH_{g}\matLambda_{1}+\matP_{\hmatH_{(i),l}}\frac{1}{N}\sum_{i=1}^N\hmatM_{(i)}\hmatM_{(i)}^T\hmatH_{g}.  
\end{equation}
Similarly, by right multiplying both sides of \eqref{eqn:mmtuiexpansion} with $\matW_{(i),l}$, we can rewrite \eqref{eqn:mmtuiexpansion} as,
\begin{equation}
\hmatM_{(i)}\hmatM_{(i)}^T\hmatH_{(i),l} = \hmatH_{(i),l}\matLambda_{2,(i)}+\hmatH_{g}\hmatH_{g}^T\hmatM_{(i)}\hmatM_{(i)}^T\hmatH_{(i),l} . 
\end{equation}
We thus prove the equations~\eqref{subeqn:kkthg} and \eqref{subeqn:kkthil}, where $\matLambda_{3,(i)} = \hmatH_{g}^T\hmatM_{(i)}\hmatM_{(i)}^T\hmatH_{(i),l} $.$\hfill\blacksquare$

We note that the KKT conditions provide a set of conditions that must be satisfied for \textit{all} stationary points of \eqref{eqn:subproblemobj}. Our next key contribution is to use these conditions to characterize a few interesting properties satisfied by \textit{all} the optimal solutions. To this goal, we heavily rely on the spectral properties of $\matLambda_1$, $\matLambda_{2,(i)}$, and $\matLambda_{3,(i)}$. 

\revise{For simplicity, we introduce three additional notations, $\matLambda_{4,(i)} = -\matLambda_{3,(i)}$,  $\matLambda_{5,(i)}=-\matLambda_{3,(i)}^T/N$, and $\matLambda_6\in \mathbb{R}^{r_1\times r_1}$ defined as,
\begin{equation}
\label{eqn:deflambda6}
\matLambda_6 = \matLambda_1 - \frac{1}{N}\sum_{i=1}^N\matLambda_{3,(i)}\matLambda_{2,(i)}^{-1}\matLambda_{3,(i)}^T.
\end{equation}
$\matLambda_6$ is a symmetric matrix. It is worth noting that $\matLambda_6$ is well defined as the diagonal matrix $\matLambda_{2,(i)}$ is invertible throughput the proof. We also introduce short-hand notation $\Delta \matP_g$ to denote $\matP_{\hmatH_g}-\matP_{\matUst_g}$ and $\Delta \matP_{(i),l}$ to denote $\matP_{\hmatH_{(i),l}}-\matP_{\matUst_{(i),l}}$.}

Spectral properties of $\matLambda_1$, $\matLambda_{2,(i)}$, $\matLambda_{3,(i)}$, and $\matLambda_{6}$ are critical for developing the solutions to KKT conditions. We will establish these properties in the following lemmas.

Before diving into these properties, we investigate the deviance of the estimates features $\hmatH_g$ and $\hmatH_{(i),l}$ to ground truth features $\matUst_g$ and $\matUst_{(i),l}$. 

\revise{\begin{lemma}
\label{lm:deltapupperbound}
If $\max_i\norm{\matE_{(i)}}_{\infty}\le 4\gop\frac{\mu^2 r}{\bn}$, and $\matE_{(i)}$ is $\alpha$-sparse with
$\alpha\le \frac{1}{60\mu^4 r^2}\frac{\sm^4}{\gop^4}\left(1+\frac{4\gop^2}{\sqrt{\theta}\sms}\right)^{-2}$, we have,
 \begin{align}
\label{eqn:dpgfnorm}
\norm{\matP_{\matUst_g}-\matP_{\hmatH_g}}_F\le \sabnb\frac{5\gop}{\sqrt{\theta}\sms},
 \end{align}   
 and 
 \begin{align}
\label{eqn:dpifnorm}
\norm{\matP_{\matUst_{(i),l}}-\matP_{\hmatH_{(i),l}}}\le 3\sabnb\frac{\gop}{\sms}\left(1+4\frac{\gop^2}{\sqrt{\theta}\sms}\right).
 \end{align} 
\end{lemma}}
\begin{proof}
 
By \citet[Theorem 1]{personalizedpca}, we know that $\hmatH_g$ and $\hmatH_{(i),l}$ corresponding to the global optimal solutions to the problem~\eqref{eqn:subproblemobj} satisfy

\begin{equation}
\label{eqn:l2errorbound}
\norm{\matP_{\matUst_g}-\matP_{\hmatH_g}}_F^2+\frac{1}{N}\sum_{i=1}^N\norm{\matP_{\matUst_{(i),l}}-\matP_{\hmatH_{(i),l}}}_F^2\le \frac{4}{N}\sum_{i=1}^N\frac{\norm{\matF_{(i)}}_F^2}{\theta\smq}.
\end{equation}

Note that the norm of the error term $\norm{\matF_{(i)}}_F$ is bounded by:
\begin{equation}
\label{eqn:ffnormbound}
\begin{aligned}
&\norm{\matF_{(i)}}_F=\norm{\matLst_{(i)}\matE_{(i),t}^T+\matE_{(i),t}{\matLst_{(i)}}^T+\matE_{(i),t}\matE_{(i),t}^T}_F\\
&\le \norm{\matE_{(i),t}}_F\left(2\norm{\matLst_{(i)}}_{2}+\norm{\matE_{(i),t}}_2\right)\\
&\le \sqrt{\alpha}\bn\norm{\matE_{(i),t}}_{\infty}\left(2\gop+\abnb\right).
\end{aligned}
\end{equation}
Therefore, we know from \eqref{eqn:l2errorbound} that
\begin{align*}
&\norm{\matP_{\matUst_g}-\matP_{\hmatH_g}}_F\\
&\le 2\frac{\norm{\matF_{(i)}}_F}{\sqrt{\theta\smq}}\le \sqrt{\alpha}\bn\norm{\matE_{(i),t}}_{\infty}\left(2\gop+\abnb\right)\frac{2}{\sqrt{\theta\smq}}\\
&\le \sabnb\frac{5\gop}{\sqrt{\theta}\sms}, 
\end{align*}
where we used the condition $\abnb \le \gop/2$ for the last inequality.

From~\eqref{eqn:subproblemobj}, we can also deduce that that the column vectors of $\hmatH_{(i),l}$ span the top invariant subspace of $\left(\matI-\matP_{\hmatH_g}\right)\hmatM_{(i)}\hmatM_{(i)}^T \left(\matI-\matP_{\hmatH_g}\right)$. Since column vectors of $\matUst_{(i),l}$ span the top invariant subspace of $\left(\matI-\matP_{\matUst_g}\right)\matMst_{(i)}\matMst_{(i)}^T \left(\matI-\matP_{\matUst_g}\right)$, we know from Weyl's theorem~\citep{taoblog} and Davis-Khan theorem~\citep{cmuslides} that,

\begin{align*}
&\norm{\matP_{\matUst_{(i),l}}-\matP_{\hmatH_{(i),l}}}\le\\
&\frac{\norm{\left(\matI-\matP_{\hmatH_g}\right)\hmatM_{(i)}\hmatM_{(i)}^T \left(\matI-\matP_{\hmatH_g}\right)-\left(\matI-\matP_{\matUst_g}\right)\matMst_{(i)}\matMst_{(i)}^T \left(\matI-\matP_{\matUst_g}\right)}_F}{\sms - \norm{\left(\matI-\matP_{\hmatH_g}\right)\hmatM_{(i)}\hmatM_{(i)}^T \left(\matI-\matP_{\hmatH_g}\right)-\left(\matI-\matP_{\matUst_g}\right)\matMst_{(i)}\matMst_{(i)}^T \left(\matI-\matP_{\matUst_g}\right)}}. \\
\end{align*}

Since 
\begin{align*}
&\norm{\left(\matI-\matP_{\hmatH_g}\right)\hmatM_{(i)}\hmatM_{(i)}^T \left(\matI-\matP_{\hmatH_g}\right)-\left(\matI-\matP_{\matUst_g}\right)\matMst_{(i)}\matMst_{(i)}^T \left(\matI-\matP_{\matUst_g}\right)}_F\\
&=\norm{\left(\matI-\matP_{\hmatH_g}\right)\left(\hmatM_{(i)}\hmatM_{(i)}^T-\matMst_{(i)}\matMst_{(i)}^T\right) \left(\matI-\matP_{\hmatH_g}\right) }_F\\
&+\norm{\matMst_{(i)}\matMst_{(i)}^T\left(\matP_{\matUst_g}-\matP_{\hmatH_g}\right)}_F+\norm{\left(\matP_{\matUst_g}-\matP_{\hmatH_g}\right)\matMst_{(i)}\matMst_{(i)}^T}_F\\
&\le \norm{\matF_{(i)}}_F+2\gop^2\norm{ \matP_{\hmatH_g}-\matP_{\matUst_g}}_F\le\frac{5}{2}\sabnb\gop\left(1+4\frac{\gop^2}{\sqrt{\theta}\sms}\right)\\
&\le \frac{\sms}{6},
\end{align*}
we have,
\begin{align*}
&\norm{\matP_{\matUst_{(i),l}}-\matP_{\hmatH_{(i),l}}}\\
&\le\frac{1}{\sms-\frac{\sms}{6} }\frac{5}{2}\sabnb\gop\left(1+4\frac{\gop^2}{\sqrt{\theta}\sms}\right)\\
&\le 3\sabnb\frac{\gop}{\sms}\left(1+4\frac{\gop^2}{\sqrt{\theta}\sms}\right).
\end{align*}
This completes the proof.

\end{proof}

With Lemma~\ref{lm:deltapupperbound}, we first provide upper bound on the operator norm of $\matLambda_{3,(i)}$.
\begin{lemma}
\label{lm:eigenvalueupperbound}
For every $i\in [N]$, suppose that $\matU^\star_{(i),l}$'s are $\theta$-misaligned, $\max_i\norm{\matE_{(i)}}_{\infty}\le 4\gop\frac{\mu^2 r}{\bn}$, and $\matE_{(i)}$ is $\alpha$-sparse with $
\alpha \le \frac{1}{10\mu^2 r}$, we have,
\begin{align*}
    \norm{\matLambda_{3,(i)}}\le 2\gop.
\end{align*}    
\end{lemma}
\begin{proof}
    From the KKT condition~\eqref{eqn:kktderivations}, we know,
    \begin{align*}
    &\norm{\matLambda_{3,(i)}}= \norm{\hmatH_g^T\left(\matT_{(i)}+\matF_{(i)}\right)\hmatH_{(i),l}}\\
    &\le \norm{\matT_{(i)}} + \norm{\matF_{(i)}}\le 2\gop.
    \end{align*}
    This completes the proof.
\end{proof}

We then estimate lower bounds on the smallest eigenvalues of $\matLambda_1$, $\matLambda_{2,(i)}$, and $\matLambda_6$. These estimates rely on more refined matrix perturbation analysis.

\begin{lemma}
\label{lm:eigenvaluelowerbound}
For every $i\in [N]$, suppose that $\matU^\star_{(i),l}$'s are $\theta$-misaligned, $\max_i\norm{\matE_{(i)}}_{\infty}\le 4\gop\frac{\mu^2 r}{\bn}$, and $\matE_{(i)}$ is $\alpha$-sparse with
$$
\alpha \le \frac{1}{64}\frac{1}{\mu^4r^2}\left(\frac{\sm}{\gop}\right)^4\left(1+2\left(\frac{\gop}{\sm}\right)^2+\frac{8}{\sqrt{\theta}}\left(\frac{\gop}{\sm}\right)^4\right)^{-2}.
$$
The minimum eigenvalues of $\matLambda_{1}$ and $\matLambda_{2,(i)}$ are lower bounded by $\frac{3}{4}\sms$. 
\end{lemma}
\begin{proof}
This lemma is a result of Weyl's theorem \citep{taoblog} and the perturbation bound on the eigenspaces. 
From the first equation in \eqref{eqn:kktderivations}, we know,
\begin{equation*}
\begin{aligned}
\matP_{\hmatH_{(i),l}}\left(\matT_{(i)}+\matF_{(i)}\right)\matP_{\hmatH_{(i),l}}\hmatH_{(i),l}=\hmatH_{(i),l}\bm{\Lambda}_{2,(i)}.
\end{aligned}   
\end{equation*}
Therefore, $\hmatH_{(i),l}$'s columns are the eigenvectors of the symmetric matrix $\matP_{\hmatH_{(i),l}}\left(\matT_{(i)}+\matF_{(i)}\right)\matP_{\hmatH_{(i),l}}$, with eigenvalues corresponding to the diagonal entries of $\bm{\Lambda}_{2,(i)}$. 
According to the definition of $\matT_{(i)}$, we know that the eigenvalues of $\matP_{\matUst_{(i),l}}\matT_{(i)}\matP_{\matUst_{(i),l}}=\matHst_{(i),l}\matSigmast_{(i),l}^2\matHst_{(i),l}^T$ are lower bounded by $\sms$. Hence, as a result of Weyl's inequality, we have
$$
\begin{aligned}
&\lambda_{\min}\left(\matLambda_{2,(i)}\right)=\lambda_{\min}\left(\matP_{\hmatH_{(i),l}}\left(\matT_{(i)}+\matF_{(i)}\right)\matP_{\hmatH_{(i),l}}\right)\\
&\ge \sms-\norm{\matP_{\matUst_{(i),l}}\matT_{(i)}\matP_{\matUst_{(i),l}}-\matP_{\hmatH_{(i),l}}\left(\matT_{(i)}+\matF_{(i)}\right)\matP_{\hmatH_{(i),l}}}_2 .
\end{aligned}
$$
On the other hand, by triangle inequalities, we have
\begin{align*}
&\norm{\matP_{\matUst_{(i),l}}\matT_{(i)}\matP_{\matUst_{(i),l}}-\matP_{\hmatH_{(i),l}}\left(\matT_{(i)}+\matF_{(i)}\right)\matP_{\hmatH_{(i),l}}}_2 \\
&\le \norm{\matP_{\matUst_{(i),l}}\matT_{(i)}\matP_{\matUst_{(i),l}}-\matP_{\hmatH_{(i),l}}\matT_{(i)}\matP_{\hmatH_{(i),l}}}_2+\norm{\matF_{(i)}}_2\\
&\le \norm{\matP_{\matUst_{(i),l}}\matT_{(i)}\matP_{\matUst_{(i),l}}-\matP_{\hmatH_{(i),l}}\matT_{(i)}\matP_{\matUst_{(i),l}}}_2\\
&\ \ +\norm{\matP_{\hmatH_{(i),l}}\matT_{(i)}\matP_{\matUst_{(i),l}}-\matP_{\hmatH_{(i),l}}\matT_{(i)}\matP_{\hmatH_{(i),l}}}_2+\norm{\matF_{(i)}}_2\\
&\le 2\gop^2\norm{\matP_{\hmatH_{(i),l}}-\matP_{\matUst_{(i),l}}}_2 + \norm{\matF_{(i)}}_2\\
&\le \sabnb\left(2\gop+\sabnb\right)\\
&+\frac{5}{2}\sabnb\frac{\gop}{\sms}\left(1+4\frac{\gop^2}{\sqrt{\theta}\sms}\right) \\
&\le \frac{1}{4}\sms,
\end{align*}
where we used the fact $\norm{\matT_{(i)}}\le \gop^2$ in the third inequality, Lemma~\ref{lm:deltapupperbound} in the 4th inequality, and the assumed upper bound on $\alpha$ in the last inequality. We thus have $\lambda_{\min}\left(\matLambda_{2,(i)}\right)\ge \frac{3}{4}\sms$.

Similarly, we can solve $\matLambda_{3,(i)}$ from the first equation of \eqref{eqn:kktderivations} as $\matLambda_{3,(i)}=\hmatH_{g}^T\left(\matT_{(i)}+\matF_{(i)}\right)\hmatH_{(i),l}$. Plugging this into the second equation of \eqref{eqn:kktderivations}, we have
$$
\frac{1}{N}\sum_{i=1}^N\left(\matI-\matP_{\hmatH_{(i),l}}\right)\left(\matT_{(i)}+\matF_{(i)}\right)\left(\matI-\matP_{\hmatH_{(i),l}}\right)\hmatH_g=\hmatH_g\matLambda_1.
$$

Thus, the columns of $\hmatH_g$ are the eigenvectors of the matrix $\frac{1}{N}\sum_{i=1}^N\left(\matI-\matP_{\hmatH_{(i),l}}\right)\left(\matT_{(i)}+\matF_{(i)}\right)\left(\matI-\matP_{\hmatH_{(i),l}}\right)$, with eigenvalues corresponding to the diagonal entries of $\matLambda_1$. Again, since the minimum eigenvalue of $\frac{1}{N}\sum_{i=1}^N\left(\matI-\matP_{\matUst_{(i),l}}\right)\matT_{(i)}\left(\matI-\matP_{\matUst_{(i),l}}\right)$ is lower bounded by $\sms$, Weyl's inequality can be invoked to provide a lower bound on the minimum eigenvalue of $\matLambda_1$:
$$
\begin{aligned}
&\lambda_{\min}\left(\matLambda_{1}\right)=\lambda_{\min}\left(\frac{1}{N}\sum_{i=1}^N\left(\matI-\matP_{\hmatH_{(i),l}}\right)\left(\matT_{(i)}+\matF_{(i)}\right)\left(\matI-\matP_{\hmatH_{(i),l}}\right)\right)\\
&\ge \sms-\\
&\norm{\frac{\sum_{i=1}^N\left(\matI-\matP_{\matUst_{(i),l}}\right)\matT_{(i)}\left(\matI-\matP_{\matUst_{(i),l}}\right)-\left(\matI-\matP_{\hmatH_{(i),l}}\right)\left(\matT_{(i)}+\matF_{(i)}\right)\left(\matI-\matP_{\hmatH_{(i),l}}\right)}{N}  }_2.  \end{aligned}
$$
The operator norm on the right hand side can be bounded by triangle inequalities. For each term in the summation, we have
$$
\begin{aligned}
&\norm{\left(\matI-\matP_{\matUst_{(i),l}}\right)\matT_{(i)}\left(\matI-\matP_{\matUst_{(i),l}}\right)-\left(\matI-\matP_{\hmatH_{(i),l}}\right)\left(\matT_{(i)}+\matF_{(i)}\right)\left(\matI-\matP_{\hmatH_{(i),l}}\right)}_2\\
\le& \norm{\left(\matI-\matP_{\matUst_{(i),l}}\right)\matT_{(i)}\left(\matI-\matP_{\matUst_{(i),l}}\right)-\left(\matI-\matP_{\hmatH_{(i),l}}\right)\matT_{(i)}\left(\matI-\matP_{\matUst_{(i),l}}\right)}\\
&+\norm{\left(\matI-\matP_{\hmatH_{(i),l}}\right)\matT_{(i)}\left(\matI-\matP_{\matUst_{(i),l}}\right)-\left(\matI-\matP_{\hmatH_{(i),l}}\right)\matT_{(i)}\left(\matI-\matP_{\hmatH_{(i),l}}\right)}\\
&+\norm{\left(\matI-\matP_{\hmatH_{(i),l}}\right)\matF_{(i)}\left(\matI-\matP_{\hmatH_{(i),l}}\right)}\\
\le& 2\gop^2\norm{\matP_{\matUst_{(i),l}}-\matP_{\hmatH_{(i),l}}}+\norm{\matF_{(i)}}\\
\le& \frac{1}{4}\sms.
\end{aligned}
$$
where we used the assumed upper bound on $\alpha$.
This completes the proof.
\end{proof}

\revise{We also provide a lower bound on the minimum eigenvalue of the symmetric matrix $\matLambda_6$. Remember that $\matLambda_6$ is defined as $\matLambda_6 = \matLambda_1 - \frac{1}{N}\sum_{i=1}^N\matLambda_{3,(i)}\matLambda_{2,(i)}^{-1}\matLambda_{3,(i)}^T$.
\begin{lemma}
\label{lm:l6eigenvaluelowerbound}
For every $i\in [N]$, suppose that $\matU^\star_{(i),l}$'s are $\theta$-misaligned, $\max_i\norm{\matE_{(i)}}_{\infty}\le 4\gop\frac{\mu^2 r}{\bn}$, and $\matE_{(i)}$ is $\alpha$-sparse with $
\alpha \le \frac{1}{\left(\mu^2r640\right)^2}\left(\frac{\sm}{\gop}\right)^8\left(1+\frac{4\gop^2}{\sqrt{\theta}\sms}\right)^{-2}$, then, the minimum eigenvalue of $\matLambda_{6}$ is lower bounded by $\frac{3}{4}\sms$. 
\end{lemma}}
\begin{proof}
The proof is constructive. We use two steps. In the first step, we introduce a block matrix $\matLambda_{7}$ defined as,
\begin{equation}
   \matLambda_7 = \begin{pmatrix}
    \sqrt{N}\matLambda_1 & \matLambda_{3,(1)}&\cdots&\matLambda_{3,(N)}\\
    \matLambda_{3,(1)}^T&\sqrt{N}\matLambda_{2,(1)}&\cdots&0\\
    \vdots&\vdots&\ddots&\vdots\\
    \matLambda_{3,(N)}&0&\cdots&\sqrt{N}\matLambda_{2,(N)}\\
   \end{pmatrix},
\end{equation}
 and show that the minimum eigenvalue of the minimum eigenvalue of $\matLambda_7$ is lower bounded by $\sqrt{N}\frac{3}{4}\sms$. Then, in the second step, we prove that the minimum eigenvalue of $\matLambda_{6}$ is lower bounded by $\frac{1}{\sqrt{N}}$ multiplies the minimum eigenvalue of $\matLambda_{7}$, $\lambda_{min}(\matLambda_6)\ge \lambda_{min}(\frac{1}{\sqrt{N}}\matLambda_7)$.

 During this proof, we further introduce 
 \begin{align*}
     \left\{\begin{aligned}
    \matLambdast_{1,(i)} &= \matHst_g^T\matMst_{(i)} \matMst_{(i)}^T\matHst_g\\
    \matLambdast_{1} &=\frac{1}{N}\sum_{i=1}^N \matLambdast_{1,(i)}\\
     \matLambdast_{2,(i)} &= \matHst_{(i),l}^T\matMst_{(i)} \matMst_{(i)}^T\matHst_{(i),l}\\
     \matLambdast_{3,(i)} &= \matHst_{g}^T\matMst_{(i)} \matMst_{(i)}^T\matHst_{(i),l},\\\end{aligned}\right. 
 \end{align*}
for notational simplicity. From the SVD~\eqref{eqn:truemodelsvd} and the assumption on singular values of $\matLst_{(i)}$, we know: 
\begin{align}
\label{apeqn:svdimplication}
[\matHst_g,\matHst_{(i),l}]^T\matLst_{(i)}\matLst_{(i)}^T[\matHst_g,\matHst_{(i),l}] =  \begin{pmatrix}
    \matLambdast_{1,(i)} & \matLambdast_{3,(i)}\\
    \matLambdast_{3,(i)}^T&\matLambdast_{2,(i)}\\
   \end{pmatrix} \succ \sms\matI.
\end{align}

\noindent\underline{\textit{Step 1: Minimum eigenvalue of} $\matLambda_7$}: 
From definitions of $\matLambda_1$, $\matLambda_{2,(i)}$, and $\matLambda_{3,(i)}$ in~\eqref{eqn:kktderivations}, we know,
\begin{align*}
   \matLambda_7 &= \underbrace{\begin{pmatrix}
    \sqrt{N}\hmatH_g^T\matT_{(0)}\hmatH_g & \hmatH_g^T\matT_{(1)}\hmatH_{(1),l}&\cdots&\hmatH_g^T\matT_{(N)}\hmatH_{(N),l}\\
    \hmatH_{(1),l}^T\matT_{(1)}\hmatH_g&\sqrt{N}\hmatH_{(1),l}^T\matT_{(1)}\hmatH_{(1),l}&\cdots&0\\
    \vdots&\vdots&\ddots&\vdots\\
    \hmatH_{(N),l}^T\matT_{(N)}\hmatH_g&0&\cdots&\sqrt{N}\hmatH_{(N),l}^T\matT_{(N)}\hmatH_{(N),l}\\
   \end{pmatrix}}_{\matLambda_{7,2}}\\
   &\underbrace{+\begin{pmatrix}
    \sqrt{N}\hmatH_g^T\matF_{(0)}\hmatH_g & \hmatH_g^T\matF_{(1)}\hmatH_{(1),l}&\cdots&\hmatH_g^T\matF_{(N)}\hmatH_{(N),l}\\
    \hmatH_{(1),l}^T\matF_{(1)}\hmatH_g&\sqrt{N}\hmatH_{(1),l}^T\matF_{(1)}\hmatH_{(1),l}&\cdots&0\\
    \vdots&\vdots&\ddots&\vdots\\
    \hmatH_{(N),l}^T\matF_{(N)}\hmatH_g&0&\cdots&\sqrt{N}\hmatH_{(N),l}^T\matF_{(N)}\hmatH_{(N),l}\\
   \end{pmatrix}}_{\matLambda_{7,1}}.
\end{align*}

By Lemma~\ref{lm:blockmatrixopnormupperbound}, the operator norm of $\matLambda_{7,1}$ is upper bounded by,
\begin{align}
&\norm{\matLambda_{7,1}}\le \max_{i=0,1,\cdots,N} \{\sqrt{N}\norm{\matF_{(i)}}\}+\sqrt{2\sum_{i=1}^N\norm{\matF_{(i)}}^2}\notag\\
    &\le 2\sqrt{N}\abnb \gop\frac{5}{2}\le\frac{\sqrt{N}\sms}{16}.\notag\\
    \label{eqn:l71upper}
\end{align}
where we used the condition $\alpha\le \frac{\sms}{\gop^2}\frac{1}{320\mu^2r}$ in the last inequality.

We can further decompose $\matLambda_{7,2}$. . Then, we can derive $\hmatH_g^T\matT_{(0)}\hmatH_g=\hmatH_g^T\matHst_g\matHst_g^T\matT_{(0)}\matHst_g\matHst_g^T\hmatH_g+\hmatH_g^T\Delta\matP_{g}\matT_{(0)}\hmatH_g+\hmatH_g^T\matHst_g\matHst_g^T\matT_{(0)}\Delta\matP_{g}\hmatH_g$,  $\hmatH_{(i),l}^T\matT_{(i)}\hmatH_{(i),l}=\hmatH_{(i),l}^T\matHst_{(i),l}\matHst_{(i),l}^T\matT_{(i)}\matHst_{(i),l}\matHst_{(i),l}^T\hmatH_{(i),l}+\hmatH_{(i),l}^T\Delta\matP_{{(i),l}}\matT_{(i)}\hmatH_{(i),l}+\hmatH_{(i),l}^T\matHst_{(i),l}\matHst_{(i),l}^T\matT_{(i)}\Delta\matP_{(i),l}\hmatH_{(i),l}$, and $\hmatH_g^T\matT_{(i)}\hmatH_{(i),l}=\hmatH_g^T\matHst_g\matHst_g^T\matT_{(i)}\matHst_{(i),l}\matHst_{(i),l}^T\hmatH_{(i),l}+\hmatH_g^T\Delta\matP_{g}\matT_{(i)}\hmatH_{(i),l}+\hmatH_g^T\matHst_g\matHst_g^T\matT_{(i)}\Delta\matP_{(i),l}\hmatH_{(i),l}$.

Therefore, we can rewrite $\matLambda_{7,2}$ as,

\scalebox{0.75}{
\begin{minipage}{\textwidth}
\begin{align*}
  &\matLambda_{7,2} =  \matLambda_{7,3}+\\ 
  &\underbrace{\begin{pmatrix}
\sqrt{N}\hmatH_g^T\matHst_g\matLambdast_{1}\matHst_g^T\hmatH_g & \hmatH_g^T\matHst_g\matLambdast_{3,(1)}\matHst_{(1),l}^T\hmatH_{(1),l}&\cdots&\hmatH_g^T\matHst_g\matLambdast_{3,(N)}\matHst_{(N),l}^T\hmatH_{(N),l}\\
    \hmatH_{(1),l}^T\matHst_{(1),l}\matLambdast_{3,(1)}^T\matHst_g^T\hmatH_g&\sqrt{N}\hmatH_{(1),l}^T\matHst_{(1),l}\matLambdast_{2,(1)}\matHst_{(1),l}^T\hmatH_{(1),l}&\cdots&0\\
    \vdots&\vdots&\ddots&\vdots\\
    \hmatH_{(N),l}^T\matHst_{(N),l}\matLambdast_{3,(N)}^T\matHst_g^T\hmatH_g&0&\cdots&\sqrt{N}\hmatH_{(N),l}^T\matHst_{(N),l}\matLambdast_{2,(N)}\matHst_{(N),l}^T\hmatH_{(N),l}\\
\end{pmatrix}}_{\matLambda_{7,4}},\\
\end{align*}
\end{minipage}}
where $\matLambda_{7,3}$ consists of residual terms that contain $\Delta \matP_g$ or $\Delta \matP_{(i),l}$. 

We use Lemma~\ref{lm:blockmatrixopnormupperbound} to estimate an upper bound for the operator norm of $\matLambda_{7,3}$. The maximum operator norm of the diagonal block of $\matLambda_{7,3}$ is
\begin{align*}
    &\max\{\sqrt{N}\norm{\hmatH_g^T\Delta\matP_{g}\matT_{(0)}\hmatH_g+\hmatH_g^T\matHst_g\matHst_g^T\matT_{(0)}\Delta\matP_{g}\hmatH_g}, \\
    &\quad\max_{i}\{\sqrt{N}\norm{\hmatH_{(i),l}^T\Delta\matP_{{(i),l}}\matT_{(i)}\hmatH_{(i),l}+\hmatH_{(i),l}^T\matHst_{(i),l}\matHst_{(i),l}^T\matT_{(i)}\Delta\matP_{(i),l}\hmatH_{(i),l}} \}\}\\
    &\le \sqrt{N} 2 \gop^2\frac{5}{2}\sabnb\frac{\gop}{\sms}\left(1+4\frac{\gop^2}{\sqrt{\theta}\sms}\right).
\end{align*}

The summation of the operator norm of the off-diagonal blocks of $\matLambda_{7,3}$ is
\begin{align*}
  &\sqrt{2\sum_{i=1}^N \norm{\hmatH_g^T\Delta\matP_{g}\matT_{(i)}\hmatH_{(i),l}+\hmatH_g^T\matHst_g\matHst_g^T\matT_{(i)}\Delta\matP_{(i),l}\hmatH_{(i),l}}^2}\\
  &\le \sqrt{2N}2 \gop^2\frac{5}{2}\sabnb\frac{\gop}{\sms}\left(1+4\frac{\gop^2}{\sqrt{\theta}\sms}\right). 
\end{align*}

As a result, Lemma~\ref{lm:blockmatrixopnormupperbound} implies that,
\begin{align}
    &\norm{\matLambda_{7,3}}\le \sqrt{N}10\gop^2\sabnb\frac{\gop}{\sms}\left(1+4\frac{\gop^2}{\sqrt{\theta}\sms}\right)\notag\\
    &\le \frac{\sqrt{N}\sms}{16},\notag\\
    \label{eqn:l73upper}
\end{align}
where we applied the condition $\alpha\le \frac{1}{\left(\mu^2r640\right)^2}\left(\frac{\sm}{\gop}\right)^8\left(1+\frac{4\gop^2}{\sqrt{\theta}\sms}\right)^{-2}$ in the last inequality.

We proceed to estimate the eigenvalue lower bound for $\matLambda_{7,4}$. We first factorize $\matLambda_{7,4}$ as,
\begin{align*}
&\matLambda_{7,4} = \text{Diag}\left(\hmatH_g^T\matHst_g,\hmatH_{(1),l}^T\matHst_{(1),l},\cdots,\hmatH_{(N),l}^T\matHst_{(N),l} \right)\\
&\times \begin{pmatrix}
    \sqrt{N}\matLambdast_1 & \matLambdast_{3,(1)}&\cdots&\matLambdast_{3,(N)}\\
    \matLambdast_{3,(1)}^T&\sqrt{N}\matLambdast_{2,(1)}&\cdots&0\\
    \vdots&\vdots&\ddots&\vdots\\
    \matLambdast_{3,(N)}^T&0&\cdots&\sqrt{N}\matLambdast_{2,(N)}\\
   \end{pmatrix}\\
   &\times \text{Diag}\left(\hmatHst_g^T\hmatH_g,\matHst_{(1),l}^T\hmatH_{(1),l},\cdots,\matHst_{(N),l}^T\hmatH_{(N),l} \right).
\end{align*}

Lemma~\ref{lm:schurcomplement} and~\eqref{apeqn:svdimplication}
indicate that $\matLambdast_{1,(i)}-\sms\matI \succ 0$, $\matLambdast_{2,(i)}-\sms\matI \succ 0$, and \begin{equation}
\label{eqn:schurderivation1}
   \left(\matLambdast_{1,(i)}-\sms\matI\right)-\matLambdast_{3,(i)}^T\left(\matLambdast_{2,(i)}-\sms\matI\right)^{-1}\matLambdast_{3,(i)}\succeq 0. 
\end{equation} 

Summing both sides of~\eqref{eqn:schurderivation1} for $i=1$ to $N$, we know,
$
N\matLambdast_1-N\sms\matI -\sum_{i=1}^N\matLambdast_{3,(i)}^T\left(\matLambdast_{2,(i)}-\sms\matI\right)^{-1}\matLambdast_{3,(i)}\succeq 0
$, which is equivalent to $\sqrt{N}\matLambdast_1-\sqrt{N}\sms\matI -\sum_{i=1}^N\matLambdast_{3,(i)}^T\left(\sqrt{N}\matLambdast_{2,(i)}-\sqrt{N}\sms\matI\right)^{-1}\matLambdast_{3,(i)}\succeq 0$. Again, Lemma~\ref{lm:schurcomplement} indicates,
\begin{align*}
    \begin{pmatrix}
    \sqrt{N}\matLambdast_1 & \matLambdast_{3,(1)}&\cdots&\matLambdast_{3,(N)}\\
    \matLambdast_{3,(1)}^T&\sqrt{N}\matLambdast_{2,(1)}&\cdots&0\\
    \vdots&\vdots&\ddots&\vdots\\
    \matLambdast_{3,(N)}^T&0&\cdots&\sqrt{N}\matLambdast_{2,(N)}\\
   \end{pmatrix}\succeq\sqrt{N}\sms\matI.
\end{align*}

On the other hand, we know form Lemma~\ref{lm:blockmatrixopnormupperbound} that,
\begin{align*}
    &\norm{\matI-\text{Diag}\left(\hmatH_g^T\matHst_g\hmatHst_g^T\hmatH_g,\hmatH_{(1),l}^T\matHst_{(1),l}\matHst_{(1),l}^T\hmatH_{(1),l},\cdots,\hmatH_{(N),l}^T\matHst_{(N),l}\matHst_{(N),l}^T\hmatH_{(N),l} \right)}\\
    &\le \max\{\norm{\Delta\matP_{g}},\max_j \norm{\Delta\matP_{(j),l)}}\}\\
    &\le \frac{\abnb(\abnb+2\gop)}{\sms}\left(1+\frac{8\gop^2}{\sqrt{\theta}\sms}\right)\le \frac{1}{8},
\end{align*}
where we applied the condition $\alpha\le \frac{\sms}{80\mu^2 r\gop^2}\left(1+\frac{8\gop^2}{\sqrt{\theta}\sms}\right)^{-1}$ in the last inequality.

Hence, Lemma~\ref{lm:lminatba} indicates,
\begin{equation*}
    \lambda_{\min}\left(\matLambda_{7,4}\right)\ge \sqrt{N}\sms \frac{7}{8}.
\end{equation*}

By Wely's theorem, we know that
\begin{align}
&\lambda_{\min}\left(\matLambda_{7}\right)\ge \sqrt{N}\sms \frac{7}{8} -\norm{\matLambda_{7,1}+\matLambda_{7,3}}\notag\\
&\ge \sqrt{N}\sms  \frac{3}{4},\label{eqn:lambda7lmin}
\end{align}
where we applied the inequality~\eqref{eqn:l71upper} and~\eqref{eqn:l73upper} in the last inequality.

\noindent\underline{\textit{Step 2: Minimum eigenvalue of} $\matLambda_6$:} The inequality~\eqref{eqn:lambda7lmin} is equivalent to $\matLambda_7-\sqrt{N}\sms  \frac{3}{4}\matI\succ 0$. Thus, Lemma~\ref{lm:schurcomplement} implies,
\begin{align*}
    \sqrt{N}\matLambda_1-\sqrt{N}\sms \frac{3}{4}\matI -\sum_{i=1}^N \matLambda_{3,(i)}^T\left(\sqrt{N}\matLambda_{2,(i)}-\sqrt{N}\sms \frac{3}{4}\matI\right)^{-1}\matLambda_{2,(i)}\succeq 0.
\end{align*}

Lemma~\ref{lm:eigenvaluelowerbound} already shows that $\matLambda_{2,(i)}\succ \sms\frac{3}{4}\matI$. As a result,
\begin{align*}
   &\matLambda_{3,(i)}^T\left(\sqrt{N}\matLambda_{2,(i)}-\sqrt{N}\sms \frac{3}{4}\matI\right)^{-1}\matLambda_{2,(i)} \\
   &=\frac{1}{\sqrt{N}}\sum_{p=0}^{\infty}\matLambda_{3,(i)}^T\left(\matLambda_{2,(i)}^{-1}+\left(\sum_{p=0}^{\infty}\sms \frac{3}{4}\matLambda_{2,(i)}^{-1}\right)^{p}\right)\matLambda_{2,(i)}\\
   &\succeq \frac{1}{\sqrt{N}}\matLambda_{3,(i)}^T\matLambda_{2,(i)}^{-1}\matLambda_{3,(i)}.
\end{align*}

By rearranging terms, we have,
\begin{align*}
    \sqrt{N}\matLambda_1 -\frac{1}{\sqrt{N}}\sum_{i=1}^N \matLambda_{3,(i)}^T\matLambda_{2,(i)}^{-1}\matLambda_{2,(i)}\succeq \sqrt{N}\sms \frac{3}{4}\matI.
\end{align*}
This completes our proof.
\end{proof}

With an understanding of spectral properties of $\matLambda_1$, $\matLambda_{2, (i)}$, and  $\matLambda_6$, we are now ready to characterize the solutions to the KKT conditions and provide a proof for Lemma~\ref{lm:informalhpertuabation}. To this goal, we first write the solutions to \eqref{eqn:kktderivations} into Taylor-like series.

\revise{\begin{lemma}
\label{lm:explicitsolution}
For every $i\in [N]$, suppose that $\matU^\star_{(i),l}$'s are $\theta$-misaligned, $\max_i\norm{\matE_{(i)}}_{\infty}\le 4\gop\frac{\mu^2 r}{\bn}$, and $\matE_{(i)}$ is $\alpha$-sparse with $
 \alpha \le \frac{1}{40\mu^2 r}\brtkappa^{-3}$.
The solutions to \eqref{eqn:kktderivations} satisfy the following,
\begin{equation}
\label{eqn:ugcondition}
\begin{aligned}
&\hmatH_g=\hmatH_{g,0}+\hmatH_{g,1}\\
&+\sum_{k=0}^\infty\sum_{p_0+p_1\ge 1}^{\infty}\cdots\sum_{p_{2k}+p_{2k+1}\ge 1}^{\infty}\sum_{i_1,i_3,\cdots,i_{2k+1}=1}^N\left[\prod_{l=0}^{k}\left(\matF_{(0)}^{p_{2l}}\matF_{(i_{2l+1})}^{p_{2l+1}}\right)\right]\hmatH_{g,0}\\
&\times \prod_{l=k}^0\left(\matLambda_{4,(i_{2l+1})}\matLambda_{4,(i_{2l+1})}\matLambda_{2,(i_{2l+1})}^{-p_{2l+1}-1}\matLambda_{5,(i_{2l+1})}\matLambda_{1}^{-p_{2l}}\matLambda_6^{-1}\right)\\
&+\sum_{k=0}^\infty\sum_{p_0+p_1\ge 1}^{\infty}\cdots\sum_{p_{2k}+p_{2k+1}\ge 1}^{\infty}\sum_{i_1,i_3,\cdots,i_{2k+1}=1}^N\left[\prod_{l=0}^{k}\left(\matF_{(0)}^{p_{2l}}\matF_{(i_{2l+1})}^{p_{2l+1}}\right)\right]\hmatH_{g,1}\\
&\times \prod_{l=k}^0\left(\matLambda_{4,(i_{2l+1})}\matLambda_{4,(i_{2l+1})}\matLambda_{2,(i_{2l+1})}^{-p_{2l+1}-1}\matLambda_{5,(i_{2l+1})}\matLambda_{1}^{-p_{2l}}\matLambda_6^{-1}\right),\\
\end{aligned}
\end{equation}
and
\begin{align}
\hmatH_{(j),l}&=\hmatH_{(i),l,0}+\sum_{p=1}^{\infty}\matF_{(j)}^p\matT_{(j)}\hmatH_{(j),l}\matLambda_{2,(j)}^{-p-1}+\hmatH_{g,1}\matLambda_{4,(j)}\matLambda_{2,(j)}^{-1}\notag\\
&+\sum_{k=0}^\infty\sum_{p_0+p_1\ge 1}^{\infty}\cdots\sum_{p_{2k}+p_{2k+1}\ge 1}^{\infty}\sum_{i_1,i_3,\cdots,i_{2k+1}=1}^N\left[\prod_{l=0}^{k}\left(\matF_{(0)}^{p_{2l}}\matF_{(i_{2l+1})}^{p_{2l+1}}\right)\right]\hmatH_{g,0}\notag\\
&\times \prod_{l=k}^0\left(\matLambda_{4,(i_{2l+1})}\matLambda_{4,(i_{2l+1})}\matLambda_{2,(i_{2l+1})}^{-p_{2l+1}-1}\matLambda_{5,(i_{2l+1})}\matLambda_{1}^{-p_{2l}}\matLambda_6^{-1}\right)\matLambda_{4,(j)}\matLambda_{2,(j)}^{-1}\notag\\
&+\sum_{k=0}^\infty\sum_{p_0+p_1\ge 1}^{\infty}\cdots\sum_{p_{2k}+p_{2k+1}\ge 1}^{\infty}\sum_{i_1,i_3,\cdots,i_{2k+1}=1}^N\left[\prod_{l=0}^{k}\left(\matF_{(0)}^{p_{2l}}\matF_{(i_{2l+1})}^{p_{2l+1}}\right)\right]\hmatH_{g,1}\notag\\
&\times \prod_{l=k}^0\left(\matLambda_{4,(i_{2l+1})}\matLambda_{4,(i_{2l+1})}\matLambda_{2,(i_{2l+1})}^{-p_{2l+1}-1}\matLambda_{5,(i_{2l+1})}\matLambda_{1}^{-p_{2l}}\matLambda_6^{-1}\right)\matLambda_{4,(j)}\matLambda_{2,(j)}^{-1}\notag\\
&+\sum_{p=1}^{\infty}\matF_{(j)}^p\hmatH_g\matLambda_{4,(j)}\hmatH_g\matLambda_{4,(j)}\matLambda_{2,(j)}^{-1},\notag\\
\label{eqn:ulcondition}
\end{align}
where $\hmatH_{g,0}$ is defined as
\begin{align}
    \hmatH_{g,0} = \matT_{(0)}\hmatH_g\matLambda_6^{-1} +\sum_{i=1}^N\matT_{(i)}\hmatH_{(i),l}\matLambda_{2,(i)}^{-1}\matLambda_{5,(i)}\matLambda_{6}^{-1},
\end{align}
$\hmatH_{g,1}$ is defined as
\begin{align}
    \hmatH_{g,1} = \sum_{p_0+p_1\ge 1}^{\infty}\sum_{i_1=1}^N\matF_{(0)}^{p_0}\matF_{(i_1)}^{p_1}\matT_{(i_1)}\hmatH_{(i_1),l}\matLambda_{2,(i_1)}^{-p_1-1}\matLambda_{5,(i_1)}\matLambda_{1}^{-p_0}\matLambda_6^{-1},
\end{align}
and $\hmatH_{(i),l,0}$ is defined as
\begin{align}
    \hmatH_{(i),l,0} = \matT_{(j)}\hmatH_{(j),1}\matLambda_2^{-1}+\hmatH_{g,0}\matLambda_{4,(j)}\matLambda_{2,(j)}^{-1}
\end{align}. 
\end{lemma}}

\begin{proof}
Notice that as we defined $\matLambda_{4,(i)}=-\matLambda_{3,(i)}$ and $\matLambda_{5,(i)}=-\matLambda_{3,(i)}^T/N$, the KKT condition in \eqref{eqn:kktderivations} can be written as the following Sylvester equations
\begin{equation}
\label{eqn:simplifiedkkt}
\left\{
\begin{aligned}
&\matF_{(i)}\hmatH_{(i),l}-\hmatH_{(i),l}\matLambda_{2,(i)} = -\left(\matT_{(i)} \hmatH_{(i),l}+\hmatH_g\matLambda_{4,(i)}\right)\\
&\matF_{(0)}\hmatH_{g}-\hmatH_{g}\matLambda_{1} = -\left(\matT_{(0)} \hmatH_{g}+\sum_{i=1}^N\hmatH_{(i),l}\matLambda_{5,(i)}\right).
\end{aligned}\right.
\end{equation}
 Note that $\sigma_{min}(\matLambda_{2,(i)})>\norm{\matF_{(i)}}$ and $\sigma_{min}(\matLambda_{1,(i)})>\norm{\matF_{(0)}}$. Therefore, according to~\citet[Theorem VII.2.2]{matrixanalysis},  the solution to \eqref{eqn:simplifiedkkt} satisfies the following equation
\begin{equation}
\label{eqn:sylvester}
\left\{
\begin{aligned}
&\hmatH_{g}=\sum_{p=0}^{\infty}\matF_{(0)}^p\matT_{(0)}\hmatH_{g}\matLambda_{1}^{-p-1}+\sum_{p=0}^{\infty}\sum_{i=1}^N\matF_{(0)}^p\hmatH_{(i),l}\matLambda_{5,(i)}\matLambda_{1}^{-p-1}\\
&\hmatH_{(i),l}=\sum_{p=0}^{\infty}\matF_{(i)}^p\matT_{(i)}\hmatH_{(i),l}\matLambda_{2,(i)}^{-p-1}+\sum_{p=0}^{\infty}\matF_{(i)}^p\hmatH_{g}\matLambda_{4,(i)}\matLambda_{2,(i)}^{-p-1}\\
\end{aligned}\right. .
\end{equation}
We can substitute $\hmatH_{(i),l}$ in the right hand side of the first equation of \eqref{eqn:sylvester} by the second equation in \eqref{eqn:sylvester}
\begin{align}
&\hmatH_{g}=\sum_{p=0}^{\infty}\matF_{(0)}^p\matT_{(0)}\hmatH_{g}\matLambda_{1}^{-p-1}+\sum_{p_0=0}^{\infty}\sum_{p_1=0}^{\infty}\sum_{i_1=1}^N\matF_{(0)}^{p_0}\matF_{(i_1)}^{p_1}\matT_{(i_1)}\hmatH_{(i_1),l}\matLambda_{2,(i_1)}^{-p_1-1}\matLambda_{5,(i_1)}\matLambda_{1}^{-p_0-1}\notag\\
&+\sum_{p_0=0}^{\infty}\sum_{p_1=0}^{\infty}\sum_{i_1=1}^N\matF_{(0)}^{p_0}\matF_{(i_1)}^{p_1}\hmatH_{g}\matLambda_{4,(i_1)}\matLambda_{2,(i_1)}^{-p_1-1}\matLambda_{5,(i_1)}\matLambda_{1}^{-p_0-1}\notag\\
&= \matT_{(0)}\hmatH_g\matLambda_1^{-1} +\sum_{i=1}^N\matT_{(i)}\hmatH_{(i),l}\matLambda_{2,(i)}^{-1}\matLambda_{5,(i)}\matLambda_{1}^{-1}+\sum_{i=1}^N\hmatH_{g}\matLambda_{4,(i)}\matLambda_{2,(i)}^{-1}\matLambda_{5,(i)}\matLambda_{1}^{-1} \notag\\
&+\sum_{p_0+p_1\ge 1}^{\infty}\sum_{i_1=1}^N\matF_{(0)}^{p_0}\matF_{(i_1)}^{p_1}\matT_{(i_1)}\hmatH_{(i_1),l}\matLambda_{2,(i_1)}^{-p_1-1}\matLambda_{5,(i_1)}\matLambda_{1}^{-p_0-1}\notag\\
&+\sum_{p_0+p_1\ge 1}^{\infty}\sum_{i_1=1}^N\matF_{(0)}^{p_0}\matF_{(i_1)}^{p_1}\hmatH_{g}\matLambda_{4,(i_1)}\matLambda_{2,(i_1)}^{-p_1-1}\matLambda_{5,(i_1)}\matLambda_{1}^{-p_0-1}.\notag\\
\label{eqn:ugexpand1}
\end{align}

We then move the third term on the right hand side of~\eqref{eqn:ugexpand1} to the left hand side, multiply both sides by $\matLambda_1\matLambda_{6}^{-1}$ on the right, and recall the definition of $\matLambda_{6,(i)}$ in~\eqref{eqn:deflambda6}, we have,

\begin{align}
\hmatH_{g}&=\underbrace{\matT_{(0)}\hmatH_g\matLambda_6^{-1} +\sum_{i=1}^N\matT_{(i)}\hmatH_{(i),l}\matLambda_{2,(i)}^{-1}\matLambda_{5,(i)}\matLambda_{6}^{-1}}_{\hmatH_{g,0}}\notag\\
&+\underbrace{\sum_{p_0+p_1\ge 1}^{\infty}\sum_{i_1=1}^N\matF_{(0)}^{p_0}\matF_{(i_1)}^{p_1}\matT_{(i_1)}\hmatH_{(i_1),l}\matLambda_{2,(i_1)}^{-p_1-1}\matLambda_{5,(i_1)}\matLambda_{1}^{-p_0}\matLambda_6^{-1}}_{\hmatH_{g,1}}\notag\\
&+\underbrace{\sum_{p_0+p_1\ge 1}^{\infty}\sum_{i_1=1}^N\matF_{(0)}^{p_0}\matF_{(i_1)}^{p_1}\hmatH_{g}\matLambda_{4,(i_1)}\matLambda_{2,(i_1)}^{-p_1-1}\matLambda_{5,(i_1)}\matLambda_{1}^{-p_0}\matLambda_6^{-1}}_{\text{residual term}}.\notag\\
\label{eqn:ugselfexplain}
\end{align}
On the right hand side of~\eqref{eqn:ugselfexplain}, one can see that the $\hmatH_{g,0}$ and $\hmatH_{g,1}$ are products of sparse matrices, incoherent matrices, and remaining terms. Therefore we can use Lemma \ref{lm:efuinfnorm} to calculate an upper bound on their maximum row norm. However, the residual term does not have such specific structure as we do not know whether $\hmatH_g$ is incoherent. As a result we cannot provide precise estimate on its maximum row norm directly. To circumvent the issue, notice that \eqref{eqn:ugselfexplain} has a recursive form. Therefore, the residual term can be replaced by
\begin{equation}
\label{eqn:replacementrule}
\begin{aligned}
&\hmatH_{g}\to \hmatH_{g,0}+\hmatH_{g,1}+\sum_{p_0+p_1\ge 1}^{\infty}\sum_{i_1=1}^N\matF_{(0)}^{p_0}\matF_{(i_1)}^{p_1}\hmatH_{g}\matLambda_{4,(i_1)}\matLambda_{2,(i_1)}^{-p_1-1}\matLambda_{5,(i_1)}\matLambda_{1}^{-p_0}\matLambda_6^{-1}.\\
\end{aligned}
\end{equation}

The result will have $5$ terms, the first $4$ of which have the structure specified in Lemma \ref{lm:efuinfnorm}. The $5$-th term does not as it contains $\hmatH_g$. We can apply the replacement rule \eqref{eqn:replacementrule} again for the $5$-th term, generating $7$ terms. After applying the replacement rule $\omega$ times, where $\omega$ is an integer, the results become,
\begin{align}
\label{eqn:ugrawform}
&\hmatH_{g}=\hmatH_{g,0}+\hmatH_{g,1}\notag\\
&+\sum_{k=0}^{\omega}\sum_{p_0+p_1\ge 1}^{\infty}\sum_{p_2+p_3\ge 1}^{\infty}\cdots\sum_{p_{2k}+p_{2k+1}\ge 1}^{\infty}\sum_{i_1=1}^{N}\sum_{i_3=1}^{N}\cdots\sum_{i_{2k+1}=1}^{N}\matF_{(0)}^{p_0}\matF_{(i_1)}^{p_1}\matF_{(0)}^{p_2}\matF_{(i_3)}^{p_3}\cdots \matF_{(0)}^{p_{2k}}\matF_{(i_{2k+1})}^{p_{2k+1}}\notag\\
&\hmatH_{g,0}\matLambda_{4,(i_{2k+1})}\matLambda_{2,(i_{2k+1})}^{-p_{2k-1}-1}\matLambda_{5,(i_{2k+1})}\matLambda_{1}^{-p_{2k-2}}\matLambda_6^{-1}\cdots \matLambda_{4,(i_{1})}\matLambda_{2,(i_{1})}^{-p_{1}-1}\matLambda_{5,(i_{1})}\matLambda_{1}^{-p_{0}}\matLambda_6^{-1}\notag\\
&+\sum_{k=0}^{\omega}\sum_{p_0+p_1\ge 1}^{\infty}\sum_{p_2+p_3\ge 1}^{\infty}\cdots\sum_{p_{2k}+p_{2k+1}\ge 1}^{\infty}\sum_{i_1=1}^{N}\sum_{i_3=1}^{N}\cdots\sum_{i_{2k+1}=1}^{N}\matF_{(0)}^{p_0}\matF_{(i_1)}^{p_1}\matF_{(0)}^{p_2}\matF_{(i_3)}^{p_3}\cdots \matF_{(0)}^{p_{2k}}\matF_{(i_{2k+1})}^{p_{2k+1}}\notag\\
&\hmatH_{g,1}\matLambda_{4,(i_{2k+1})}\matLambda_{2,(i_{2k+1})}^{-p_{2k-1}-1}\matLambda_{5,(i_{2k+1})}\matLambda_{1}^{-p_{2k-2}}\matLambda_6^{-1}\cdots \matLambda_{4,(i_{1})}\matLambda_{2,(i_{1})}^{-p_{1}-1}\matLambda_{5,(i_{1})}\matLambda_{1}^{-p_{0}}\matLambda_6^{-1}\notag\\
&+\sum_{p_0+p_1\ge 1}^{\infty}\sum_{p_2+p_3\ge 1}^{\infty}\cdots\sum_{p_{2\omega+2}+p_{2\omega+3}\ge 1}^{\infty}\sum_{i_1=1}^{N}\sum_{i_3=1}^{N}\cdots\sum_{i_{2\omega+3}=1}^{N}\matF_{(0)}^{p_0}\matF_{(i_1)}^{p_1}\matF_{(0)}^{p_2}\matF_{(i_3)}^{p_3}\cdots \matF_{(0)}^{p_{2\omega+2}}\matF_{(i_{2\omega+3})}^{p_{2\omega+3}}\notag\\
&\hmatH_{g}\matLambda_{4,(i_{2\omega+3})}\matLambda_{2,(i_{2\omega+3})}^{-p_{2\omega+3}-1}\matLambda_{5,(i_{2\omega+3})}\matLambda_{1}^{-p_{2\omega+2}}\matLambda_6^{-1}\cdots \matLambda_{4,(i_{1})}\matLambda_{2,(i_{1})}^{-p_{1}-1}\matLambda_{5,(i_{1})}\matLambda_{1}^{-p_0}\matLambda_6^{-1},\notag\\
\end{align}
which holds for any integer $\omega\ge 0$.

Recall that our goal is to write $\hmatH_g$ in a form with which we can easily determine its maximum row norm. By observing \eqref{eqn:ugrawform}, we know Lemma \ref{lm:efuinfnorm} can be applied to estimate the maximum row norm of all but the last terms. The last summation term still cannot be handled by Lemma \ref{lm:efuinfnorm} directly. To resolve the issue, we take an alternative route to use $\omega$ to control the last summation term.

We claim that under the provided upper bound for $\alpha$, the last term will approach zero in the limit $\omega\to \infty$. To see this, note that Lemma \ref{lm:eigenvaluelowerbound} and Lemma~\ref{lm:l6eigenvaluelowerbound} show that $\sigma_{\min}(\matLambda_{1})$,  $\sigma_{\min}(\matLambda_{2,(i)})$, and $\sigma_{\min}(\matLambda_{6})$ are lower bounded by $\frac{3}{4}\sms$. Since $\norm{\matF_{(i)}}\le \abnb(2\gop+\abnb)\le \abnb \frac{5}{2}\gop$ for each $i$, we have,

\begin{align*}
&\sum_{p_0+p_1\ge 1}^{\infty}\cdots\sum_{p_{2\omega+2}+p_{2\omega+3}\ge 1}^{\infty}\sum_{i_1=1}^{N}\sum_{i_3=1}^{N}\cdots\sum_{i_{2\omega+3}=1}^{N}\norm{\matF_{(0)}^{p_0}\matF_{(i_1)}^{p_1}\matF_{(0)}^{p_2}\matF_{(i_3)}^{p_3}\cdots \matF_{(i_{2\omega+3})}^{p_{2\omega+3}}}_F\\
&\norm{\hmatH_{g}\matLambda_{4,(i_{2\omega+3})}\matLambda_{2,(i_{2\omega+3})}^{-p_{2\omega+3}-1}\matLambda_{5,(i_{2\omega+3})}\matLambda_{1}^{-p_{2\omega+2}}\matLambda_6^{-1}\cdots \matLambda_{4,(i_{1})}\matLambda_{2,(i_{1})}^{-p_{1}-1}\matLambda_{5,(i_{1})}\matLambda_{1}^{-p_0}\matLambda_6^{-1}}\\
&\le \sum_{p_0+p_1\ge 1}^{\infty}\cdots\sum_{p_{2\omega+2}+p_{2\omega+3}\ge 1}^{\infty}\left(\frac{\abnb\left(\frac{5}{2}\gop\right)}{\frac{3}{4}\sms}\right)^{p_0+\cdots+p_{2\omega+3}}\\
&\times \left(\frac{2\gop^2}{\frac{3}{4}\sms}\right)^{2\left(\omega+2\right)}\\
&\le \left(4 \frac{\abnb\left(\frac{5}{2}\gop\right)}{\frac{3}{4}\sms}\right)^{2\left(\omega+2\right)}\left(\frac{2\gop^2}{\frac{3}{4}\sms}\right)^{2\left(\omega+2\right)}\\
&\le\left(\fracud\right)^{2(\omega+2)},
\end{align*}
where we used Lemma~\ref{lem_series} in the first inequality and the condition that $\alpha \le \frac{1}{40\mu^2 r}\brtkappa^{-3}$ in the last inequality.

Therefore, we can take the limit $\omega\to\infty$ in \eqref{eqn:ugrawform} and rewrite it as a series. The series is absolutely convergent when $\alpha$ is small. Finally, we prove \eqref{eqn:ugcondition}. Though \eqref{eqn:ugcondition} is an infinite series, each term in the series is the product of sparse matrices and an incoherent matrix. Such structure will be useful later when we use Lemma \ref{lm:efuinfnorm} to calculate the maximum row norm of $\hmatH_g$. 

Now we proceed to derive an expansion for $\hmatH_{(i),l}$. We can replace $\hmatH_g$ on the right hand side of the second equation of \eqref{eqn:sylvester} with \eqref{eqn:ugrawform} to derive,
\begin{align}
&\hmatH_{(i),l}=\matT_{(i)}\hmatH_{(i),l}\matLambda_{2,(i)}^{-1}+\hmatH_{g}\matLambda_{4,(i)}\matLambda_{2,(i)}^{-1}\notag\\
&+\sum_{p=1}^{\infty}\matF_{(i)}^p\matT_{(i)}\hmatH_{(i),l}\matLambda_{2,(i)}^{-p-1}+\sum_{p=1}^{\infty}\matF_{(i)}^p\hmatH_{g}\matLambda_{4,(i)}\matLambda_{2,(i)}^{-p-1}\notag\\
&+\sum_{k=0}^\infty\sum_{p_0+p_1\ge 1}^{\infty}\cdots\sum_{p_{2k}+p_{2k+1}\ge 1}^{\infty}\sum_{i_1,i_3,\cdots,i_{2k+1}=1}^N\left[\prod_{l=0}^{k}\left(\matF_{(0)}^{p_{2l}}\matF_{(i_{2l+1})}^{p_{2l+1}}\right)\right]\hmatH_{g,0}\notag\\
&\times \prod_{l=k}^0\left(\matLambda_{4,(i_{2l+1})}\matLambda_{4,(i_{2l+1})}\matLambda_{2,(i_{2l+1})}^{-p_{2l+1}-1}\matLambda_{5,(i_{2l+1})}\matLambda_{1}^{-p_{2l}}\matLambda_6^{-1}\right)\matLambda_{4,(i)}\matLambda_{2,(i)}^{-1}\notag\\
&+\sum_{k=0}^\infty\sum_{p_0+p_1\ge 1}^{\infty}\cdots\sum_{p_{2k}+p_{2k+1}\ge 1}^{\infty}\sum_{i_1,i_3,\cdots,i_{2k+1}=1}^N\left[\prod_{l=0}^{k}\left(\matF_{(0)}^{p_{2l}}\matF_{(i_{2l+1})}^{p_{2l+1}}\right)\right]\hmatH_{g,1}\notag\\
&\times \prod_{l=k}^0\left(\matLambda_{4,(i_{2l+1})}\matLambda_{4,(i_{2l+1})}\matLambda_{2,(i_{2l+1})}^{-p_{2l+1}-1}\matLambda_{5,(i_{2l+1})}\matLambda_{1}^{-p_{2l}}\matLambda_6^{-1}\right)\matLambda_{4,(i)}\matLambda_{2,(i)}^{-1}.\notag\\
\label{eqn:uirawform}
\end{align}

We thus prove~\eqref{eqn:ulcondition}.
\end{proof}

\revise{In Lemma \ref{lm:explicitsolution},  although the series of $\hmatH_g$ and $\hmatH_{(i),l}$'s have infinite terms, when $\alpha$ is not too large, the leading term is only the first term. This is delineated in the following lemma, which is a formal version of Lemma \ref{lm:informalhpertuabation}. For simplicity, we introduce a notation
\begin{equation}
\label{eqn:zetadef}
    \zeta = \frac{\abnb\left(\abnb+2\gop\right)}{\frac{3}{4}\sms}.
\end{equation}}

\revise{\begin{lemma}
\label{lm:usimpleexpansion}
Suppose that the conditions of Lemma \ref{lm:explicitsolution} are satisfied. Additionally, suppose that
$
\alpha\le \frac{6-3\sqrt{2}}{80}\frac{1}{\mu^2r}\left(\invcnum\right)^2
$, we have
\begin{align}
    \hmatH_g &= \hmatH_{g,0}+\dtmatH_g\nonumber\\
    \hmatH_{(i),l} &= \hmatH_{(i),l,0}+\dtmatH_{(i),l}\nonumber
\end{align}
where $\dtmatH_{g}$ and $\dtmatH_{(i),l}$ satisfy
\begin{align}
    \label{eqn:deltaugfnorm}
\norm{\dtmatH_g} &\le \zeta\left(\frac{2\gop^2}{\frac{3}{4}\sms}\right)^2\left(1+2\left(\frac{2\gop^2}{\frac{3}{4}\sms}\right)+2\left(\frac{2\gop^2}{\frac{3}{4}\sms}\right)^2\right)\\
\label{eqn:deltauginfnorm}
\max_j\norm{\vece_j^T\dtmatH_g} &\le \zeta\sqrt{\frac{\mu^2 r}{n_1}}4\left(\frac{2\gop^2}{\frac{3}{4}\sms}\right)^2\left(1+\left(\frac{2\gop^2}{\frac{3}{4}\sms}\right)+\left(\frac{2\gop^2}{\frac{3}{4}\sms}\right)^2\right)\\
\label{eqn:deltaulfnorm}
\norm{\dtmatH_{(i),l}} &\le \zeta \tkappa\left(2+\tkappa+3\brtkappa^2+4\brtkappa^3+4\brtkappa^4\right)\\
\label{eqn:deltaulinfnorm}
\max_j\norm{\vece_j^T\dtmatH_{(i),l}} &\le \zeta \sqrt{\murnone} 2\tkappa\notag\\
&\times\left(1+\tkappa+5\brtkappa^2+4\brtkappa^3+4\brtkappa^4\right) ,
\end{align}
with $\zeta$ is defined in~\eqref{eqn:zetadef}.
\end{lemma}}
\begin{proof}
We need to provide upper bounds on the series in Lemma \ref{lm:explicitsolution}. From Lemma \ref{lm:explicitsolution}, $\hmatH_g$ can be written as a series. We can define $\dtmatH_g$ as the summation of all but the first term in the series, as in
\begin{equation*}
\begin{aligned}
\dtmatH_g&=\hmatH_{g,1}+\sum_{k=0}^\infty\sum_{p_0+p_1\ge 1}^{\infty}\cdots\sum_{p_{2k}+p_{2k+1}\ge 1}^{\infty}\sum_{i_1,i_3,\cdots,i_{2k+1}=1}^N\left[\prod_{l=0}^{k}\left(\matF_{(0)}^{p_{2l}}\matF_{(i_{2l+1})}^{p_{2l+1}}\right)\right]\hmatH_{g,0}\\
&\times \prod_{l=k}^0\left(\matLambda_{4,(i_{2l+1})}\matLambda_{2,(i_{2l+1})}^{-p_{2l+1}-1}\matLambda_{5,(i_{2l+1})}\matLambda_{1}^{-p_{2l}}\matLambda_6^{-1}\right)\\
&+\sum_{k=0}^\infty\sum_{p_0+p_1\ge 1}^{\infty}\cdots\sum_{p_{2k}+p_{2k+1}\ge 1}^{\infty}\sum_{i_1,i_3,\cdots,i_{2k+1}=1}^N\left[\prod_{l=0}^{k}\left(\matF_{(0)}^{p_{2l}}\matF_{(i_{2l+1})}^{p_{2l+1}}\right)\right]\hmatH_{g,1}\\
&\times \prod_{l=k}^0\left(\matLambda_{4,(i_{2l+1})}\matLambda_{2,(i_{2l+1})}^{-p_{2l+1}-1}\matLambda_{5,(i_{2l+1})}\matLambda_{1}^{-p_{2l}}\matLambda_6^{-1}\right).\\
\end{aligned}
\end{equation*}
Hence, by applying Lemma~\ref{lm:productupperbound}, we have
\begin{align*}
&\norm{\dtmatH_g}\le\norm{\hmatH_{g,1}}+\sum_{k=0}^\infty\sum_{p_0+p_1\ge 1}^{\infty}\cdots\sum_{p_{2k}+p_{2k+1}\ge 1}^{\infty}\sum_{i_1,i_3,\cdots,i_{2k+1}=1}^N\\
&\left[\prod_{l=0}^{k}\left(\norm{\matF_{(0)}}^{p_{2l}}\norm{\matF_{(i_{2l+1})}}^{p_{2l+1}}\right)\right]\norm{\hmatH_{g,0}}\prod_{l=k}^0\left(\norm{\matLambda_{4,(i_{2l+1})}\matLambda_{2,(i_{2l+1})}^{-p_{2l+1}-1}\matLambda_{5,(i_{2l+1})}\matLambda_{1}^{-p_{2l}}\matLambda_6^{-1}}\right)\\
&+\sum_{k=0}^\infty\sum_{p_0+p_1\ge 1}^{\infty}\cdots\sum_{p_{2k}+p_{2k+1}\ge 1}^{\infty}\sum_{i_1,i_3,\cdots,i_{2k+1}=1}^N\left[\prod_{l=0}^{k}\left(\norm{\matF_{(0)}}^{p_{2l}}\norm{\matF_{(i_{2l+1})}}^{p_{2l+1}}\right)\right]\norm{\hmatH_{g,1}}\\
&\times \prod_{l=k}^0\norm{\matLambda_{4,(i_{2l+1})}\matLambda_{2,(i_{2l+1})}^{-p_{2l+1}-1}\matLambda_{5,(i_{2l+1})}\matLambda_{1}^{-p_{2l}}\matLambda_6^{-1}}.\\
\end{align*}
We first estimate an upper bound for $\norm{\hmatH_{g,1}}$,
\begin{align*}
&\norm{\hmatH_{g,1}}\le \sum_{p_0+p_1\ge 1}^{\infty}\sum_{i_1=1}^N\norm{\matF_{(0)}}^{p_0}\norm{\matF_{(i_1)}}^{p_1}\norm{\matT_{(i_1)}\hmatH_{(i_1),l}\matLambda_{2,(i_1)}^{-p_1-1}\matLambda_{5,(i_1)}\matLambda_1^{-p_0}\matLambda_6^{-1} }\\
&\le \sum_{p_0+p_1\ge 1}^{\infty} \left(\abnb\left(\abnb+2\gop\right)\right)^{p_0+p_1}\\
&\times 2\left(\frac{\gop^2}{\frac{3}{4}\sms}\right)^2 \left(\frac{1}{\frac{3}{4}\sms}\right)^{p_0+p_1}\\
&\le\frac{\abnb\left(\abnb+2\gop\right)}{\frac{3}{4}\sms}2\left(\frac{\gop^2}{\frac{3}{4}\sms}\right)^2\\
&2\left(1-\frac{\abnb\left(\abnb+2\gop\right)}{\frac{3}{4}\sms}\right)^{-1}\\
&\le \frac{\abnb\left(\abnb+2\gop\right)}{\frac{3}{4}\sms}4\sqrt{2}\left(\frac{\gop^2}{\frac{3}{4}\sms}\right)^2,
\end{align*}
where we used Lemma \ref{lm:productupperbound} in the first inequality, the upper bound on $\norm{\matF_{(i)}}$ in the second inequality. Because of the upper bound on $\alpha$, we can use auxiliary Lemma~\ref{lem_series} to derive an upper bound on the series. The last inequality comes from the fact that $(1-\frac{\abnb\left(\abnb+2\gop\right)}{\frac{3}{4}\sms})^{-1}\le \sqrt{2}$.

Therefore, we can proceed to estimate,
\begin{align*}
&\norm{\dtmatH_g}\\
&\le \zeta4\sqrt{2}\left(\frac{\gop^2}{\frac{3}{4}\sms}\right)^2+\sum_{k=0}^{\infty}\left(\zeta\frac{2}{1-\zeta}\right)^{k+1}\left(\frac{2\gop^2}{\frac{3}{4}\sms}\right)^{2(k+1)}\left(\frac{\gop^2}{\frac{3}{4}\sms}\right)^22 \zeta\frac{2}{1-\zeta}\\
&+\sum_{k=0}^{\infty}\left(\zeta\frac{2}{1-\zeta}\right)^{k+1}\left(\frac{2\gop^2}{\frac{3}{4}\sms}\right)^{2(k+1)}\left(\frac{\gop^2}{\frac{3}{4}\sms}\right)\left(1+\frac{2\gop^2}{\frac{3}{4}\sms}\right) \\
&\le \zeta 4\sqrt{2}\left(\frac{\gop^2}{\frac{3}{4}\sms}\right)^2\left(1-\frac{2\zeta}{1-\zeta}\left(\frac{2\gop^2}{\frac{3}{4}\sms}\right)^2\right)^{-1}\\
&+\zeta \frac{2}{1-\zeta}\left(\frac{2\gop^2}{\frac{3}{4}\sms}\right)^{2}\left(\frac{\gop^2}{\frac{3}{4}\sms}\right)\left(1+\frac{2\gop^2}{\frac{3}{4}\sms}\right)\left(1-\frac{2\zeta}{1-\zeta}\left(\frac{2\gop^2}{\frac{3}{4}\sms}\right)^2\right)^{-1}.\\
\end{align*}
We can estimate an upper bound on $\max_k\norm{\vece_k^T\dtmatH_g}$ in a similar fashion. 

We first show that, for any $j=1,2,\cdots,N$,
\begin{align}
&\max_k\norm{\vece_k^T\prod_{\ell=1}^k\matF_{(i_{\ell})}^{p_{\ell}}\matT_{(j)}}\notag\\
&= \max_k\norm{\vece_k^T\prod_{\ell=1}^k\matF_{(i_{\ell})}^{p_{\ell}}\left(\matHst_g\matHst_g^T+\matHst_{(j),l}\matHst_{(j),l}^T\right)\matT_{(j)}}\notag\\
&\le \max_k\norm{\vece_k^T\prod_{\ell=1}^k\matF_{(i_{\ell})}^{p_{\ell}}\matHst_g\matHst_g^T\matT_{(j)}}+\max_k\norm{\vece_k^T\prod_{\ell=1}^k\matF_{(i_{\ell})}^{p_{\ell}}\matHst_g\matHst_g^T\matT_{(j)}}\notag\\
&\le 
2\gop^2\sqrt{\frac{\mu^2 r}{n_1}}\left(\abnb (\abnb+2\gop)\right)^{\sum_{m=1}^kp_m},\notag\\
\label{eqn:eftjinfnorm}
\end{align}
where we used the triangle inequality in the first inequality, and Lemma \ref{lm:efuinfnorm} together with $r=r_1+r_2$ in the second inequality.

A similar equality also holds for $\max_k\norm{\vece_k^T\prod_{\ell=1}^k\matF_{(i_{\ell})}^{p_{\ell}}\matT_{(0)}}$:
\begin{equation}
\label{eqn:eft0infnorm}
\begin{aligned}
&\max_k\norm{\vece_k^T\prod_{\ell=1}^k\matF_{(i_{\ell})}^{p_{\ell}}\matT_{(0)}}=\max_k\norm{\vece_k^T\prod_{\ell=1}^k\matF_{(i_{\ell})}^{p_{\ell}}\frac{1}{N}\sum_{j=1}^N\matT_{(j)}}\le \frac{1}{N}\sum_{j=1}^N\max_k\norm{\vece_k^T\prod_{\ell=1}^k\matF_{(i_{\ell})}^{p_{\ell}}\matT_{(j)}}\\
&\le 2\gop^2\sqrt{\frac{\mu^2 r}{n_1}}\left(\abnb (\abnb+2\gop)\right)^{\sum_{m=1}^kp_{m}}.
\end{aligned}
\end{equation}
Combining the above two inequalities, we have
\begin{align*}
&\max_j\norm{\vece_j^T\dtmatH_g}\\
&\le \max_j\norm{\vece_j^T\hmatH_{g,1}}+\sum_{k=0}^\infty\sum_{p_0+p_1\ge 1}^{\infty}\cdots\sum_{p_{2k}+p_{2k+1}\ge 1}^{\infty}\sum_{i_1,i_3,\cdots,i_{2k+1}=1}^N\\
&\max_j\norm{\vece_j^T\left[\prod_{l=0}^{k}\left(\matF_{(0)}^{p_{2l}}\matF_{(i_{2l+1})}^{p_{2l+1}}\right)\right]\hmatH_{g,0}} \prod_{l=k}^0\norm{\matLambda_{4,(i_{2l+1})}\matLambda_{2,(i_{2l+1})}^{-p_{2l+1}-1}\matLambda_{5,(i_{2l+1})}\matLambda_{1}^{-p_{2l}}\matLambda_6^{-1}}\\
&+\sum_{k=0}^\infty\sum_{p_0+p_1\ge 1}^{\infty}\cdots\sum_{p_{2k}+p_{2k+1}\ge 1}^{\infty}\sum_{i_1,i_3,\cdots,i_{2k+1}=1}^N\max_j\norm{\vece_j^T\left[\prod_{l=0}^{k}\left(\matF_{(0)}^{p_{2l}}\matF_{(i_{2l+1})}^{p_{2l+1}}\right)\right]\hmatH_{g,1}}\\
&\times \prod_{l=k}^0\norm{\matLambda_{4,(i_{2l+1})}\matLambda_{2,(i_{2l+1})}^{-p_{2l+1}-1}\matLambda_{5,(i_{2l+1})}\matLambda_{1}^{-p_{2l}}\matLambda_6^{-1}}\\
&\le \sum_{p_0+p_1\ge 1}\zeta^{p_0+p_1}\sqrt{\murnone}\left(\frac{2\gop^2}{\frac{3}{4}\sms}\right)^2\\
&+\sum_{k=0}^{\infty}\sum_{p_0+p_1\ge 1}\cdots\sum_{p_{2k+2}+p_{2k+3}\ge 1}\zeta^{p_0+\cdots+p_{2k+3}}\sqrt{\murnone}\left(\frac{2\gop^2}{\frac{3}{4}\sms}\right)^2\left(\frac{2\gop^2}{\frac{3}{4}\sms}\right)^{2(k+1)}\\
&+\sum_{k=0}^{\infty}\sum_{p_0+p_1\ge 1}\cdots\sum_{p_{2k}+p_{2k+1}\ge 1}\zeta^{p_0+\cdots+p_{2k+1}}\left(\frac{2\gop^2}{\frac{3}{4}\sms}\right)^{2(k+1)}\left(\frac{2\gop^2}{\frac{3}{4}\sms}\right)^{2}\left(1+\left(\frac{2\gop^2}{\frac{3}{4}\sms}\right)^{2}\right)\\
&\le \zeta\frac{2}{1-\zeta}\sqrt{\murnone}\left(\frac{2\gop^2}{\frac{3}{4}\sms}\right)^2+\sum_{k=0}^{\infty}\left(\frac{2\zeta}{1-\zeta}\left(\frac{2\gop^2}{\frac{3}{4}\sms}\right)^{2}\right)^{k+1}\sqrt{\murnone}\left(\frac{2\gop^2}{\frac{3}{4}\sms}\right)^{2}\\
&+\sum_{k=0}^{\infty}\left(\frac{2\zeta}{1-\zeta}\left(\frac{2\gop^2}{\frac{3}{4}\sms}\right)^{2}\right)^{k+1}\sqrt{\murnone}\left(\frac{2\gop^2}{\frac{3}{4}\sms}\right)\left(1+\frac{2\gop^2}{\frac{3}{4}\sms}\right)\\
&\le \zeta\frac{2}{1-\zeta}\sqrt{\murnone}\left(\frac{2\gop^2}{\frac{3}{4}\sms}\right)^2\left(1-\frac{2\zeta}{1-\zeta}\left(\frac{2\gop^2}{\frac{3}{4}\sms}\right)^{2}\right)^{-1}\left(1+\left(\frac{2\gop^2}{\frac{3}{4}\sms}\right)\left(1+\frac{2\gop^2}{\frac{3}{4}\sms}\right) \right),\\
\end{align*} 
where we applied \eqref{eqn:eftjinfnorm}, \eqref{eqn:eft0infnorm} in the second inequality, and Lemma~\ref{lem_series} in the third inequality.

Similarly, we define $\dtmatH_{(i),l}$ as the summation,
\begin{align*}
&\dtmatH_{(i),l}=\sum_{p=1}^{\infty}\matF_{(i)}^p\matT_{(i)}\hmatH_{(i),l}\matLambda_{2,(i)}^{-p-1}+\sum_{p=1}^{\infty}\matF_{(i)}^p\hmatH_{g,0}\matLambda_{4,(i)}\matLambda_{2,(i)}^{-p-1}\notag\\
&+\sum_{p=0}^{\infty}\matF_{(i)}^p\dtmatH_{g}\matLambda_{4,(i)}\matLambda_{2,(i)}^{-p-1}.\notag\\
\end{align*}

We first calculate the $\ell_2$ norm of $\dtmatH_{(i),l}$ as,
\begin{align*}
&\norm{\dtmatH_{(i),l}}\le\sum_{p=1}^\infty \norm{\matF_{(i)}^p\matT_{(i)}\hmatH_{(i),l}\matLambda_{2,(i)}^{-p-1} }+\sum_{p=1}^{\infty}\norm{\matF_{(i)}^p\hmatH_{g,0}\matLambda_{4,(i)}\matLambda_{2,(i)}^{-p-1}}\\
&+\sum_{p=0}^{\infty}\norm{\matF_{(i)}^p\dtmatH_{g}\matLambda_{4,(i)}\matLambda_{2,(i)}^{-p-1}}\\
&\le \sum_{p=1}^{\infty}\zeta^p\gop^2\frac{1}{\frac{3}{4}\sms}+\sum_{p=1}^{\infty}\zeta^p\gop^2\frac{1}{\frac{3}{4}\sms}\frac{1}{2}\brtkappa^2\left(1+\tkappa\right)\\
&+\norm{\dtmatH_g}\sum_{p=0}^{\infty}\zeta^p\frac{\gop^2}{\frac{3}{4}\sms}\\
&\le \zeta\frac{1}{1-\zeta}\frac{\gop^2}{\frac{3}{4}\sms}\left(2+\tkappa+3\brtkappa^2+4\brtkappa^3+4\brtkappa^4\right),
\end{align*}
where we applied the upper bound on $\norm{\dtmatH}$ in the last inequality.

Finally, we have,

\begin{align*}
&\max_j\norm{\vece_j^T\dtmatH_{(i),l}}\\
&\le \sum_{p_0=1}^\infty \max_j\norm{\vece_j^T\matF_{(0)}^{p_0}\matT_{(i)}}\norm{\hmatH_{(i),l}\matLambda_{2,(i)}^{-p_{0}-1}}+\sum_{p=1}^{\infty}\max_j\norm{\vece_j^T\matF_{(i)}^p\hmatH_{g,0}\matLambda_{4,(i)}\matLambda_{2,(i)}^{-p-1}}\\
&+\sum_{p=0}^{\infty}\max_j\norm{\vece_j^T\matF_{(i)}^p\dtmatH_{g}}\norm{\matLambda_{4,(i)}\matLambda_{2,(i)}^{-p-1}}.\\
\end{align*}

The first two summations can be upper bounded by Lemma~\ref{lm:productupperbound}, and the last summation can be estimated in a similar way we calculate $\max_j\norm{\vece_j^T\dtmatH_{g}}$. We omit the details and present the estimated upper bound for brevity.

This completes our proof.
\end{proof}

\revise{Equipped with the aforementioned perturbation analysis on $\hmatH_g$ and $\matH_{(i),l}$, we are ready to provide the formal version of Lemma \ref{lm:informalldiffinfnorm}.
\begin{lemma}
\label{lm:ldiffinfnorm}
Under the same conditions as Lemma \ref{lm:usimpleexpansion}, we have:
$$
\norm{\matLst_{(i)}- \hmatL_{(i)}}_{\infty}\le \sqrt{\alpha} \mu^2 r \max_j\norm{\matE_{(j)}}_{\infty} C_4,
$$
where $C_4$ is a constant satisfying,
\begin{equation}
C_4= \mathcal{O}\left(\bcnum^{10}\frac{1}{\sqrt{\theta}}+\bcnum^{18}\right).
\end{equation}
\end{lemma}}
\begin{proof}
Notice that $\matLst_{(i)}=\matLst_{(i),g}+\matLst_{(i),l}$, and $\hmatL_{(i)}=(\matP_{\hmatH_g}+\matP_{\hmatH_{(i),l}})\hmatM_{(i)}=(\matP_{\hmatH_g}+\matP_{\hmatH_{(i),l}})( \matLst_{(i),g}+\matLst_{(i),l}+\matE_{(i),t})$. Therefore, we have

\begin{align}
&\norm{\matLst_{(i)}- \hmatL_{(i)}}_{\infty}\notag\\
&\le \norm{(\matP_{\hmatH_g}+\matP_{\hmatH_{(i),l}})( \matLst_{(i),g}+\matLst_{(i),l}+\matE_{(i),t}) - \matLst_{(i),g}-\matLst_{(i),l} }_{\infty}\notag\\
&\le \norm{\matP_{\hmatH_g}( \matLst_{(i),g}+\matLst_{(i),l}+\matE_{(i),t})- \matLst_{(i),g}}_{\infty}\notag\\
&+\norm{\matP_{\hmatH_{(i),l}}( \matLst_{(i),g}+\matLst_{(i),l}+\matE_{(i),t})- \matLst_{(i),l}}_{\infty}\notag\\
&\le\norm{\matP_{\hmatH_g}\matLst_{(i),l} }_{\infty}+\norm{\matP_{\hmatH_{(i),l}}\matLst_{(i),g} }_{\infty}\notag\\
&+\Big|\Big|\left(\hmatH_{g,0}\hmatH_{g,0}^T+\hmatH_{g,0}\dtmatH_{g}^T+\dtmatH_{g}\hmatH_{g,0}^T+\dtmatH_{g}\dtmatH_{g}^T\right)(\matLst_{(i),g}+\matE_{(i),t})-\matLst_{(i),g} \Big|\Big|_{\infty}\notag\\
&+\Big|\Big|\Big(\hmatH_{(i),l,0}\hmatH_{(i),l,0}^T+\hmatH_{(i),l,0}\dtmatH_{(i),l}^T+\dtmatH_{(i),l}\hmatH_{(i),l,0}^T+\dtmatH_{(i),l}\dtmatH_{(i),l}^T\Big)\notag\\
&(\matLst_{(i),l}+\matE_{(i),t})-\matLst_{(i),l} \Big|\Big|_{\infty}\notag\\
&\le \norm{\hmatH_{g,0}\hmatH_{g,0}^T\matLst_{(i),g}-\matLst_{(i),g}}_{\infty}+\norm{\hmatH_{(i),l,0}\hmatH_{(i),l,0}^T\matLst_{(i),l}-\matLst_{(i),l}}_{\infty}\notag\\
&+\norm{\hmatH_{g,0}\dtmatH_g^T\matLst_{(i),g}}_{\infty}+\norm{\dtmatH_g\hmatH_{g,0}^T\matLst_{(i),g}}_{\infty}+\norm{\dtmatH_g\dtmatH_g^T\matLst_{(i),g}}_{\infty}\notag\\
&+\norm{\hmatH_{g,0}\dtmatH_g^T\matE_{(i),t}}_{\infty}+\norm{\dtmatH_g\hmatH_{g,0}\matE_{(i),t}}_{\infty}+\norm{\dtmatH_g\dtmatH_g^T\matE_{(i),t}}_{\infty}\notag\\
&+\norm{\hmatH_{(i),l,0}\dtmatH_{(i),l}^T\matLst_{(i),l}}_{\infty}+\norm{\dtmatH_{(i),l}\hmatH_{(i),l,0}\matLst_{(i),l}}_{\infty}+\norm{\dtmatH_{(i),l}\dtmatH_{(i),l}^T\matLst_{(i),l}}_{\infty}\notag\\
&+\norm{\hmatH_{(i),l,0}\dtmatH_{(i),l}^T\matE_{(i),t}}_{\infty}+\norm{\dtmatH_{(i),l}\hmatH_{(i),l,0}^T\matE_{(i),t}}_{\infty}+\norm{\dtmatH_{(i),l}\dtmatH_{(i),l}^T\matE_{(i),t}}_{\infty}\notag\\
&+\norm{\matP_{\hmatH_g}\matLst_{(i),l} }_{\infty}+\norm{\matP_{\hmatH_{(i),l}}\matLst_{(i),g} }_{\infty}.\notag\\
\label{eqn:sketch}
\end{align}

There are $16$ terms in \eqref{eqn:sketch}, we will bound each of them respectively.

\noindent\underline{\it Bounding the first term of~\eqref{eqn:sketch}:}
\begin{align}
&\norm{\hmatH_{g,0}\hmatH_{g,0}^T\matLst_{(i),g}-\matLst_{(i),g}}_{\infty}\notag\\
&\le\norm{\matHst_{g}\matHst_{g}^T\hmatH_{g,0}\hmatH_{g,0}^T\matLst_{(i),g}-\matLst_{(i),g}}_{\infty}+\norm{\left(\matI-\matHst_{g}\matHst_{g}^T\right)\hmatH_{g,0}\hmatH_{g,0}^T\matLst_{(i),g}}_{\infty}\notag\\
&\le \murnb\norm{\matHst_{g}\matHst_{g}^T\hmatH_{g,0}\hmatH_{g,0}^T\matLst_{(i),g}-\matLst_{(i),g}}+ \norm{\left(\matI-\matHst_{g}\matHst_{g}^T\right)\hmatH_{g,0}\hmatH_{g,0}^T\matLst_{(i),g}}_{\infty}\notag\\
&\le \murnb\norm{\hmatH_{g,0}\hmatH_{g,0}^T\matLst_{(i),g}-\matLst_{(i),g}}+\murnb\norm{\left(\matI-\matHst_{g}\matHst_{g}^T\right)\hmatH_{g,0}\hmatH_{g,0}^T\matLst_{(i),g}}\notag\\
&+\norm{\left(\matI-\matHst_{g}\matHst_{g}^T\right)\hmatH_{g,0}\hmatH_{g,0}^T\matLst_{(i),g}}_{\infty}.\notag\\
\label{eqn:firstterm}\tag{TM1}
\end{align}

Recall the definition of $\hmatH_{g,0}$ as,
\begin{align*}
&\hmatH_{g,0} = \matT_{(0)}\hmatH_g\matLambda_6^{-1}+\sum_{i=1}^N
\matT_{(i)}\hmatH_{(i),l}\matLambda_{2,(i)}^{-1}\matLambda_{5,(i)}\matLambda_6^{-1}\\
&=\hmatH_g\underbrace{-\matF_{(0)}\hmatH_g\matLambda_6^{-1}-\sum_{i=1}^N
\matF_{(i)}\hmatH_{(i),l}\matLambda_{2,(i)}^{-1}\matLambda_{5,(i)}\matLambda_6^{-1}}_{\dtmatH_{g,0}   },
\end{align*}
where we used the KKT condition~\eqref{subeqn:kkthil} and the definition of $\matLambda_6=\matLambda_1+\sum_{i=1}^N\matLambda_{3,(i)}\matLambda_{2,(i)}^{-1}\matLambda_{5,(i)}$.

The first term in \eqref{eqn:firstterm} is thus bounded by,
\begin{align*}
&\murnb\norm{\hmatH_{g,0}\hmatH_{g,0}^T\matLst_{(i),g}-\matLst_{(i),g}}\\
&\le \murnb\norm{\hmatH_{g}\hmatH_{g}^T\matLst_{(i),g}-\matHst_{g}\matHst_{g}^T\matLst_{(i),g}}\\
&+\murnb\norm{\hmatH_{g}\dtmatH_{g,0}^T\matLst_{(i),g}}+\murnb\norm{\dtmatH_{g,0}\hmatH_{g}^T\matLst_{(i),g}}+\murnb\norm{\dtmatH_{g,0}\dtmatH_{g,0}^T\matLst_{(i),g}}.
\label{eqn:firstfirstterm}
\end{align*}

The second term in~\eqref{eqn:firstterm} is bounded by
\begin{align*}
&\norm{\left(\matI-\matHst_{g}\matHst_{g}^T\right)\hmatH_{g,0}\hmatH_{g,0}^T\matLst_{(i),g}}\notag\\
&\le \norm{\left(\matI-\matHst_{g}\matHst_{g}^T\right)\hmatH_{g}\hmatH_{g}^T\matLst_{(i),g}}+\norm{\left(\matI-\matHst_{g}\matHst_{g}^T\right)\dtmatH_{g,0}\hmatH_{g}^T\matLst_{(i),g}}\\
&+\norm{\left(\matI-\matHst_{g}\matHst_{g}^T\right)\hmatH_{g}\dtmatH_{g,0}^T\matLst_{(i),g}}+\norm{\left(\matI-\matHst_{g}\matHst_{g}^T\right)\dtmatH_{g,0}\dtmatH_{g,0}^T\matLst_{(i),g}}\\
&\le \norm{\hmatH_{g}\hmatH_{g}^T-\matHst_{g}\matHst_{g}^T}\gop+2\norm{\dtmatH_{g,0}}\gop+\norm{\dtmatH_{g,0}}^2\gop .
\end{align*}
The third term in~\eqref{eqn:firstterm} is bounded by

\begin{align}
&\norm{\left(\matI-\matHst_{g}\matHst_{g}^T\right)\hmatH_{g,0}\hmatH_{g,0}^T\matLst_{(i),g}}_{\infty}\notag\\
&=\norm{\sum_{i=1}^N
\left(\hmatHst_{(i)}\hmatHst_{(i)}^T\frac{\matT_{(i)}}{N}\hmatH_g\matLambda_6^{-1}+
\hmatHst_{(i)}\hmatHst_{(i)}^T\matT_{(i)}\hmatH_{(i),l}\matLambda_{2,(i)}^{-1}\matLambda_{5,(i)}\matLambda_6^{-1}\right)\matLst_{(i),g}}_{\infty}\notag\\
&\le \murnb\sum_{i=1}^N\norm{
\left(\hmatHst_{(i)}\hmatHst_{(i)}^T\frac{\matT_{(i)}}{N}\hmatH_g\matLambda_6^{-1}+
\hmatHst_{(i)}\hmatHst_{(i)}^T\matT_{(i)}\hmatH_{(i),l}\matLambda_{2,(i)}^{-1}\matLambda_{5,(i)}\matLambda_6^{-1}\right)\matLst_{(i),g}}.\notag\\
\end{align} 

We know that,
\begin{align*}
&\matHst_{(i),l}\matHst_{(i),l}^T\matT_{(i)}=\matP_{\hmatH_{(i),l}}\matT_{(i)}+\left(\matHst_{(i),l}\matHst_{(i),l}^T-\matP_{\hmatH_{(i),l}}\right)\matT_{(i)}\\
&=\matP_{\hmatH_{(i),l}}\matS_{(i)} - \matP_{\hmatH_{(i),l}}\matF_{(i)}+\left(\matHst_{(i),l}\matHst_{(i),l}^T-\matP_{\hmatH_{(i),l}}\right)\matT_{(i)},\\
\end{align*}
and that,
\begin{align*}
\matP_{\hmatH_{(i),l}}\frac{\matS_{(i)}}{N}\hmatH_g+\matP_{\hmatH_{(i),l}}\matS_{(i)}\hmatH_{(i),l}\matLambda_{2,(i)}^{-1}\matLambda_{5,(i)}=0.
\end{align*}

As a result, we have
\begin{align}
&\norm{\left(\matI-\matHst_{g}\matHst_{g}^T\right)\hmatH_{g,0}\hmatH_{g,0}^T\matLst_{(i),g}}_{\infty}\notag\\
&\le \murnb\sum_{i=1}^N\Big\lVert
\hmatHst_{(i)}\hmatHst_{(i)}^T\frac{-\matP_{\hmatH_{(i),l}}\matF_{(i)}+\left(\matHst_{(i),l}\matHst_{(i),l}^T-\matP_{\hmatH_{(i),l}}\right)\matT_{(i)}}{N}\notag\\
&\times\left(\hmatH_g+\hmatH_{(i),l}\matLambda_{2,(i)}^{-1}N\matLambda_{5,(i)}\right)\matLambda_6^{-1}\matLst_{(i),g}\Big\rVert\notag\\
&\le \murnb \gop\frac{1}{N}\sum_{i=1}^N\frac{\left(\norm{\matF_{(i)}}+\gop^2\norm{\matHst_{(i),l}\matHst_{(i),l}^T-\matP_{\hmatH_{(i),l}}}\right)}{\frac{3}{4}\sms}\left(1+\frac{2\gop^2}{{\frac{3}{4}\sms}}\right).\notag
\end{align}

Combing them all, we have,
\begin{align}
&\norm{\hmatH_{g,0}\hmatH_{g,0}^T\matLst_{(i),g}-\matLst_{(i),g}}_{\infty}\notag\\
&\le \murnb \gop\left(2\norm{\Delta\matP_g}+6\frac{\norm{\matF_{(0)}}}{\frac{3}{4}\sms}\brtkappa+\frac{1}{N}\sum_{i=1}^N\frac{\norm{\matF_{(i)}}+\gop^2\norm{\Delta\matP_{(i),l}} }{\frac{3}{4}\sms}\tkappa\right).
\label{eqn:firsttermagain}
\end{align}

\noindent\underline{\it Bounding the second term of~\eqref{eqn:sketch}:}

\begin{align}
&\norm{\hmatH_{(i),l,0}\hmatH_{(i),l,0}^T\matLst_{(i),l}-\matLst_{(i),l}}_{\infty}\notag\\
&\le \norm{\matHst_{(i),l}\matHst_{(i),l}^T\hmatH_{(i),l,0}\hmatH_{(i),l,0}^T\matLst_{(i),l}-\matLst_{(i),l}}_{\infty}\notag\\
&+\norm{\left(\matI-\matHst_{(i),l}\matHst_{(i),l}^T\right)\hmatH_{(i),l,0}\hmatH_{(i),l,0}^T\matLst_{(i),l}}_{\infty}\notag\\
&\le \murnb\norm{\matHst_{(i),l}\matHst_{(i),l}^T\hmatH_{(i),l,0}\hmatH_{(i),l,0}^T\matLst_{(i),l}-\matLst_{(i),l}}\notag\\
&+\norm{\left(\matI-\matHst_{(i),l}\matHst_{(i),l}^T\right)\hmatH_{(i),l,0}\hmatH_{(i),l,0}^T\matLst_{(i),l}}_{\infty}\notag\\
&\le \murnb\norm{\hmatH_{(i),l,0}\hmatH_{(i),l,0}^T\matLst_{(i),l}-\matLst_{(i),l}}+\murnb\norm{\left(\matI-\matHst_{(i),l}\matHst_{(i),l}^T\right)\hmatH_{(i),l,0}\hmatH_{(i),l,0}^T\matLst_{(i),l}}_{\infty}\notag\\
&+\norm{\left(\matI-\matHst_{(i),l}\matHst_{(i),l}^T\right)\hmatH_{(i),l,0}\hmatH_{(i),l,0}^T\matLst_{(i),l}}_{\infty}.\notag\\
\label{eqn:secondterm}\tag{TM2}
\end{align}

We will estimate upper bounds of three terms in~\eqref{eqn:secondterm} respectively. 

Recall that the definition of $\hmatH_{(i),l,0}$ as,
\begin{align*}
&\hmatH_{(i),l,0} = \matT_{(i)}\hmatH_{
(i),l}\matLambda_{2,(i)}^{-1}+
\hmatH_{g,0}\matLambda_{4,(i)}\matLambda_{2,(i)}^{-1}\\
&=\matS_{(i)}\hmatH_{
(i),l}\matLambda_{2,(i)}^{-1}+
\hmatH_{g}\matLambda_{4,(i)}\matLambda_{2,(i)}^{-1}-\matF_{(i)}\hmatH_{(i),l}\matLambda_{2,(i)}^{-1}+\dtmatH_{g,0}\matLambda_{4,(i)}\matLambda_{2,(i)}^{-1}\\
&=\hmatH_{(i),l}-\dtmatH_{(i),l}.
\end{align*}

The first term in \eqref{eqn:secondterm} is thus bounded by,
\begin{align*}
&\murnb\norm{\hmatH_{(i),l,0}\hmatH_{(i),l,0}^T\matLst_{(i),l}-\matLst_{(i),l}}\\
&\le \murnb\norm{\hmatH_{(i),l}\hmatH_{(i),l}\matLst_{(i),g}-\matHst_{g}\matHst_{g}^T\matLst_{(i),g}}\\
&+\murnb\norm{\hmatH_{(i),l}\dtmatH_{(i),l,0}^T\matLst_{(i),l}}+\murnb\norm{\dtmatH_{(i),l,0}\hmatH_{(i),l}^T\matLst_{(i),l}}+\murnb\norm{\dtmatH_{(i),l,0}\dtmatH_{(i),l,0}^T\matLst_{(i),l}}.
\end{align*}

The second term in~\eqref{eqn:secondterm} is upper bounded by
\begin{align*}
&\norm{\left(\matI-\matHst_{(i),l}\matHst_{(i),l}^T\right)\hmatH_{(i),l,0}\hmatH_{(i),l,0}^T\matLst_{(i),l}}\notag\\
&\le \norm{\left(\matI-\matHst_{(i),l}\matHst_{(i),l}^T\right)\matP_{\hmatH_{(i),l}}\matLst_{(i),l}}+\norm{\left(\matI-\matHst_{(i),l}\matHst_{(i),l}^T\right)\dtmatH_{(i),l,0}\hmatH_{(i),l}^T\matLst_{(i),l}}\\
&+\norm{\left(\matI-\matHst_{(i),l}\matHst_{(i),l}^T\right)\hmatH_{(i),l}\dtmatH_{(i),l,0}^T\matLst_{(i),l}}+\norm{\left(\matI-\matHst_{(i),l}\matHst_{(i),l}^T\right)\dtmatH_{(i),l,0}\dtmatH_{(i),l,0}^T\matLst_{(i),l}}\\
&\le \norm{\matP_{\hmatH_{(i),l}}-\matHst_{(i),l}\matHst_{(i),l}^T}\gop+2\norm{\dtmatH_{(i),l,0}}\gop+\norm{\dtmatH_{(i),l,0}}^2\gop.
\end{align*}

Then we bound the third term of~\eqref{eqn:secondterm}. From the definition of $\hmatH_{(i),l,0}$ and $\matT_{(i)}$, we know,
\begin{align*}
&\left(\matI-\matHst_{(i),l}\matHst_{(i),l}^T\right)\hmatH_{(i),l,0}\\
&=\matHst_g\matHst_g^T\matT_{(i)}\hmatH_{(i),l}\matLambda_{2,(i)}^{-1}+\matHst_g\matHst_g^T\hmatH_{g,0}\matLambda_{4,(i)}\matLambda_{2,(i)}^{-1}\\
&+\left(\matI-\matHst_g\matHst_g^T-\matHst_{(i),l}\matHst_{(i),l}^T\right)\hmatH_{g,0}\matLambda_{4,(i)}\matLambda_{2,(i)}^{-1}\\
&=\matHst_g\matHst_g^T\matT_{(i)}\hmatH_{(i),l}\matLambda_{2,(i)}^{-1}\\
&+\matHst_g\matHst_g^T\left(\matT_{(0)}\hmatH_{g}\matLambda_{6}^{-1}+\sum_{j=1}^N\matT_{(j)}\hmatH_{(j),l} \matLambda_{2,(j)}^{-1} \matLambda_{5,(j)}\matLambda_{6}^{-1}\right)\matLambda_{4,(i)}\matLambda_{2,(i)}^{-1}+\left(\matI-\matHst_{(i),l}\matHst_{(i),l}^T\right)\\
&\times\left(\sum_{j=1}^N\matHst_{(j),l}\matHst_{(j),l}^T\frac{\matT_{(j)}}{N}\hmatH_{g}+\sum_{j=1}^N\matHst_{(j),l}\matHst_{(j),l}^T\matT_{(j)}\hmatH_{(j),l} \matLambda_{2,(j)}^{-1} \matLambda_{5,(j)}\right)\matLambda_{6}^{-1}\matLambda_{4,(i)}\matLambda_{2,(i)}^{-1}\\
&=\matHst_g\matHst_g^T\matP_{\hmatH_g}\matS_{(i)}\hmatH_{(i),l}\matLambda_{2,(i)}^{-1}\\
&+\matHst_g\matHst_g^T\matP_{\hmatH_g}\left(\matS_{(0)}\hmatH_{g}\matLambda_{6}^{-1}+\sum_{j=1}^N\matS_{(j)}\hmatH_{(j),l} \matLambda_{2,(j)}^{-1} \matLambda_{5,(j)}\matLambda_{6}^{-1}\right)\matLambda_{4,(i)}\matLambda_{2,(i)}^{-1}\\
&-\matHst_g\matHst_g^T\matP_{\hmatH_g}\matF_{(i)}\hmatH_{(i),l}\matLambda_{2,(i)}^{-1}-\matHst_g\matHst_g^T\Delta\matP_g\matT_{(i)}\hmatH_{(i),l}\matLambda_{2,(i)}^{-1}\\
&-\matHst_g\matHst_g^T\matP_{\hmatH_g}\left(\matF_{(0)}\hmatH_{g}\matLambda_{6}^{-1}+\sum_{j=1}^N\matF_{(j)}\hmatH_{(j),l} \matLambda_{2,(j)}^{-1} \matLambda_{5,(j)}\matLambda_{6}^{-1}\right)\matLambda_{4,(i)}\matLambda_{2,(i)}^{-1}\\
&-\matHst_g\matHst_g^T\Delta\matP_g\left(\matT_{(0)}\hmatH_{g}\matLambda_{6}^{-1}+\sum_{j=1}^N\matT_{(j)}\hmatH_{(j),l} \matLambda_{2,(j)}^{-1} \matLambda_{5,(j)}\matLambda_{6}^{-1}\right)\matLambda_{4,(i)}\matLambda_{2,(i)}^{-1}\\
&+\left(\matI-\matHst_{(i),l}\matHst_{(i),l}^T\right)\\
&\times\left(\sum_{j=1}^N\matHst_{(j),l}\matHst_{(j),l}^T\hmatH_{(j),l}\hmatH_{(j),l}^T\matS_{(j)}\left(\frac{\hmatH_{g}}{N}+\hmatH_{(j),l} \matLambda_{2,(j)}^{-1} \matLambda_{5,(j)}\right)\right)\matLambda_{6}^{-1}\matLambda_{4,(i)}\matLambda_{2,(i)}^{-1}\\
&-\left(\matI-\matHst_{(i),l}\matHst_{(i),l}^T\right)\\
&\times\left(\sum_{j=1}^N\matHst_{(j),l}\matHst_{(j),l}^T\hmatH_{(j),l}\hmatH_{(j),l}^T\matF_{(j)}\left(\frac{\hmatH_{g}}{N}+\hmatH_{(j),l} \matLambda_{2,(j)}^{-1} \matLambda_{5,(j)}\right)\right)\matLambda_{6}^{-1}\matLambda_{4,(i)}\matLambda_{2,(i)}^{-1}\\
&-\left(\matI-\matHst_{(i),l}\matHst_{(i),l}^T\right)\\
&\times\left(\sum_{j=1}^N\matHst_{(j),l}\matHst_{(j),l}^T\Delta\matP_{(j),l}\matT_{(j)}\left(\frac{\hmatH_{g}}{N}+\hmatH_{(j),l} \matLambda_{2,(j)}^{-1} \matLambda_{5,(j)}\right)\right)\matLambda_{6}^{-1}\matLambda_{4,(i)}\matLambda_{2,(i)}^{-1}\\
\end{align*}.

From the KKT conditions, we know $\hmatH_g^T\matS_{i}\hmatH_{(i),l}+\hmatH_g^T\left(\matS_{(0)}\hmatH_{g}\matLambda_{6}^{-1}+\sum_{j=1}^N\matS_{(j)}\hmatH_{(j),l} \matLambda_{2,(j)}^{-1} \matLambda_{5,(j)}\matLambda_{6}^{-1}\right)\matLambda_{4,(i)}=0$ and $\hmatH_{(j),l}^T\matS_{(j)}\left(\frac{\hmatH_{g}}{N}+\hmatH_{(j),l} \matLambda_{2,(j)}^{-1} \matLambda_{5,(j)}\right)=0$.

Therefore, we have,
\begin{align}
&\max_k\abs{\vece_k^T\left(\matI-\matHst_{(i),l}\matHst_{(i),l}^T\right)\hmatH_{(i),l,0}}\notag\\
&\le \sqrt{\murnone}\Big(\frac{\norm{\matF_{(i)}}}{\frac{3}{4}\sms}+\frac{\norm{\matF_{(0)}}}{\frac{3}{4}\sms}\tkappa+\norm{\Delta\matP_{g}}\frac{\gop^2}{\frac{3}{4}\sms}\left(1+\tkappa+\brtkappa^2\right)\notag\\
&+\frac{1}{N}\sum_{j=1}^N\frac{\norm{\matF_{(j)}}}{\frac{3}{4}\sms}\tkappa\left(2+3\tkappa\right)\Big)+\frac{1}{N}\sum_{j=1}^N\norm{\Delta\matP_{(j)}}\frac{\gop^2}{\frac{3}{4}\sms}2\tkappa\left(1+\tkappa\right)\Big).\notag\\
\label{eqn:secondthird}
\end{align}

Combing these results, we have,
\begin{align}
&\norm{\hmatH_{(i),l,0}\hmatH_{(i),l,0}^T\matLst_{(i),l}-\matLst_{(i),l}}_{\infty}\notag\\
&\le \murnb \gop \times \Big(\frac{\norm{\matF_{(i)}}}{\frac{3}{4}\sms}+\frac{\norm{\matF_{(0)}}}{\frac{3}{4}\sms}\left(6+7\tkappa\right)\notag\\
&+\norm{\Delta\matP_{g}}\frac{\gop^2}{\frac{3}{4}\sms}\left(1+\tkappa+\brtkappa^2\right)+2\norm{\Delta\matP_{(i),l}}\notag\\
&+\frac{1}{N}\sum_{j=1}^N\frac{\norm{\matF_{(j)}}}{\frac{3}{4}\sms}\tkappa\left(2+3\tkappa\right)\Big)+\frac{1}{N}\sum_{j=1}^N\norm{\Delta\matP_{(j)}}\frac{\gop^2}{\frac{3}{4}\sms}2\tkappa\left(1+\tkappa\right)\Big).\notag\\
\end{align}
\noindent\underline{\it Bounding the third term of~\eqref{eqn:sketch}:}

\begin{align}
&\norm{\hmatH_{g,0}\dtmatH_g^T\matLst_{(i),g}}_{\infty}\notag\\
&=\max_{j,k}\abs{\vece_j^T\left(\hgzerodef\right)\dtmatH_g\matLst_{(i),g}\vece_k}\notag\\
&\le \frac{\mu^2 r}{\bn} \gop\norm{\dtmatH_g}\left(\tkappa+\brtkappa^2\right).\notag\\
\label{eqn:thirdterm}
\end{align}

\noindent\underline{\it Bounding the fourth term of~\eqref{eqn:sketch}:}
\begin{align}
&\norm{\hmatH_{g,0}\dtmatH_g^T\matE_{(i),t}}_{\infty}\notag\\
&=\max_{j,k}\abs{\sum_l\vece_j^T\hmatH_{g,0}\dtmatH_g^T\vece_l\vece_l^T\matE_{(i),t}\vece_k}\notag\\
&\le\max_{j,l}\abs{\vece_j^T\hmatH_{g,0}\dtmatH_g^T\vece_l}\alpha n_1 \Bi \notag\\
&=\max_{j,l}\abs{\vece_j^T\left(\hgzerodef\right)\dtmatH_g^T\vece_l}\alpha n_1 \Bi \notag\\
&\le\frac{\mu^2 r}{n_1}\frac{2\gop^2}{\frac{3}{4}\sms}\left(1+\tkappa\right)\alpha n_1\Bi,\notag\\
\label{eqn:fourthterm}
\end{align}
where we used the definition of $\alpha$-sparsity in the second inequality, and applied Lemma~\ref{lm:usimpleexpansion} in the last inequality.

\noindent\underline{\it Bounding the fifth term of~\eqref{eqn:sketch}:}
\begin{align*}
&\norm{\dtmatH_g\hmatH_{g,0}^T\matLst_{(i),g}}_{\infty}\\
&=\max_{j,k}\abs{\vece_j^T\dtmatH_g\left(\hgzerodef\right)\matLst_{(i),g}\vece_k}\\
&\le \max_j\norm{\vece_j^T\dtmatH_g}\sqrt{\murntwo}\norm{\hgzerodef}\\
&\le \gop\max_j\norm{\vece_j^T\dtmatH_g}\sqrt{\murntwo}\frac{\gop^2}{\frac{3}{4}\sms}\left(1+\tkappa\right).\\
\end{align*}

\noindent\underline{\it Bounding the sixth term of~\eqref{eqn:sketch}:}
\begin{align*}
&\norm{\dtmatH_g\hmatH_{g,0}^T\matE_{(i),t}}_{\infty}\\
&=\max_{j,k}\abs{\sum_l\vece_j^T\dtmatH_g\left(\hgzerodef\right)^T\matE_{(i),t}\vece_k}\\
&\le \max_j\norm{\vece_j^T\dtmatH_g}\tkappa\left(1+\tkappa\right)\sqrt{\frac{\mu^2 r}{n_1}}\alpha n_1\B,\\
\end{align*}
where we applied the incoherence condition on $\matT_{(0)}$ and $\matT_{(i)}$, and the relation $\sum_l\abs{\vece_l^T\matE_{(i),t}\vece_k}\le \alpha n_1\B$.

\noindent\underline{\it Bounding the seventh term of~\eqref{eqn:sketch}:}

\begin{align}
\label{eqn:seventhterm}
&\norm{\dtmatH_g\dtmatH_g^T\matLst_{(i),g}}_{\infty}\notag\\
&=\max_{j,k}\abs{\vece_j^T\dtmatH_g\dtmatH_g^T\matLst_{(i),g}\vece_k}\notag\\
&\le \max_{j}\norm{\vece_j^T\dtmatH_g}\sqrt{\murntwo}\gop,\notag\\
\end{align}
where we applied the incoherence on $\matLst_{(i),g}$ and $\norm{\dtmatH_g}\le 1$ in the first inequality.

\noindent\underline{\it Bounding the eighth term of~\eqref{eqn:sketch}:}
\begin{align*}
&\norm{\dtmatH_g\dtmatH_g^T\matE_{(i),t}}_{\infty}\\
&=\max_{j,k}\abs{\sum_l\vece_j^T\dtmatH_g\dtmatH_g^T\vece_l\vece_l^T\matE_{(i),t}\vece_k}\\
&\le \max_{j,l}\norm{\vece_j^T\dtmatH_g}\norm{\vece_l^T\dtmatH_g}\alpha n_1\B\\
&=\left(\max_{j}\norm{\vece_j^T\dtmatH_g}\right)^2\alpha n_1\B,
\end{align*}
where we applied $\sum_l\abs{\vece_l^T\matE_{(i),t}\vece_k}\le \alpha n_1 \B$ in the first inequality.

\noindent\underline{\it Bounding the ninth term of~\eqref{eqn:sketch}:}
\begin{align*}
&\norm{\hmatH_{(i),l,0}\dtmatH_{(i),l}^T\matLst_{(i),l}}_{\infty}\\
&=\max_{j,k}\abs{\vece_j^T\left(\hilzerodef\right)\dtmatH_{(i),l}^T\matLst_{(i),l}\vece_k}\\
&=\murnb\gop\tkappa\left(1+\tkappa+\left(\tkappa\right)^2\right)\norm{\dtmatH_{(i),l}},
\end{align*}
where we applied the incoherence condition of $\matT_{(0)}$ and $\matT_{(i)}$.

\noindent\underline{\it Bounding the tenth term of~\eqref{eqn:sketch}:}
\begin{align*}
&\norm{\hmatH_{(i),l,0}\dtmatH_{(i),l}^T\matE_{(i),t}}_{\infty}\\
&=\max_{j,k}\abs{\sum_l\vece_j^T\hmatH_{(i),l,0}\dtmatH_{(i),l}^T\vece_l\vece_l^T\matE_{(i),t}\vece_k}\\
&=\max_{j,l}\abs{\vece_j^T\hmatH_{(i),l,0}\dtmatH_{(i),l}^T\vece_l}\alpha n_1\B\\
&\le \sqrt{\frac{\mu^2 r}{n_1}}\tkappa\left(1+\tkappa+\left(\tkappa\right)^2\right)\norm{\vece_l^T\dtmatH_{(i),l}}\alpha n_1\B,\\
\end{align*}
where we applied the incoherence condition in the first inequality, \eqref{eqn:deltaulinfnorm} in the second inequality.

\noindent\underline{\it Bounding the eleventh term of~\eqref{eqn:sketch}:}
\begin{align*}
&\norm{\dtmatH_{(i),l}\hmatH_{(i),l,0}^T\matLst_{(i),l}}_{\infty}\\
&=\max_{j,k}\abs{\vece_j^T\dtmatH_{(i),l}\hmatH_{(i),l,0}^T\matLst_{(i),l}\vece_k}\\
&\le\max_{j}\norm{\vece_j^T\dtmatH_{(i),l}}\frac{\gop^2}{\frac{3}{4}\sms}\left(1+\tkappa+\brtkappa^2\right)\gop\sqrt{\frac{\mu^2 r}{n_2}},\\
\end{align*}
where we applied the incoherence condition in the first inequality, and \eqref{eqn:deltaulinfnorm} in the second inequality.

\noindent\underline{\it Bounding the twelfth term of~\eqref{eqn:sketch}:}

\begin{align}
&\norm{\dtmatH_{(i),l}\hmatH_{(i),l,0}^T\matE_{(i),t}}_{\infty}\notag\\
&=\max_{j,k}\abs{\sum_l\vece_j^T\dtmatH_{(i),l}\hmatH_{(i),l,0}^T\vece_l\vece_l^T\matE_{(i),l}\vece_k}\notag\\
&\le\max_{j,l}\abs{\vece_j^T\dtmatH_{(i),l}\hmatH_{(i),l,0}^T\vece_l}\alpha n_1\B\notag\\
&\le \max_{j}\abs{\vece_j^T\dtmatH_{(i),l}}\sqrt{\murnone}\frac{2\gop^2}{\frac{3}{4}\sms}\left(1+\tkappa+\brtkappa^2\right)\alpha n_1\B,\notag\\
\label{eqn:twelfthterm}
\end{align}
where we applied the condition $\sum_l\abs{\vece_l^T\matE_{(i),t}\vece_k}\le \alpha n_1\B$ in the first inequality, the incoherence condition in the second inequality.

\noindent\underline{\it Bounding the thirteenth term of~\eqref{eqn:sketch}:}
\begin{equation}
\label{eqn:thirteenthterm}
\begin{aligned}
&\norm{\dtmatH_{(i),l}\dtmatH_{(i),l}^T\matLst_{(i),l}}_{\infty}\\
&=\max_{j,k}\abs{\vece_j^T\dtmatH_{(i),l}\dtmatH_{(i),l}^T\matLst_{(i),l}\vece_k}\\
&\le \max_j\norm{\vece_j^T\dtmatH_{(i),l}}\norm{\dtmatH_{(i),l}}\gop \sqrt{\frac{\mu^2 r}{n_2}}\\
&\le \max_j\norm{\vece_j^T\dtmatH_{(i),l}}\gop \sqrt{\frac{\mu^2 r}{n_2}},\\
\end{aligned}
\end{equation}
where we apply the incoherence condition in the first inequality, and $\norm{\dtmatH_{(i),l}}\le 1$ in the second inequality.

\noindent\underline{\it Bounding the fourteenth term of~\eqref{eqn:sketch}:}
\begin{align*}
&\norm{\dtmatH_{(i),l}\dtmatH_{(i),l}^T\matE_{(i),t}}_{\infty}\\
&=\max_{j,k}\abs{\sum_m\vece_j^T\dtmatH_{(i),l}\dtmatH_{(i),l}^T\vece_m\vece_m^T\matE_{(i),t}\vece_k}\\
&\le \max_{j.m}\norm{\vece_j^T\dtmatH_{(i),l}}\norm{\vece_m^T\dtmatH_{(i),l}} \alpha n_1\B \\
&\le \left(\max_{j}\norm{\vece_j^T\dtmatH_{(i),l}}\right)^2 \alpha n_1\B,
\end{align*}
where we applied the condition $\sum_l\abs{\vece_l^T\matE_{(i),t}\vece_k}\le \alpha n_1\B$ in the first inequality.

\noindent\underline{\it Bounding the fifteenth term of~\eqref{eqn:sketch}:}
\begin{align*}
&\norm{\matP_{\hmatH_g}\matLst_{(i),l} }_{\infty}\\
&=\max_{j,k}\abs{\vece_j^T \left(\hmatH_{g,0}+\dtmatH_g\right)\hmatH_g^T\matHst_{(i),l}\matSigma_{(i),l}\matWst_{(i),l}^T\vece_k}\\
&\le \left(\max_j\abs{\vece_j^T\hmatH_{g,0}}+\max_j\abs{\vece_j^T\dtmatH_{g}}\right)\norm{\hmatH_g^T\matHst_{(i),l}}\gop \sqrt{\frac{\mu^2 r}{n_2}}\\
&\le \left(\max_j\abs{\vece_j^T\hmatH_{g,0}}+\max_j\abs{\vece_j^T\dtmatH_{g}}\right)\norm{\Delta\matP_g}\gop \sqrt{\frac{\mu^2 r}{n_2}}\\
&=\murnb\tkappa\left(1+\tkappa\right)\norm{\Delta\matP_g}\gop+ \gop \sqrt{\frac{\mu^2 r}{n_2}}\norm{\Delta\matP_g}\max_j\abs{\vece_j^T\dtmatH_{g}}.
\end{align*}

The second inequality comes from the relation $\norm{\hmatH_g^T\matHst_{(i),l}}=\norm{\hmatH_g^T\left(\matI-\matHst_g\matHst_g^T\right)\matHst_{(i),l}}=\norm{\hmatH_g^T\left(\matP_{\hmatH_g}-\matP_{\hmatH_g}\matHst_g\matHst_g^T\right)\matHst_{(i),l}}\le\norm{\matP_{\hmatH_g}-\matP_{\hmatH_g}\matHst_g\matHst_g^T}=\norm{\matP_{\hmatH_g}(\matP_{\hmatH_g}-\matHst_g\matHst_g^T)}\le \norm{\matP_{\hmatH_g}-\matHst_g\matHst_g^T}$. 

\noindent\underline{\it Bounding the sixteenth term of~\eqref{eqn:sketch}:}

\begin{align}
&\norm{\matP_{\hmatH_{(i),l}}\matLst_{(i),g} }_{\infty} =\max_{j,k}\abs{\vece_j^T\left(\hmatH_{(i),0}+\dtmatH_{(i),l}\right)\hmatH_{(i),l}^T\matHst_{g}\matSigmast_{(i),g}\matWst_{(i),g}^T\vece_k }\notag\\
&\le \left(\max_j\norm{\vece_j^T\hmatH_{(i),l,0}}+\max_j\norm{\vece_j^T\dtmatH_{(i),l}}\right)\norm{\hmatH_{(i),l}^T\matHst_{g}}\gop\sqrt{\frac{\mu^2 r}{n_2}}\notag\\
&\le \left(\max_j\norm{\vece_j^T\hmatH_{(i),l,0}}+\max_j\norm{\vece_j^T\dtmatH_{(i),l}}\right)\norm{\Delta\matP_{(i),l}}\gop\sqrt{\frac{\mu^2 r}{n_2}}\notag\\
&\le \murnb\gop\tkappa\left(1+\tkappa+\brtkappa^2\right)\norm{\Delta\matP_{(i),l}}\\
&+\norm{\Delta\matP_{(i),l}}\gop\sqrt{\frac{\mu^2 r}{n_2}}\max_j\norm{\vece_j^T\dtmatH_{(i),l}}
\label{eqn:sixteenthterm}\tag{TM16}
\end{align}, 
where we again applied Lemma \ref{lm:usimpleexpansion} in the first inequality. The second inequality comes from the relation $\norm{\hmatH_{(i),l}^T\matHst_{g}}=\norm{\hmatH_{(i),l}^T\left(\matI-\matP_{\hmatH_g}\right)\matHst_{g}}=\norm{\hmatH_{(i),l}^T\left(\matP_{\matHst_g}-\matP_{\hmatH_g}\matP_{\matHst_g}\right)\matHst_{g}}\le\norm{\matP_{\matHst_g}-\matP_{\hmatH_g}\matP_{\matHst_g}}\le \norm{\matP_{\matHst_g}-\matP_{\matHst_g}}$.

Combining these sixteen terms \eqref{eqn:firstterm}-\eqref{eqn:sixteenthterm} and considering the fact that $\alpha\le 1$, we have,
$$
\norm{\matLst_{(i)}- \hmatL_{(i)}}_{\infty}\le \sqrt{\alpha} \mu^2 r \norm{\matE}_{\infty} C_4,
$$
where 

\begin{align}
\label{eqn:c4def}
&C_4=34327\brtkappa^{9}+534\brtkappa^5\theta^{-\frac{1}{2}}=\mO\left(\kappa^{9}+\kappa^{5}\theta^{-\frac{1}{2}}\right).
\end{align}

This completes our proof.
\end{proof}
Finally we will prove Theorem \ref{thm:recovery}. We will first state its formal version below.
\begin{theorem}
\label{thm:formalrecovery}
Suppose that the conditions of Lemma \ref{lm:usimpleexpansion} are satisfied. Additionally, suppose that there exists a constant $0<\rhom<1$ such that  $\alpha \le \frac{\rhom^2}{4\mu^4r^2C_4^2}$. Then, the following statements hold at iteration $t\geq 1$ of Algorithm \ref{alg:altmin} with $\lambda_1= \frac{\gop\mu^2r}{\sqrt{n_1n_2}}$, $\epsilon\le \lambda_1\left(1-\rhom\right)$, and $1-\frac{\epsilon}{\lambda_1}>\rho\ge \rhom$:
\begin{enumerate}
    \item \label{condition:supportcontain}$ \supp{\hmatS_{(i),t}}\subset\supp{\matSst_{(i)}}$ for every $i\in [N]$.
    \item     \label{condition:sdifinfnorm}
    $\norm{\hmatS_{(i),t}-\matSst_{(i)}}_{\infty}\le 2\lambda_t\le 4\gop \frac{\mu^2 r}{\bn}$ for every $i\in [N]$.
    \item     \label{condition:ldifinfnorm}
    $\norm{\hmatLeps_{(i),t}-\matLst}_{\infty}\le \epsilon+\rho\lambda_t$ for every $i\in [N]$.
\end{enumerate} Moreover, we have
\begin{equation}
\label{eqn:ugvgdiffinfnorm}
\norm{\hmatUeps_{g,t}\hmatVeps_{(i),g,t}^T-\matUst_g\matVst_{(i),g}^T}_{\infty}= \mathcal{O}\left(\rho^t+\frac{\epsilon}{1-\rho}\right),\quad   \text{for every  } i\in [N].
\end{equation}

and 
\begin{equation}
\label{eqn:uilvildiffinfnorm}
\norm{\hmatUeps_{(i),l,t}\hmatVeps_{(i),l,t}^T-\matUst_{(i),l}\matVst_{(i),l}^T}_{\infty}= \mathcal{O}\left(\rho^t+\frac{\epsilon}{1-\rho}\right),\quad   \text{for every  } i\in [N].
\end{equation}
\end{theorem}
\paragraph{Remark} The definition of the term $\rho_{\min}$ in the statement of the above theorem is kept intentionally implicit to streamline the presentation. In what follows,  we will give an estimate of the requirements on $\alpha$ purely in terms of the parameters of the problem. Lemma~\ref{lm:deltapupperbound} requires $\alpha= \mO \left(\frac{\theta}{\mu^4 r^2}\bcnum^8\right) $. Lemma~\ref{lm:eigenvalueupperbound} requires $\alpha= \mO \left(\frac{1}{\mu^2 r}\right)$.  Lemma \ref{lm:eigenvaluelowerbound} requires $\alpha= \mathcal{O}\left(\frac{\theta}{\mu^4r^2}\left(\invcnum\right)^{12}\right)$. Lemma~\ref{lm:l6eigenvaluelowerbound} requires $\alpha= \mathcal{O}\left(\frac{\theta}{\mu^4r^2}\left(\invcnum\right)^{12}\right)$. Lemma \ref{lm:explicitsolution} requires  $\alpha=\mathcal{O}\left(\frac{\theta}{\mu^2r}\left(\invcnum\right)^6\right)$. And Lemma \ref{lm:usimpleexpansion} requires $\alpha= \mathcal{O}\left(\frac{\theta}{\mu^2r}\left(\invcnum\right)^2\right)$. As $C_4=\mathcal{O}\left(\frac{1}{\sqrt{\theta}}\bcnum^{10}\right)$, the additional requirement in Theorem \ref{thm:formalrecovery} requires $\alpha= \mathcal{O}\left(\frac{\theta}{\mu^4r^2}\left(\invcnum\right)^{20}\right)$. Taking the intersections of all these requirements, we can derive the upper bound on $\alpha$ as $ \alpha = \mathcal{O}\left(\frac{\theta}{\mu^4r^2}\left(\invcnum\right)^{20}\right)$.

\begin{proof}
We will prove this theorem by induction.\vspace{1mm}

\noindent \underline{\it Base case:}
At $t=1$, $\hmatL_{(i),0}=0$. As $\lambda_1=\frac{\mu^2 r}{\bn}\gop$, we have  $\hmatS_{(i),1}=\hard{\frac{\mu^2 r}{\bn}\gop}{\matM_{(i)}}$. By definition of hard-thresholding, if the $jk$-th entry of $\hmatS_{(i),1}$ is nonzero, we know $\abs{[\matM_{(i)}]_{jk}}> \frac{\mu^2 r}{\bn}\gop$. Since $\abs{[\matLst_{(i)}]_{jk}}\le \frac{\mu^2 r}{\bn}\gop$ for each $j$ and $k$, we must have $\abs{[\matSst_{(i)}]_{jk}}>0$. This proves Claim \ref{condition:supportcontain} for $t=1$.

Now we will prove Claim \ref{condition:sdifinfnorm} holds when $t=1$. If $[\hmatS_{(i),1}]_{jk}= 0$, we know $\abs{[\matSst_{(i)}]_{jk}+[\matLst_{(i)}]_{jk}}\le \mu^2r/\bn \gop$, thus $\abs{[\matSst_{(i)}]_{jk}}\le 2\mu^2r/\bn \gop$. If $[\hmatS_{(i),1}]_{jk}\neq0$, by the definition of hard-thresholding, we know $[\hmatS_{(i),1}]_{jk}= [\matM_{(i)}]_{jk}= [\matSst_{(i)}]_{jk}+[\matLst_{(i)}]_{jk}$. By rearranging terms, we have $\abs{[\matSst_{(i)}]_{jk}-[\hmatS_{(i),1}]_{jk}}=\abs{[\matLst_{(i)}]_{jk}}\le \mu^2r/\bn \gop $. We hence proved Claim \ref{condition:sdifinfnorm} for $t=1$.

Since $\matE_{(i),1}=\matSst_{(i)}-\hmatS_{(i),1}$, we have $\norm{\matE_{(i),1}}_{\infty}\le 2\frac{\mu^2r}{\bn} \gop$ for each $i$ as well. Also, by Claim \ref{condition:supportcontain}, $\matE_{(i),1}$'s are $\alpha$-sparse. Therefore by Lemma \ref{lm:ldiffinfnorm}, when $\alpha \le \frac{\rhom^2}{4\mu^4r^2C_4^2}\le \frac{\rho^2}{4\mu^4r^2C_4^2} $, $\norm{\hmatL_{(i),1}-\matLst_{(i)}}_{\infty}\le 2\sqrt{\alpha}\frac{\mu^2 r}{\bn}\gop C_4\le \rho\lambda_1$. From the definition of $\epsilon$-optimality and triangle inequality, we know $\norm{\hmatLeps_{(i),1}-\matLst_{(i)}}_{\infty}\le \rho\lambda_1 + \epsilon$. We thus proved Claim \ref{condition:ldifinfnorm} for $t=1$.\vspace{1mm}

\noindent \underline{\it Induction step:}
Now supposing that Claims \ref{condition:supportcontain}, \ref{condition:sdifinfnorm}, and \ref{condition:ldifinfnorm} hold for iterations $1,\cdots,t$, we will show their correctness for the iteration $t+1$.
Since Claim \ref{condition:ldifinfnorm} holds for iteration $t$, we know$\norm{\hmatLeps_{(i),t}-\matLst_{(i)}}_{\infty}\le \rho \lambda_t + \epsilon $ under the condition $\alpha\le \frac{\rho^2}{4\mu^4 r^2 C_4^2 }$. With the choice of $\lambda_{t+1}=\rho \lambda_t+\epsilon$, if the $jk$-th entry of $\hmatS_{(i),t+1}$ is nonzero, we have $\abs{[\matSst_{(i)}]_{jk}+[\matLst_{(i)}]_{jk}-[\hmatLeps_{(i),t}]_{jk}}> \lambda_{t+1}$. Since $\abs{[\matLst_{(i)}]_{jk}-[\hmatL_{(i),t}]_{jk}}\le \lambda_{t+1}$, we must have $\abs{[\matSst_{(i)}]_{jk}}>0$. This proves Claim \ref{condition:supportcontain} for iteration $t+1$.

We will now proceed to prove Claim \ref{condition:sdifinfnorm}. We consider each entry of $\hmatS_{(i),t+1}=\hard{\lambda_{t+1}}{\matSst_{(i)}+\matLst_{(i)}-\hmatLeps_{(i),t}}$. From the definition of hard-thresholding, we know $\abs{[\hmatS_{(i),t+1}]_{jk}-\left([\matSst_{(i)}]_{jk}+[\matLst_{(i)}]_{jk}-[\hmatLeps_{(i),t}]_{jk}\right)}\le \lambda_{t+1}$. Remember that we know $\abs{[\matLst_{(i)}]_{jk}-[\hmatLeps_{(i),t}]_{jk}}\le \lambda_{t+1}$ from the correctness of Claim \ref{condition:ldifinfnorm} at iteration $t$ and the upper bound on $\alpha$, we can derive $\abs{[\hmatS_{(i),t+1}]_{jk}-[\matSst_{(i)}]_{jk}}\le 2\lambda_{t+1}$ by triangle inequality. 
We hence prove Claim \ref{condition:sdifinfnorm}.

For Claim \ref{condition:ldifinfnorm}, since $\matE_{(i),t+1}=\matSst_{(i)}-\hmatS_{(i),t+1}$, we have $\norm{\matE_{(i),t+1}}_{\infty}\le 2\lambda_{t+1}$ for each $i$ as well. Also, by Claim \ref{condition:supportcontain} at iteration $t$, $\matE_{(i),t}$'s are $\alpha$-sparse at iteration $t$. Therefore by Lemma \ref{lm:ldiffinfnorm}, $\norm{\hmatL_{(i),t+1}-\matLst}_{\infty}\le 2\sqrt{\alpha}\mu^2 r C_4 \lambda_{t+1}$. Under the constraint that $\alpha\le \frac{\rho^2}{4\mu^4r^2C_4^2}$, we know $\norm{\hmatL_{(i),t+1}-\matLst}_{\infty}\le\rho\lambda_{t+1}$. From the definition of $\epsilon$-optimality and triangle inequality, we have $\norm{\hmatLeps_{(i),t+1}-\matLst}_{\infty}\le\rho\lambda_{t+1}+\epsilon$. We thus proved Claim \ref{condition:ldifinfnorm} at iteration $t+1$.

Combining them, we can conclude that \ref{condition:supportcontain}, \ref{condition:sdifinfnorm}, and \ref{condition:ldifinfnorm} hold for every $t=1,2,\cdots$. 

Finally, we will prove \eqref{eqn:ugvgdiffinfnorm} and \eqref{eqn:uilvildiffinfnorm}. We have known that $\hmatU_{g,t}\hmatV_{(i),g,t}^T=\matP_{\hmatH_g}\hmatM_{(i)}$, then from similar analysis of \eqref{eqn:sketch}, we have,
$$
\begin{aligned}
&\norm{\hmatU_{g,t}\hmatV_{(i),g,t}^T-\matUst_{g}\matVst_{(i),g}^T}_{\infty}\\
&=\norm{\matP_{\hmatH_g}\left(\matLst_{(i),g}+\matLst_{(i),l}+\matE_{(i)}\right)- \matLst_{(i),g}}_{\infty}\\
&\le\norm{\matP_{\hmatH_g}\matLst_{(i),l} }_{\infty}\\
&+\Big|\Big|\left(\matT_{(0)}\hmatH_g\matLambda_1^{-2}\hmatH_g^T\matT_{(0)}+\matT_{(0)}\hmatH_g\matLambda_1^{-1}\dtmatH_g^T+\dtmatH_g\matLambda_1^{-1}\hmatH_g^T\matT_{(0)}+\dtmatH_g\dtmatH_g^T\right)\\
&(\matLst_{(i),g}+\matE_{(i),t})-\matLst_{(i),g} \Big|\Big|_{\infty}.\\
\end{aligned}
$$
In Lemma \ref{lm:ldiffinfnorm}, we have shown that each term above is upper bounded by $\mathcal{O}(\max_i\norm{\matE_{(i),t}}_{\infty})$. Therefore by Claim \ref{condition:sdifinfnorm}, we have $\norm{\hmatU_{g,t}\hmatV_{(i),g,t}^T-\matUst_{g}\matVst_{(i),g}^T}_{\infty}= \mathcal{O}(\max_i\norm{\matE_{(i),t}}_{\infty})=\mathcal{O}(\lambda_t)=\mathcal{O}(\rho^t+\frac{\epsilon}{1-\rho})$. \eqref{eqn:ugvgdiffinfnorm} follows accordingly by triangle inequality. 

We can prove \eqref{eqn:uilvildiffinfnorm} in a  similar way. This completes our proof of Theorem \ref{thm:formalrecovery}.
\end{proof}

\section{Auxiliary Lemma}\
This section discusses some helper lemmas useful for our main proofs. These lemmas are mostly derived from basic linear algebra and series.

The following lemma is a well-known result and provides an upper bound on the norm of product matrices.
\begin{lemma}
\label{lm:productupperbound}
Form two matrices $\matA\in\mathbb{R}^{m\times n}$ and $\matB\in\mathbb{R}^{n\times p}$, we have,
$$
\norm{\matA\matB}_F\le \norm{\matA}_2\norm{\matB}_F
$$
and
$$
\norm{\matA\matB}_2\le \norm{\matA}_2\norm{\matB}_2
$$
\end{lemma}
\begin{proof}
    The proof is straightforward and can be found in \cite{ruoyuconvergence}.
\end{proof}

\revise{\begin{lemma}\label{lem_series}
For $x,y\in[0,1)$ such that $x+y<1$, the following relation holds:
\begin{subequations}
    \begin{align}
        \label{eqn:series1}
\sum_{p_1+p_2\ge 1}x^{p_1+p_2}\le \frac{2x}{1-x}.
    \end{align}
\end{subequations}
\end{lemma}}
\begin{proof}
This proof follows from the direct calculation.

\begin{align*}
&\sum_{p_1+p_2\ge 1}x^{p_1+p_2}=\sum_{p_1=0}^{\infty}\sum_{p_2=0}^{\infty}x^{p_1+p_2}-1\\
&= \left( \sum_{p_1=0}^{\infty}x^{p_1}\right)\left(\sum_{p_2=0}^{\infty}x^{p_2}\right)-1 = \frac{1}{(1-x)^2}-1\\
&=\frac{2x-x^2}{(1-x)^2}\le\frac{2x}{1-x}.
\end{align*}

\end{proof}

We also present a lemma related to the Schur complement of block matrices.

\revise{\begin{lemma}
\label{lm:schurcomplement}
For symmetric matrices $\matA_0 \in \mathbb{R}^{r_0\times r_0}, \matA_{1}\in \mathbb{R}^{r_1\times r_1},\cdots, \matA_{N}\in \mathbb{R}^{r_1\times r_1}$, and $\matB_{i}\in \mathbb{R}^{r_0\times r_N}$ for $i\in\{1,\cdots,N\}$, we can construct a symmetric block matrix $\matC$ as,
\begin{equation}
\matC = 
\begin{pmatrix}
\matA_0   & \matB_1     & \matB_2   &\cdots & \matB_N\\
\matB_1^T & \matA_1     & 0     &\cdots & 0\\
\matB_2^T & 0       & \matA_2   &\cdots & 0\\
\vdots     &\vdots   &\ddots &\cdots & 0\\
\matB_N^T & 0       & 0     &\cdots & \matA_N
\end{pmatrix}.
\end{equation}
Then, $\matC$ is positive definite if and only if $\matA_1,\matA_2,\cdots,\matA_N$ are positive definite and $\matA_0 -\sum_{i=1}^N\matB_i\matA_i^{-1}\matB_i^T$ is positive definite. 
\end{lemma}}
\revise{\begin{proof}
Sicne $\matA_i$'s are positive definite, they are invertible. Thus we can decompose $\matC$ as
\begin{align*}
&\matC =     \underbrace{\begin{pmatrix}
\matI   & \matB_1\matA_1^{-1}     & \matB_2\matA_2^{-1}   &\cdots & \matB_N\matA_N^{-1}\\
0 & \matI     & 0     &\cdots & 0\\
0 & 0       & \matI   &\cdots & 0\\
\vdots     &\vdots   &\vdots &\ddots & 0\\
0 & 0       & 0     &\cdots & \matI
\end{pmatrix}}_{\matC_1}
\underbrace{\begin{pmatrix}
\matA_0-\sum_i\matB_i\matA_i^{-1}\matB_i^T   & 0    &\cdots & 0\\
0 & \matA_1       &\cdots & 0\\
\vdots     &\vdots   &\ddots  & 0\\
 0       & 0     &\cdots & \matA_N
\end{pmatrix}}_{\matC_2}\\
&\quad \underbrace{\cdot\begin{pmatrix}
\matI   & 0     & 0   &\cdots & 0\\
\matA_1^{-1}\matB_1^T & \matI     & 0     &\cdots & 0\\
\matA_2^{-1}\matB_2^T & 0       & \matI  &\cdots & 0\\
\vdots     &\vdots   &\vdots &\ddots & 0\\
\matA_{N}^T\matB_N^T & 0       & 0     &\cdots & \matI
\end{pmatrix}}_{\matC_1^T}.
\end{align*}
On the right hand side, $\matC_1$ and $\matC_1^T$ are both invertible. Thus, $\matC$ is positive definite if and only if $\matC_2$ is positive definite. Since $\matC_2$ is a block diagonal matrix, we prove the statement in the lemma.  
\end{proof}}

The following lemma provides an eigenvalue lower bound on the product of three matrices.
\revise{\begin{lemma}
\label{lm:lminatba}
For matrix $\matA \in \mathbb{R}^{n\times n}$, and symmetric positive semidefinite matrix $\matB \in \mathbb{R}^{n\times n}$, we know that,
\begin{align*}
\lambda_{\min}\left(\matA^T\matB\matA\right)\ge \lambda_{\min}\left(\matB\right)\lambda_{\min}\left(\matA^T\matA\right).
\end{align*} 
\end{lemma}}

\revise{\begin{proof}
The proof follows from the Courant–Fischer–Weyl variational principle.
\begin{align*}
&\lambda_{\min}\left(\matA^T\matB\matA\right) = \min_{\norm{\vecv}=1}\vecv^T\matA^T\matB\matA\vecv\\
&\ge \lambda_{\min}\left(\matB\right)\min_{\norm{\vecv}=1}\norm{\matA\vecv}^2\\
&= \lambda_{\min}\left(\matB\right)\lambda_{\min}\left(\matA^T\matA\right).
\end{align*}
\end{proof}}

We finally present the lemma that provides an upper bound of the operator norm of block matrices.

\revise{\begin{lemma}
\label{lm:blockmatrixopnormupperbound}
    For a symmetric block matrix $\matC$ defined as
    \begin{equation}
\matC = 
\begin{pmatrix}
\matA_1   & \matB_{12}     & \matB_{13}   &\cdots & \matB_{1N}\\
\matB_{12}^T & \matA_2     & \matB_{23}     &\cdots & \matB_{2N}\\
\matB_{13}^T & \matB_{23}^T       & \matA_3   &\cdots & \matB_{3N}\\
\vdots     &\vdots   &\vdots &\ddots & \vdots\\
\matB_{1N}^T & \matB_{2N}^T       & \matB_{3N}^T     &\cdots & \matA_N
\end{pmatrix},
\end{equation}
where $\matA_{i}$'s are symmetric, we have,
\begin{align}
\label{apeqn:snormaplusb}
    \norm{\matC}\le \max_{i=1,\cdots,N}\{\norm{\matA_i}\}+\sqrt{2\sum_{i< j}\norm{\matB_{ij}}^2}.
\end{align}
\end{lemma}}
Note: in~\eqref{apeqn:snormaplusb}, the diagonal blocks and off-diagonal blocks are treated differently.
\begin{proof}
    We first prove for the special case where $\matB_{ij}=0$. In this case, 
    \begin{align*}
    &\norm{\matC}^2=\max_{\norm{\vecv}=1}\norm{\matC\vecv}^2\\
    &=\max_{\norm{\vecv}=1}\sum_{i=1}^N\norm{\matA_i\vecv_i}^2\le \max_{\norm{\vecv}=1}\sum_{i=1}^N\norm{\matA_i}^2\norm{\vecv_i}^2\\
    &\le \max_{i=1,\cdots,N}\{\norm{\matA_i}^2\}\times\sum_{i=1}^N\norm{\vecv_i}^2=\max_{i=1,\cdots,N}\{\norm{\matA_i}^2\}.
    \end{align*}

We then prove for the special case where $\matA_i=0$. We have
\begin{align*}
    &\norm{\matC}^2=\max_{\norm{\vecv}=1}\norm{\matC\vecv}^2=\max_{\norm{\vecv}=1}\sum_{i=1}^N\norm{\sum_{j\neq i}\matB_{ij}\vecv_j}^2\\
    &=\max_{\norm{\vecv}=1}\sum_{i=1}^N\sum_{j,k\neq i}\vecv_k^T\matB_{ik}^T\matB_{ij}\vecv_j\le \max_{\norm{\vecv}=1}\sum_{i=1}^N\sum_{j,k\neq i}\norm{\vecv_k}\norm{\matB_{ik}}\norm{\matB_{ij}}\norm{\vecv_j}\\
    &= \max_{\norm{\vecv}=1}\sum_{i=1}^N\left(\sum_{j\neq i}\norm{\matB_{ij}}\norm{\vecv_j}\right)^2\le \max_{\norm{\vecv}=1}\sum_{i=1}^N \left(\sum_{j\neq i}\norm{\matB_{ij}}^2 \right)\left(\sum_{j}\norm{\vecv_{j}}^2 \right)\\
    &=2\sum_{i<j}\norm{\matB_{ij}}^2,
\end{align*}
where we used Cauchy-Schwarz inequality in the second inequality.

By applying triangle inequality of the matrix operator norm, we can combine the upper bounds derived from two special cases and obtain~\eqref{apeqn:snormaplusb}.
\end{proof}
\vskip 0.2in
\bibliography{sample}

\end{document}

%% file: headjmlr.tex
\usepackage{xspace}
\def\name{\texttt{TCMF}\xspace}
\def\inneralg{\texttt{JIMF}\xspace}
\def\perpca{\texttt{PerPCA}\xspace}
\def\hmf{\texttt{HMF}\xspace}
\def\robustpca{Robuct PCA\xspace}

\newcommand{\revise}[1]{\textcolor{black}{ #1}}

\newcommand{\mO}{\mathcal{O}}

\newcommand{\abs}[1]{\left\vert #1\right\vert}
\newcommand{\norm}[1]{\left\lVert#1\right\rVert}

\newcommand{\supp}[1]{\text{supp}\left(#1\right)}

\newcommand{\tr}[1]{\text{Tr}\left(#1\right)}
\newcommand{\Pj}[1]{\mathbf{P}_{#1}}

\newcommand{\hard}[2]{\text{Hard}_{#1}\left[#2\right]}

\newcommand{\grof}[2]{\mathcal{GR}_{#1}\left(#2\right)}

\newcommand{\matA}{\mathbf{A}}
\newcommand{\matB}{\mathbf{B}}
\newcommand{\matC}{\mathbf{C}}

\newcommand{\matE}{\mathbf{E}}

\newcommand{\matF}{\mathbf{F}}
\newcommand{\matG}{\mathbf{G}}
\newcommand{\matH}{\mathbf{H}}
\newcommand{\hmatHst}{{\hat{\mathbf{H}}}^{\star}{}}
\newcommand{\hmatH}{{\hat{\mathbf{H}}}{}}
\newcommand{\matHst}{{\mathbf{H}^\star}}
\newcommand{\matI}{\mathbf{I}}
\newcommand{\matL}{\mathbf{L}}
\newcommand{\matLst}{{\mathbf{L}^\star}}

\newcommand{\matM}{\mathbf{M}}
\newcommand{\matMst}{{\mathbf{M}^\star}}
\newcommand{\matP}{\mathbf{P}}

\newcommand{\matQ}{\mathbf{Q}}
\newcommand{\matR}{\mathbf{R}}
\newcommand{\matSigma}{\mathbf{\Sigma}}
\newcommand{\matSigmast}{{\mathbf{\Sigma}^\star}}
\newcommand{\matT}{\mathbf{T}}
\newcommand{\matLambda}{\mathbf{\Lambda}}
\newcommand{\matLambdast}{{\mathbf{\Lambda}^\star}}

\newcommand{\matS}{\mathbf{S}}
\newcommand{\matSst}{{\mathbf{S}^\star}}

\newcommand{\matU}{\mathbf{U}}
\newcommand{\matUst}{{\mathbf{U}^\star}}

\newcommand{\dtmatH}{\delta\mathbf{H}}

\newcommand{\matV}{\mathbf{V}}
\newcommand{\matVst}{{\mathbf{V}^\star}}
\newcommand{\matW}{\mathbf{W}}
\newcommand{\matWst}{{\mathbf{W}^\star}}
\newcommand{\matX}{\mathbf{X}}
\newcommand{\matY}{\mathbf{Y}}

\newcommand{\hmatL}{\hat{\mathbf{L}}}
\newcommand{\hmatLeps}{\hat{\mathbf{L}}^{\epsilon}}

\newcommand{\hmatM}{\hat{\mathbf{M}}}

\newcommand{\hmatS}{\hat{\mathbf{S}}}

\newcommand{\hmatU}{\hat{\mathbf{U}}}
\newcommand{\hmatV}{\hat{\mathbf{V}}}
\newcommand{\hmatUeps}{\hat{\mathbf{U}}^{\epsilon}{}}
\newcommand{\hmatVeps}{\hat{\mathbf{V}}^{\epsilon}{}}

\newcommand{\vecv}{\mathbf{v}}
\newcommand{\vech}{\mathbf{h}}

\newcommand{\vece}{\mathbf{e}}

\newcommand{\vecx}{\mathbf{x}}

\newcommand{\tf}{\tilde{f}}

\newcommand{\gop}{\sigma_{\max}}

\newcommand{\fracud}{\frac{1}{2}}

\newcommand{\bn}{{\overline{n}}}

\newcommand{\abnb}{\alpha \bn\max_{i}\norm{\matE_{(i)}}_{\infty}}
\newcommand{\sabnb}{\sqrt{\alpha} \bn\max_{i}\norm{\matE_{(i)}}_{\infty}}

\newcommand{\B}{\norm{\matE}_{\infty}}
\newcommand{\Bi}{\norm{\matE_{(i)}}_{\infty}}

\newcommand{\murnone}{\frac{\mu^2 r}{n_1}}
\newcommand{\murntwo}{\frac{\mu^2 r}{n_2}}
\newcommand{\murnb}{\frac{\mu^2 r}{\overline{n}}}

\newcommand{\sm}{\sigma_{\min}}
\newcommand{\sms}{\sigma_{\min}^2}
\newcommand{\smq}{\sigma_{\min}^4}
\newcommand{\rhom}{\rho_{\min}}

\newcommand{\bcnum}{\left(\frac{\gop}{\sm}\right)}
\newcommand{\invcnum}{\frac{\sm}{\gop}}

\newcommand{\tkappa}{\frac{2\gop^2}{\frac{3}{4}\sms}}
\newcommand{\brtkappa}{\left(\frac{2\gop^2}{\frac{3}{4}\sms}\right)}

\newcommand{\hgzerodef}{\matT_{(0)}\hmatH_g\matLambda_6^{-1}+\sum_{i=1}^N\matT_{(i)}\hmatH_{(i),l}\matLambda_{2,(i)}^{-1}\matLambda_{5,(i)}\matLambda_6^{-1}}
\newcommand{\hilzerodef}{\matT_{(i)}\matH_{(i),l}\matLambda_{2,(i)}^{-1}+\hmatH_{g,0}\matLambda_{4,(i)}\matLambda_{2,(i)}^{-1}}